\documentclass{article}

\usepackage[T1]{fontenc}
\usepackage[htt]{hyphenat}

\usepackage{microtype}
\usepackage{graphicx}
\usepackage{subfigure}
\usepackage{booktabs} 
\usepackage{natbib}
\usepackage{algorithm}
\usepackage{algorithmic}
\usepackage{amsmath}
\usepackage{amssymb}
\usepackage{graphicx}
\usepackage{mathtools}
\usepackage{amsfonts}
\usepackage{amsthm,bm}
\usepackage{color}
\usepackage{enumitem}
\usepackage{dsfont}
\usepackage{pgfplots}
\pgfplotsset{width=10cm,compat=1.9}
 \usepackage{pifont}
\usepackage{thmtools} 
\usepackage{thm-restate}
\usepackage{sidecap}

\newcommand{\R}{\mathbb{R}}

\renewcommand{\phi}{\varphi}

\graphicspath {{figures/}}

\usepackage{hyperref}
\usepackage{caption}
\usepackage[super]{nth}
\newtheorem{lemma}{Lemma}

\newtheorem{theorem}{Theorem}

\newtheorem{proposition}{Proposition}
\newtheorem{example}{Example}
\newtheorem{remark}{Remark}
\newtheorem{corollary}{Corollary}

\mathtoolsset{showonlyrefs}

\usepackage{breakcites}
\usepackage{authblk}
\usepackage[margin=1.35in]{geometry}
\usepackage{enumitem}

\usepackage{hyperref}

\hypersetup{
     colorlinks = true,
     linkcolor = brown,
     anchorcolor = blue,
     citecolor = teal,
     filecolor = blue,
     urlcolor = black
     }

\title{Tighter Sparse Approximation Bounds for ReLU Neural Networks}

\author[a,b]{Carles Domingo-Enrich}
\author[a]{Youssef Mroueh}
\affil[a]{IBM Research AI}
\affil[b]{Courant Institute of Mathematical Sciences, New York University}

\begin{document}

\maketitle

\begin{abstract}
A well-known line of work \citep{barron1993universal,breiman1993hinging,klusowski2018approximation} provides bounds on the width $n$ of a ReLU two-layer neural network needed to approximate a function $f$ over the ball $\mathcal{B}_R(\R^d)$ up to error $\epsilon$, when the \emph{Fourier} based quantity $C_f = \frac{1}{(2\pi)^{d/2}} \int_{\R^d} \|\xi\|^2 |\hat{f}(\xi)| \ d\xi$ is finite. More recently \cite{ongie2019function} used the \emph{Radon transform} as a tool for analysis of infinite-width ReLU two-layer networks. In particular, they introduce the concept of Radon-based $\mathcal{R}$-norms and show that a function defined on $\R^d$ can be represented as an infinite-width two-layer neural network if and only if its $\mathcal{R}$-norm is finite. In this work, we extend the framework of \citep{ongie2019function} and define similar  Radon-based semi-norms ($\mathcal{R}, \mathcal{U}$-norms) such that a function admits an infinite-width neural network representation on a bounded open set $\mathcal{U} \subseteq \R^d$ when its $\mathcal{R}, \mathcal{U}$-norm is finite. Building on this, we derive sparse (finite-width) neural network approximation bounds that refine those of \cite{breiman1993hinging,klusowski2018approximation}. Finally, we show that infinite-width neural network representations on bounded open sets are not unique and study their structure, providing a functional view of mode connectivity.
\end{abstract}

\section{Introduction}
Extensive work has shown that for a neural network to be able to generalize, the size or magnitude of the parameters is more important than the size of the network, when the latter is large enough \citep{bartlett1997forvalid,neyshabur2015insearch,zhang2016understanding}. Under certain regimes, the size of the neural networks used in practice is so large that the training data is fit perfectly and an infinite-width approximation is appropriate. In this setting, what matters to obtain good generalization is to fit the data using the right inductive bias, which is specified by how network parameters are controlled \citep{wei2020regularization} together with the training algorithm used \citep{lyu2020gradient}.

The infinite-width two-layer neural network model has been studied from several perspectives due to its simplicity. One can replace the finite-width ReLU network $\frac{1}{n} \sum_{i=1}^{n} a_i (\langle \omega_i, x \rangle - b_i)_{+}$ by an integral over the parameter space with respect to a signed Radon measure: $\int (\langle \omega, x \rangle - b)_{+} \, d\alpha(\omega,b)$. Thus, controlling the magnitude of the neural network parameters is akin to controlling the measure $\alpha$ according to a certain norm. \cite{bach2017breaking} introduced the $\mathcal{F}_1$-space, which is the infinite-width neural network space with norm $\inf\{ \int |b| \, d|\alpha|(\omega,b) \}$, derived from the finite-width regularizer $\frac{1}{n} \sum_{i=1}^{n} |a_i| \|(\omega_i,b_i)\|_2$ (the infimum is over all the measures $\alpha$ which represent the function at hand). 
A different line of work \citep{savarese2019infinite,ongie2019function} consider the infinite-width spaces with norm $\inf\{ \|\alpha\|_{\text{TV}} = \int \, d|\alpha|(\omega,b) \}$, which is derived from the finite-width regularizer $\frac{1}{n} \sum_{i=1}^{n} |a_i| \|\omega_i\|_2$ (i.e. omitting the bias term). Both of these works seek to find expressions for this norm, leading to characterizations of the functions that are representable by infinite-width networks. \cite{savarese2019infinite} solves the problem in the one-dimensional case: they show that for a function $f$ on $\R$, this norm takes value $\max\{ \int_{\R} |f''(x)| \, dx, |f'(-\infty) + f'(\infty)|\}$. \cite{ongie2019function} give an expression for this norm (the \textit{$\mathcal{R}$-norm}) for functions on $\R^d$, making use of Radon transforms (see \autoref{subsec:existing_representation_results}). 

Although we mentioned in the first paragraph that in many occasions the network size is large enough that the specific number of neurons is irrelevant, when the target function is hard to approximate it is interesting to have an idea of how many neurons one needs to approximate it. The first contribution in this direction was by \cite{cybenko1989approximation, hornik1989multilayer}, which show that two-layer neural networks with enough neurons can approximate any reasonable function on bounded sets in the uniform convergence topology. Later on, \cite{barron1993universal, breiman1993hinging} provided sparse approximation bounds stating that if a function $f$ is such that a certain quantity $C_f$ constructed from the Fourier transform $\hat{f}$ is finite, then there exists a neural network of width $n$ such that the $L^2$ approximation error with respect to a distribution of bounded support is lower than $O(C_f/n)$. More recently, \cite{klusowski2018approximation} provided alternative sparse approximation bounds of \cite{breiman1993hinging} by restricting to networks with bounded weights and a slightly better dependency on $n$ at the expense of a constant factor increasing with $d$ (see \autoref{subsec:existing_sparse}). 

\textbf{Contributions}. In our work, we seek to characterize the functions that coincide with an infinite-width two-layer neural network on a fixed bounded open set. This endeavor is interesting in itself because in practice, we want to learn target functions for which we know samples on a bounded set, and we are typically unconcerned with the values that the learned functions take at infinity. Moreover, the tools that we develop allow us to derive state-of-the-art sparse approximation bounds. 
Our main contributions are the following:
\begin{itemize}[leftmargin=5.5mm]
    \item In the spirit of the $\mathcal{R}$-norm introduced by \cite{ongie2019function}, for any bounded open set $\mathcal{U} \subseteq \R^d$ we define the $\mathcal{R},\mathcal{U}$-norm of a function on $\R^d$, and show that when the $\mathcal{R},\mathcal{U}$-norm of $f$ is finite, $f(x)$ can admits a representation of the form $\int_{\mathbb{S}^{d-1} \times \R} (\langle \omega, x \rangle - b)_{+} \, d\alpha(\omega,b) + \langle v, x \rangle + c$ for $x \in \mathcal{U}$, where $v \in \R^d$, $c \in \R$ and $\alpha$ is an even signed Radon measure.
    \item Using the $\mathcal{R},\mathcal{U}$-norm, we derive function approximation bounds for neural networks with a fixed finite width. We compute the $\mathcal{R},\mathcal{U}$-norm of a function in terms of its Fourier representation, and show that it admits an upper bound by the quantity $C_f$. This shows that our approximation bound is tighter than the previous bound by \cite{breiman1993hinging}, and meaningful in more instances (e.g. for finite-width neural networks). We also show $\mathcal{R},\mathcal{U}$-norm-based bounds analogous to the ones of \cite{klusowski2018approximation}.
    \item Setting $\mathcal{U}$ as the open unit ball of radius $R$, we show that neural network representations of $f$ on $\mathcal{U}$ hold for multiple even Radon measures, which contrasts with the uniqueness result provided by \cite{ongie2019function} for the case of $\R^d$. We study the structure of the sets of Radon measures which give rise to the same function on $\mathcal{U}$. The non-uniqueness of the measure representing a measure could be linked to the phenomenon of mode connectivity. 
\end{itemize}

\textbf{Additional related work}. There have been other recent works which have used the Radon transform to study neural networks in settings different from ours \citep{parhi2021banach, bartolucci2021understanding}. These two works consider the $\mathcal{R}$-norm as a regularizer for an inverse problem, and proceed to prove representer theorems: there exists a solution of the regularized problem which is a two-layer neural network equal to the number of datapoints. Regarding infinite-width network spaces, \cite{e2020representation} present several equivalent definitions and provides a review. A well-known line of work \citep{mei2018mean,chizat2018global,rotskoff2018neural} studies the convergence of gradient descent for infinite-width two-layer neural networks.

\section{Framework}

\subsection{Notation} \label{subsec:notation}
$\mathbb{S}^{d-1}$ denotes the $(d-1)$-dimensional hypersphere (as a submanifold of $\R^d$) and $\mathcal{B}_R(\R^d)$ is the Euclidean open ball of radius $R$.
For $U \subseteq \R^d$ measurable, the space $C_0(U)$ of functions vanishing at infinity contains the continuous functions $f$ such that for any $\epsilon > 0$, there exists compact $K \subseteq U$ depending on $f$ such that $|f(x)| < \epsilon$ for $x \in U \setminus K$. $\mathcal{P}(U)$ is the set of Borel probability measures, $\mathcal{M}(U)$ is the space of finite signed Radon measures (which may be seen as the dual of $C_0(U)$). Throughout the paper, the term Radon measure refers to a finite signed Radon measure for shortness. If $\gamma \in \mathcal{M}(U)$, then ${\|\gamma\|}_{\text{TV}}$ is the total variation (TV) norm of $\gamma$. $\mathcal{M}_{\mathbb{C}}(U)$ denotes the space of complex-valued finite signed Radon measures, defined as the dual space of $C_0(U,\mathbb{C})$ (the space of complex-valued functions vanishing at infinity). 
We denote by $\mathcal{S}(\R^d)$ the space of Schwartz functions, which contains the functions in $\mathcal{C}^{\infty}(\R^{d})$ whose derivatives of any order decay faster than polynomials of all orders, i.e. for all $k, p \in (\mathbb{N}_0)^d$, $\sup_{x \in \R^d} |x^{k} \partial^{(p)} \phi(x)| < +\infty$. 
For $f \in L^1(\R^{d})$, we use $\hat{f}$ to denote the unitary Fourier transform with angular frequency, defined as $\hat{f}(\xi) = \frac{1}{(2\pi)^{d/2}} \int_{\R^{d}} f(x) e^{-i \langle \xi, x \rangle} dx$. If $\hat{f} \in L^1(\R^{d})$ as well, we have the inversion formula $f(x) = \frac{1}{(2\pi)^{d/2}} \int_{\R^d} \hat{f}(\xi) e^{i \langle \xi, x \rangle} dx$. The Fourier transform is a continuous automorphism on $\mathcal{S}(\R^d)$. 


\subsection{Existing sparse approximation bounds} \label{subsec:existing_sparse}

One of the classical results of the theory of two-layer neural networks (\cite{breiman1993hinging}, building on \citep{barron1993universal}) states that given a probability measure $p \in \mathcal{P}(\mathcal{B}_{R}(\R^d))$ and a function $f : \mathcal{B}_{R}(\R^d) \to \R$ admitting a Fourier representation of the form $f(x) = \frac{1}{(2\pi)^{d/2}} \int_{\R^d} e^{i \langle \xi, x \rangle} d\hat{f}(\xi)$, where $\hat{f} \in \mathcal{M}_{\mathbb{C}}(\R^d)$ is a complex-valued Radon measure 
such that
$C_f = \frac{1}{(2\pi)^{d/2}} \int_{\R^d} \|\xi\|_2^2 \ d|\hat{f}|(\xi) < +\infty$, there exists a two-layer neural network $\tilde{f}(x) = \frac{1}{n} \sum_{i=1}^{n} a_i (\langle x, \omega_i \rangle - b_i)_{+}$ such that 
\begin{align} \label{eq:breiman_bound}
    \frac{1}{\text{vol}(\mathcal{B}_{R}(\R^d))} \int_{\mathcal{B}_{R}(\R^d)} (f(x) - \tilde{f}(x))^2 \ dx \leq \frac{(2 R)^4 C_f^2}{n}.
\end{align}
These classical results do not provide bounds on the magnitude of the neural network weights. More recently,
\cite{klusowski2018approximation} showed similar approximation bounds for two-layer ReLU networks under additional $l^1$ and $l^0$ bounds on the weights $a_i, \omega_i$. Namely, if $\tilde{C}_f = \frac{1}{(2\pi)^{d/2}} \int_{\R^d} \|\xi\|_1^2 \ d|\hat{f}|(\xi) < +\infty$ there exists a two-layer neural network $\tilde{f}(x) = a_0 + \langle \omega_0, x \rangle + \frac{\kappa}{n} \sum_{i=1}^{n} a_i (\langle w_i, x \rangle - b_i)_{+}$
with $|a_i| \leq 1, \ \|\omega_i\| \leq 1, \ b_i \in [0,1]$, and $\kappa \leq 2 \tilde{C}_f$, and 
\begin{align} \label{eq:klusowski_bound}
    \sup_{x \in [-1,1]^d} |f(x) - \tilde{f}(x)| \leq c \ \tilde{C}_f \sqrt{d + \log n} \ n^{-1/2 -1/d},
\end{align}
where $c$ is a universal constant.

\subsection{Representation results on $\R^d$ based on the Radon transform} \label{subsec:existing_representation_results}

One defines $\mathbb{P}^d$ denotes the space of hyperplanes on $\R^d$, whose elements may be represented by points in $\mathbb{S}^{d-1} \times \R$ by identifying $\{ x | \langle \omega, x \rangle = b\}$ with both $(\omega,b)$ and $(-\omega,-b)$. Thus, functions on $\mathbb{P}^d$ are even functions on $\mathbb{S}^{d-1} \times \R$ and we will use both notions interchangeably\footnote{Similarly, the space $\mathcal{M}(\mathbb{P}^d)$ of Radon measures over $\mathbb{P}^d$ contains the even measures in $\mathcal{M}(\mathbb{S}^{d-1} \times \R)$. If $\alpha \in \mathcal{M}(\mathbb{P}^d)$, $\int_{\mathbb{S}^{d-1} \times \R} \phi(\omega,b) \, d\alpha(\omega,b)$ is well defined for any measurable function $\phi$ on $\mathbb{S}^{d-1} \times \R$, but $\int_{\mathbb{P}^d} \phi(\omega,b) \, d\alpha(\omega,b)$ is only defined for even $\phi$.}.

\paragraph{The Radon transform and the dual Radon transform.} If $f : \R^d \to \R$ is a function which is integrable over all the hyperplanes of $\R^d$, we may define the Radon transform $\mathcal{R}f : \mathbb{P}^d \to \R$ as
\begin{align}
    \mathcal{R}f(\omega,b) &= \int_{\{x | \langle \omega, x \rangle = b\}} f(x) \, dx, \quad \forall (\omega,b) \in \mathbb{S}^{d-1} \times \R.
\end{align}
That is, one integrates the function $f$ over the hyperplane $(\omega,b)$. If $\Phi : \mathbb{P}^d \to \R$ is a continuous function, the dual Radon transform $\mathcal{R}^{*} \Phi : \R^d \to \R$ is defined as
\begin{align}
    \mathcal{R}^{*}\Phi(x) &= \int_{\mathbb{S}^{d-1}} \Phi(\omega, \langle \omega, x \rangle) \, d\omega, \quad \forall x \in \R^d,
\end{align}
where the integral is with respect to the Hausdorff measure over $\mathbb{S}^{d-1}$. $\mathcal{R}$ and $\mathcal{R}^{*}$ are adjoint operators in the appropriate domains (see \autoref{lem:prop22helgason}).

\paragraph{The Radon inversion formula.} When $f \in C^{\infty}(\R^d)$, one has that (Theorem 3.1, \cite{helgason2011integral})
\begin{align} \label{eq:helgason_inversion}
    f 
    = c_d (-\Delta)^{(d-1)/2} \mathcal{R}^* \mathcal{R} f
\end{align}
where $c_d = \frac{1}{2(2\pi)^{d-1}}$ and $(-\Delta)^{s/2}$ denotes the (negative) fractional Laplacian, defined via its Fourier transform as $\widehat{(-\Delta)^{s/2} f} (\xi) = \|\xi\|^s \hat{f}(\xi)$.

\paragraph{The $\mathcal{R}$-norm.} Given a function $f: \R^d \to \R$, \cite{ongie2019function} introduce the quantity
\begin{align} \label{eq:r_norm_def}
    \|f\|_{\mathcal{R}} = 
    \begin{cases}
    \sup \{-c_d \langle f, (-\Delta)^{(d+1)/2} \mathcal{R}^{*}\psi \rangle \ | \ \psi \in \mathcal{S}(\mathbb{S}^{d-1} \times \R), \ \psi \text{ even}, \ \|\psi\|_{\infty} \leq 1 \} \ &\text{if } f \text{ Lipschitz} \\
    +\infty \ &\text{otherwise}.
    \end{cases}
\end{align}
They call it the $\mathcal{R}$-norm of $f$, although it is formally a semi-norm. 
Here, the space $\mathcal{S}(\mathbb{S}^{d-1} \times \R)$ of Schwartz functions on $\mathbb{S}^{d-1} \times \R$ is defined, in analogy with $\mathcal{S}(\R^d)$, as the space of $C^{\infty}$ functions $\psi$ on $\mathbb{S}^{d-1} \times \R$ which for any integers $k,l \geq 0$ and any differential operator $D$ on $\mathbb{S}^{d-1}$ satisfy $\sup_{(\omega,b) \in \mathbb{S}^{d-1} \times \R} |(1+|b|^k) \partial_{b}^k (D \psi)(\omega,b) | < +\infty$ (\cite{helgason2011integral}, p. 5). Moreover, $\mathcal{S}(\mathbb{P}^d) = \{ \psi \in \mathcal{S}(\mathbb{S}^{d-1} \times \R) \ | \ \psi \text{ even} \}$, which means the conditions on $\psi$ in \eqref{eq:r_norm_def} can be written as $\psi \in \mathcal{S}(\mathbb{P}^d), \|\psi\|_{\infty} \leq 1$.

The finiteness of the $\mathcal{R}$-norm indicates whether a function on $\R^d$ admits an exact representation as an infinitely wide neural network. Namely, \cite{ongie2019function} in their Lemma 10 show that $\|f\|_{\mathcal{R}}$ is finite if and only if there exists a (unique) even measure $\alpha \in \mathcal{M}(\mathbb{S}^{d-1} \times \R)$ and (unique) $v \in \R^d$, $c \in \R$ such that for any $x \in \R^d$,
\begin{align} \label{eq:ongie_decomposition}
    f(x) = \int_{\mathbb{S}^{d-1} \times \R} \left( \langle \omega, x \rangle - b \right)_{+} \, d\alpha(\omega,b) + \langle v, x \rangle + c,
\end{align}
in which case, $\|f\|_{\mathcal{R}} = \|\alpha\|_{\text{TV}}$. 

Remark the following differences between this result and the bounds by \citep{breiman1993hinging,klusowski2018approximation} shown in equations \eqref{eq:breiman_bound} and \eqref{eq:klusowski_bound}: 
\vspace{-5pt}
\begin{enumerate}[label=(\roman*)]
\item in \eqref{eq:ongie_decomposition} we have an exact representation with infinite-width neural networks instead of an approximation result with finite-width, 
\item in \eqref{eq:ongie_decomposition} the representation holds on $\R^d$ instead of a bounded domain. 
\end{enumerate}
\vspace{-5pt}
In our work, we derive representation results similar to the ones of \cite{ongie2019function} for functions defined on bounded open sets, which naturally give rise to sparse approximation results that refine those of \citep{breiman1993hinging,klusowski2018approximation}. 

One property that makes the Radon transform and its dual useful to analyze neural networks can be understood at a very high level via the following argument: if $f(x) = \int_{\mathbb{S}^{d-1} \times \R} (\langle \omega, x \rangle - b)_{+} \rho(\omega,b) \, d(\omega,b) + \langle v, x \rangle + c$ for some smooth rapidly decreasing function $\rho$, then $\Delta f (x) = \int_{\mathbb{S}^{d-1} \times \R} \delta_{\langle \omega, x \rangle = b} \rho(\omega,b) \, d(\omega,b) = \int_{\mathbb{S}^{d-1}} \rho(\omega,\langle \omega, x \rangle) \, d\omega = (\mathcal{R}^{*} \rho)(x)$. For a general function $f$ of the form \eqref{eq:ongie_decomposition}, one has similarly that $\langle \Delta f, \phi \rangle = \langle \alpha, \mathcal{R} \phi \rangle$ for any $\phi \in \mathcal{S}(\R^d)$. This property relates the evaluations of the measure $\alpha$ to the function $\Delta f$ via the Radon transform, and is the main ingredient in the proof of Lemma 10 of \cite{ongie2019function}. While we also rely on it, we need many additional tools to deal with the case of bounded open sets. 


\section{Representation results on bounded open sets} \label{sec:representation_open}
\paragraph{Schwartz functions on open sets.} Let $\mathcal{U} \subseteq \R^d$ be an open subset. The space of Schwartz functions on $\mathcal{U}$ may be defined as $\mathcal{S}(\mathcal{U}) = \bigcap_{z \in \R^d \setminus \mathcal{U}} \bigcap_{k \in (\mathbb{N}_0)^d} \{ f \in \mathcal{S}(\R^d) \ | \ \partial^{(k)} f(z) = 0 \}$, i.e. they are those Schwartz functions on $\R^d$ such that the derivatives of all orders vanish outside of $\mathcal{U}$ (c.f. Def. 3.2, \cite{shaviv2020tempered}).
The structure of $\mathcal{S}(\mathcal{U})$ is similar to $\mathcal{S}(\R^d)$ in that its topology is given by a family of semi-norms indexed by $((\mathbb{N}_0)^d)^2$: $\|f\|_{k,k'} = \sup_{x \in \mathcal{U}} |x^{k} \cdot f^{(k')}(x)|$. Similarly, if $\mathcal{V} \subseteq 
\mathbb{P}^d$ is open, we define $\mathcal{S}(\mathcal{V}) = \bigcap_{(\omega,b) \in (\mathbb{S}^{d-1} \times \R) \setminus \mathcal{V}} \bigcap_{k \in (\mathbb{N}_0)^2} \{ f \in \mathcal{S}(\mathbb{P}^d) \ | \ \partial_b^{k_1} \hat{\Delta}^{k_2} f(\omega,b) = 0 \}$, where $\hat{\Delta}$ is the spherical Laplacian. 

\paragraph{The $\mathcal{R},\mathcal{U}$-norm.}  
Let $\mathcal{U} \subseteq \R^d$ be a bounded open set, and let $\tilde{\mathcal{U}} := \{ (\omega, \langle \omega, x \rangle) \in \mathbb{S}^{d-1} \times \R \, | \, x \in \mathcal{U} \}$. For any function $f : \mathbb{R}^d \to \R$, we define the $\mathcal{R}, \mathcal{U}$-norm of $f$ as
\begin{align} \label{eq:R_U_norm}
    \|f\|_{\mathcal{R},\mathcal{U}} = 
    \sup \{-c_d \langle f, (-\Delta)^{(d+1)/2} \mathcal{R}^{*}\psi \rangle \ | \ \psi \in \mathcal{S}(\tilde{\mathcal{U}}), \ \psi \text{ even}, \ \|\psi\|_{\infty} \leq 1 \}.
\end{align}
Note the similarity between this quantity and the $\mathcal{R}$-norm defined in \eqref{eq:r_norm_def}; the main differences are that the supremum here is taken over the even Schwartz functions on $\tilde{\mathcal{U}}$ instead of $\mathbb{S}^{d-1} \times \R$, and that the non-Lipschitz case does not need a separate treatment. Remark that $\|f\|_{\mathcal{R},\mathcal{U}} \leq \|f\|_{\mathcal{R}}$. If $f$ has enough regularity, we may write $\|f\|_{\mathcal{R},\mathcal{U}} = c_d \int_{\tilde{\mathcal{U}}} |\mathcal{R} (-\Delta)^{(d+1)/2} f|(\omega,b) \, d(\omega,b)$, using that the fractional Laplacian is self-adjoint and $\mathcal{R}^{*}$ is the adjoint of $\mathcal{R}$.

Define $\mathbb{P}^d_{\mathcal{U}}$ to be the bounded open set of hyperplanes of $\R^d$ that intersect $\mathcal{U}$, which in analogy with \autoref{subsec:existing_representation_results}, is equal to $\tilde{\mathcal{U}}$ up to the identification of $(\omega,b)$ with $(-\omega,-b)$. Similarly, note that $\mathcal{S}(\mathbb{P}^d_{\mathcal{U}}) = \{ \psi \in \mathcal{S}(\tilde{\mathcal{U}}), \ \psi \text{ even} \}$, which allows to rewrite the conditions in \eqref{eq:R_U_norm} as $\psi \in \mathcal{S}(\mathbb{P}^d_{\mathcal{U}}), \|\psi\|_{\infty} \leq 1$. 

The following proposition, which is based on the Riesz-Markov-Kakutani representation theorem, shows that when the $\mathcal{R},\mathcal{U}$-norm is finite, it can be associated to a unique Radon measure over $\mathbb{P}^d_{\mathcal{U}}$.

\begin{proposition} \label{prop:alpha_existence}
If $\|f\|_{\mathcal{R},\mathcal{U}} < +\infty$, there exists a unique Radon measure $\alpha \in \mathcal{M}(\mathbb{P}^d_{\mathcal{U}})$ such that $-c_d \langle f, (-\Delta)^{(d+1)/2} \mathcal{R}^{*}\psi \rangle = \int_{\mathbb{P}^d_{\mathcal{U}}} \psi(\omega,b) \, d\alpha(\omega,b)$ for any $\psi \in \mathcal{S}(\mathbb{P}^d_{\mathcal{U}})$. Moreover, $\|f\|_{\mathcal{R},\mathcal{U}} = \|\alpha\|_{\text{TV}}$.
\end{proposition}

Building on this, we see that a neural network representation for bounded $\mathcal{U}$ holds when the $\mathcal{R},\mathcal{U}$-norm is finite:

\begin{theorem} \label{thm:main_theorem_1}
Let $\mathcal{U}$ be a open, bounded subset of $\R^d$. Let $f : \R^d \to \R$ such that $\|f\|_{\mathcal{R},\mathcal{U}} < +\infty$. Let $\alpha \in \mathcal{M}(\mathbb{P}^d_{\mathcal{U}})$ be given by \autoref{prop:alpha_existence}. For any $\phi \in \mathcal{S}(\mathcal{U})$, there exist unique $v \in \R^d$ and $c \in \R$ such that
\begin{align} 
\begin{split} \label{eq:f_psi_2}
    \int_{\mathcal{U}} f(x) \phi(x) \, dx &=  \int_{\mathcal{U}} \left(
    \int_{\tilde{\mathcal{U}}} (\langle \omega, x \rangle - t)_{+} \, d\alpha(\omega, t) + \langle v, x \rangle + c \right) \, \phi(x) \, dx, 
\end{split}
\end{align}
That is, $f(x) = 
\int_{\tilde{\mathcal{U}}} (\langle \omega, x \rangle - t)_{+} \, d\alpha(\omega, t) + \langle v, x \rangle + c$ for $x$ a.e. (almost everywhere) in $\mathcal{U}$. If $f$ is continuous, then the equality holds for all $x \in \mathcal{U}$.
\end{theorem}

Remark that this theorem does not claim that the representation given by $\alpha, v, c$ is unique, unlike Lemma 10 by \cite{ongie2019function} concerning analogous representations on $\R^d$. In \autoref{sec:not_unique} we see that such representations are in fact not unique, for particular choices of the set $\mathcal{U}$. We want to underline that the proof of \autoref{thm:main_theorem_1} uses completely different tools from the ones of Lemma 10 by \cite{ongie2019function}: their result relies critically on the fact that the only harmonic Lipschitz functions on $\R^d$ are affine functions, which is not true for functions on bounded subsets in our setting.  

\section{Sparse approximation for functions with bounded $\mathcal{R}, \mathcal{U}$-norm} \label{sec:sparse_approx}

In this section, we show how to obtain approximation bounds of a function $f$ on a bounded open set $\mathcal{U}$ using a fixed-width neural network with bounded coefficients, in terms of the $\mathcal{R},\mathcal{U}$-norm introduced in the previous section.

\begin{theorem} \label{thm:main_theorem_2}
Let $\mathcal{U} \subseteq \mathcal{B}_{R}(\R^d)$ be a bounded open set. Suppose that $f : \R^d \to \R$ is such that $\|f\|_{\mathcal{R},\mathcal{U}}$ is finite, where $\|\cdot\|_{\mathcal{R},\mathcal{U}}$ is defined in \eqref{eq:R_U_norm}. Let $v \in \R^d, c \in \R$ as in Theorem~\ref{thm:main_theorem_1}.
Then, there exists $\{(\omega_i, b_i)\}_{i=1}^n \subseteq \tilde{\mathcal{U}}$ and $\{a_i\}_{i=1}^n \subseteq \{\pm 1\}$ such that the function $\tilde{f} : \R^d \to \R$ defined as
\begin{align} \label{eq:tilde_f_def}
    \tilde{f}(x) = \frac{\|f\|_{\mathcal{R},\mathcal{U}}}{n}\sum_{i=1}^{n} a_i (\langle \omega_i, x \rangle - b_i)_{+} + \langle v, x \rangle + c
\end{align}
fulfills,
for $x$ a.e. in $\mathcal{U}$,
\begin{align} \label{eq:sup_norm_bound_thm_2}
     \left| \tilde{f}(x) - f(x) \right| \leq \frac{R \|f\|_{\mathcal{R},\mathcal{U}}}{\sqrt{n}}. 
\end{align}
The equality holds for all $x \in \mathcal{U}$ if $f$ is continuous.
\end{theorem}

The proof of \autoref{thm:main_theorem_2} (in \autoref{sec:sparse_approx_proofs}) uses the neural network representation \eqref{eq:f_psi_2} and a probabilistic argument. If one samples $\{(\omega_i, b_i)\}_{i=1}^n$ from a probability distribution proportional to $|\alpha|$, a Rademacher complexity bound upper-bounds the expectation of the supremum norm between $\tilde{f}$ and $f$, which yields the result. 

Note the resemblance of \eqref{eq:sup_norm_bound_thm_2} with the bound \eqref{eq:breiman_bound}; the $\mathcal{R},\mathcal{U}$ norm of $f$ replaces the quantity $C_f$. We can also use the $\mathcal{R},\mathcal{U}$-norm to obtain a bound analogous to \eqref{eq:klusowski_bound}, that is, with a slightly better dependency in the exponent of $n$ at the expense of a constant factor growing with the dimension. 

\begin{proposition} \label{prop:like_klusowski}
Let $f : \R^d \to \R$ and $\mathcal{U} \subseteq \mathcal{B}_1(\R^d)$ open such that $\|f\|_{\mathcal{R},\mathcal{U}} < +\infty$. Then, then there exist $\{a_{i}\}_{i=1}^n \subseteq [-1,1]$, $\{\omega_i\}_{i=1}^n \subseteq \{ \omega \in \R^d | \|\omega\|_1 = 1 \}$ and $\{b_i\}_{i=1}^n \subseteq [0,1]$ and $\kappa < \sqrt{d}\|f\|_{\mathcal{R},\mathcal{U}}$ such that the function 
\begin{align}
    \tilde{f}(x) = \frac{\kappa}{n} \sum_{i=1}^n a_i (\langle \omega_i, x \rangle - b_i)_{+}
\end{align}
fulfills, for $x$ a.e. in $\mathcal{U}$ and some universal constant $c > 0$,
\begin{align} \label{eq:bound_like_klusowski}
    |f(x) - \tilde{f}(x)| \leq c \kappa \sqrt{d + \log n} \, n^{-1/2 - 1/d}.
\end{align}
\end{proposition}

The proof of this result (in \autoref{sec:sparse_approx_proofs}) follows readily from the representation \eqref{eq:f_psi_2} and Theorem 1 of \cite{klusowski2018approximation}.

\subsection{Links with the Fourier sparse approximation bounds}

The following result shows that setting $\mathcal{U} = \mathcal{B}_R(\R^d)$, the $\mathcal{R},\mathcal{U}$-norm can be bounded by the Fourier-based quantities $C_f, \tilde{C}_f$ introduced in \autoref{subsec:existing_sparse}.

\begin{theorem} \label{thm:fourier_refinement}
Assume that the function $f : \R^d \to \R$ admits a Fourier representation of the form $f(x) = \frac{1}{(2\pi)^{d/2}} \int_{\R^d} e^{i \langle \xi, x \rangle} d\hat{f}(\xi)$ with $\hat{f} \in \mathcal{M}_{\mathbb{C}}(\R^d)$ a complex-valued Radon measure. Let $C_f$ be the quantity used in the sparse approximation bound by \cite{breiman1993hinging} (see \autoref{subsec:existing_sparse}). Then, one has that
\begin{align} \label{eq:upper_bound_R_U_norm}
    \|f\|_{\mathcal{R},\mathcal{B}_R(\R^d)} \leq 2 R C_f
\end{align}
\end{theorem}

As a direct consequence of \autoref{thm:fourier_refinement}, when $\mathcal{U} = \mathcal{B}_R(\R^d)$ the right-hand side of \eqref{eq:sup_norm_bound_thm_2} can be upper-bounded by $2 R^2 C_f/\sqrt{n}$. This allows to refine the bound \eqref{eq:breiman_bound} from \cite{breiman1993hinging} to a bound in the supremum norm over $\mathcal{B}_R(\R^d)$, and where the approximating neural network $\tilde{f}(x) = \frac{1}{n} \sum_{i=1}^n a_i (\langle x, \omega_i \rangle - b_i)_{+} + \langle v, x \rangle + c$ fulfills $|a_i| \leq \|f\|_{\mathcal{R},\mathcal{B}_R(\R^d)}$, $\|\omega_i\|_2 \leq 1$ and $b_i \in (-R,R)$.

While we are able to prove the bound \eqref{eq:upper_bound_R_U_norm}, the Fourier representation of $f$ does not allow for a manageable expression for the measure $\alpha$ described in \autoref{prop:alpha_existence}. For that, the following theorem starts from a slightly modified Fourier representation of $f$, under which one can describe the measure $\alpha$ and provide a formula for the $\mathcal{R},\mathcal{U}$-norm. 



\begin{theorem} \label{thm:fourier_corrected}
Let $f: \R^d \to \R$ admitting the representation 
\begin{align} \label{eq:spherical_representation}
    f(x) = \int_{\mathbb{S}^{d-1} \times \R} e^{i b \langle \omega, x \rangle} \, d\mu(\omega,b),
\end{align}
for some complex-valued Radon measures $\mu \in \mathcal{M}_{\mathbb{C}}(\mathbb{S}^{d-1} \times \R)$ such that $d\mu(\omega,b) = d\mu(-\omega,-b) = d\bar{\mu}(-\omega,b) = d\bar{\mu}(\omega,-b)$, and $\int_{\mathbb{S}^{d-1} \times \R} b^2 d|\mu|(\omega,b) < +\infty$. Choosing $\mathcal{U} = \mathcal{B}_R(\R^d)$, the unique measure $\alpha \in \mathcal{M}(\mathbb{P}_R^{d})$ specified by \autoref{prop:alpha_existence} takes the following form:
\begin{align}
    d\alpha(\omega,b) = 
    - \int_{\R} t^2 e^{-itb} \, d\mu(\omega,t) 
    \, db,
\end{align}
Note that $\alpha$ is real-valued because $\int_{\R} t^2 e^{-itb} \, d\mu(\omega,t) \in \R$ as $\overline{t^2 \, d\mu(\omega,t)} = (-t)^2 \, d\mu(\omega,-t)$.
Consequently, the $\mathcal{R},\mathcal{B}_R(\R^d)$-norm of $f$ is
\begin{align} \label{eq:norm_alpha_fourier}
    \|f\|_{\mathcal{R},\mathcal{B}_R(\R^d)} = \|\alpha\|_{\text{TV}} = \int_{-R}^{R} \int_{\mathbb{S}^{d-1}} \left|
    \int_{\R} t^2 e^{-itb} \, d\mu(\omega,t) 
    \right| \, db.
\end{align}
\end{theorem}
Remark that $\mu$ is the pushforward of the measure $\hat{f}$ by the mappings $\xi \mapsto (\pm \xi/\|\xi\|, \pm \xi)$. When the Fourier transform $\hat{f}$ admits a density, one may obtain the density of $\mu$ via a change from Euclidean to spherical coordinates (up to an additional factor $1/(2\pi)^{d/2}$): $d\mu(\omega,b) = \frac{1}{2(2\pi)^{d/2}} \text{vol}(\mathbb{S}^{d-1}) \hat{f}(b\omega) |b|^{d-1} \, d(\omega,b)$. Hence, \autoref{thm:fourier_corrected} provides an operative way to compute the $\mathcal{R},\mathcal{U}$-norm of $f$ if one has access to the Fourier transform of $\hat{f}$. Note that equation \eqref{eq:norm_alpha_fourier} implies that the $\mathcal{R},\mathcal{B}_R(\R^d)$-norm of $f$ increases with $R$, and in particular is smaller than the $\mathcal{R}$-norm of $f$, which corresponds to setting $R = \infty$.

Theorems \ref{thm:fourier_refinement} and \ref{thm:fourier_corrected} are proven jointly in \autoref{sec:sparse_approx_proofs}. Note that from the expression \eqref{eq:norm_alpha_fourier} one can easily see that $\|f\|_{\mathcal{R}, \mathcal{B}_R(\R^d)}$ is upper-bounded by $2 R C_f$:
\begin{align} \label{eq:upper_bound_explained}
    \int_{-R}^{R} \int_{\mathbb{S}^{d-1}} \left|
    \int_{\R} t^2 e^{-itb} \, d\mu(\omega,t)
    \right| \, db \leq \int_{-R}^{R} \int_{\mathbb{S}^{d-1}}
    \int_{\R} t^2 \, d|\mu|(\omega,t) \, db = \frac{2 R}{(2\pi)^{d/2}}
    \int_{\R^d} \|\xi\|^2 d|\hat{f}|(\xi),
\end{align}
where the equality holds since $\mu$ is the pushforward of $\hat{f}/(2\pi)^{d/2}$. Equation \eqref{eq:upper_bound_explained} makes apparent the norm $\|f\|_{\mathcal{R},\mathcal{B}_R(\R^d)}$ is sometimes much smaller than the quantities $C_f, \tilde{C}_f$, as is showcased by the following one-dimensional example (see the proof in \autoref{sec:sparse_approx_proofs}). In these situations, the sparse approximation bounds that we provide in \autoref{thm:main_theorem_2} and \autoref{prop:like_klusowski} are much better than the ones in \eqref{eq:breiman_bound}-\eqref{eq:klusowski_bound}.

\begin{example} \label{ex:one_dim}
Take the function $f : \R \to \R$ defined as $f(x) = \cos(x) - \cos((1+\epsilon) x)$, with $\epsilon > 0$. $f$ admits the Fourier representation $f(x) = \frac{1}{(2\pi)^{1/2}}\int_{\R} \sqrt{\frac{\pi}{2}} (\delta_{1}(\xi) + \delta_{-1}(\xi) - \delta_{1+\epsilon}(\xi) - \delta_{-1-\epsilon}(\xi)) e^{i \xi x} \, d\xi$. We have that $C_f = 2 + 2\epsilon +  \epsilon^2$, and $\|f\|_{\mathcal{R},\mathcal{B}_R(\R^d)} \leq R \left( R \epsilon + 2 \epsilon + \epsilon^2 \right)$.
$\|f\|_{\mathcal{R},\mathcal{B}_R(\R^d)}$ goes to zero as $\epsilon \to 0$, while $C_f$ converges to 2.
\end{example}

An interesting class of functions for which $\|f\|_{\mathcal{R},\mathcal{B}_R(\R^d)}$ is finite but $C_f, \tilde{C}_f$ are infinite are functions that can be written as a finite-width neural network on $\mathcal{B}_R(\R^d)$, as shown in the following proposition.

\begin{proposition} \label{prop:finite_width}
Let $f : \R^d \to \R$ defined as $f(x) = \frac{1}{n} \sum_{i=1}^{n} a_i (\langle \omega_i, x \rangle - b_i)_{+}$ for all $x \in \R^d$, with $\{\omega_i\}_{i=1}^{n} \subseteq \mathbb{S}^{d-1}$, $\{ a_i\}_{i=1}^n, \{ b_i\}_{i=1}^n \subseteq \R$. Then, for any bounded open set $\mathcal{U}$, we have $\|f\|_{\mathcal{R},\mathcal{U}} \leq \frac{1}{n} \sum_{i=1}^{n} |a_i|$, while $C_f, \tilde{C}_f = +\infty$ if $f$ is not an affine function.
\end{proposition}

\autoref{prop:finite_width} makes use of the fact that the $\mathcal{R},\mathcal{U}$-norm is always upper-bounded by the $\mathcal{R}$-norm, which also means that all the bounds developed in \cite{ongie2019function} apply for the $\mathcal{R},\mathcal{U}$-norm. The fact that finite-width neural networks have infinite $C_f$ was stated by \cite{e2020representation}, that used them to show the gap between the functions with finite $C_f$ and functions representable by infinite-width neural networks (belonging to the Barron space, in their terminology). It remains to be seen whether the gap is closed when considering functions with finite $\mathcal{R},\mathcal{U}$-norm, i.e., whether any function admitting an infinite-width representation \eqref{eq:f_psi_2} on $\mathcal{U}$ has a finite $\mathcal{R},\mathcal{U}$-norm.

\textbf{Moving to the non-linear Radon transform.} In many applications the function of interest $f$ may be better represented as $ \int (\langle \omega, \phi(x) \rangle - t)_{+} \, d\alpha(\omega, t) + \langle v, x \rangle + c$, where $\phi$ is a fixed finite dimensional, non-linear and bounded feature map. Our results trivially extend to this case where in the Radon transform hyperplanes are replaced by hyperplanes in the feature space. This can be seen as the ``kernel trick" applied to the Radon transform. The corresponding $\|f\|_{\mathcal{R}, \phi(\mathcal{U})}$ corresponds to the sparsity of the decomposition in the feature space, and we have better approximation when  $\|f\|_{\mathcal{R}, \phi(\mathcal{U})}< \|f\|_{\mathcal{R}, \mathcal{U}}$.
This gives a simple condition for when  transfer learning is successful, and explains the success of using random fourier features as a preprocessing in implicit neural representations of images and surfaces \citep{tancik2020fourfeat}. In order to go beyond the fixed feature maps and tackle deeper ReLU networks,  we think that the non-linear Radon transform \citep{non-linearRadon} is an interesting tool to explore. We note that  \cite{parhi2021kinds} introduced recently a representer theorem for deep ReLU networks using  Radon transforms as a regularizer.       


\section{Infinite-width representations are not unique on bounded sets} \label{sec:not_unique}

\cite{ongie2019function} show that when the $\mathcal{R}$-norm of $f$ is finite, there is a unique measure $\alpha \in \mathcal{M}(\mathbb{P}^d)$ such that the representation \eqref{eq:ongie_decomposition} holds for $x \in \R^d$. In this section we show that when we only ask the representation to hold for $x$ in a bounded open set, there exist several measures that do the job; in fact, they span an infinite-dimensional space.

Let $\mathcal{U} = \mathcal{B}_R(\R^d)$ be the open ball of radius $R > 0$ in $\R^d$, which means that $\tilde{\mathcal{U}} = \mathbb{S}^{d-1} \times (-R,R)$ and $\mathbb{P}^d_{\mathcal{U}}$ is the set of hyperplanes $\{ x | \langle \omega, x \rangle = b \}$ such that $\|\omega\|=1$ and $b \in (-R,R)$, which we denote by $\mathbb{P}_{R}^d$ for simplicity. In the following we will construct a space of Radon measures $\alpha \in \mathcal{M}(\mathbb{P}_{R}^d)$ whose neural network representation \eqref{eq:ongie_decomposition} coincide for all $x \in \mathcal{B}_R(\R^d)$. Note that since any bounded subset of $\R^d$ is included in some open ball, our results imply that such representations are non-unique on any bounded set. 

\begin{remark}
    When one considers representations on $\mathcal{B}_R(\R^d)$ of the sort \eqref{eq:ongie_decomposition} with the measure $\alpha$ lying in the larger space $\mathcal{M}(\mathbb{S}^{d-1} \times \R)$, the non-uniqueness is apparent because there are two `trivial' kinds of symmetry at play:
    \vspace{-5pt}
    \begin{enumerate}[label=(\roman*)]
    \item Related to parity: when the measure $\alpha$ is odd, we have $\int_{\mathbb{S}^{d-1} \times \R} (\langle \omega, x \rangle - b)_{+} \, d\alpha(\omega,b) = \frac{1}{2}\int_{\mathbb{S}^{d-1} \times \R} (\langle \omega, x \rangle - b)_{+} - (-\langle \omega, x \rangle + b)_{+} \, d\alpha(\omega,b) = \langle \frac{1}{2} \int_{\mathbb{S}^{d-1} \times \R} \omega \, d\alpha(\omega,b), x \rangle - \frac{1}{2} \int_{\mathbb{S}^{d-1} \times \R} b \, d\alpha(\omega,b)$, which is an affine function of $x$.
    \vspace{-5pt}
    \item Related to boundedness: if $(\omega,b) \in \mathbb{S}^{d-1} \times (\R \setminus (-R,R))$, $x \mapsto (\langle \omega, x \rangle - b)_{+}$ restricted to $\mathcal{B}_R(\R^d)$ is an affine function of $x$. Hence, if $\alpha$ is supported on $\mathbb{S}^{d-1} \times (\R \setminus (-R,R))$, $x \mapsto \int_{\mathbb{S}^{d-1} \times \R} (\langle \omega, x \rangle - b)_{+} \, d\alpha(\omega,b)$ is an affine function when restricted to $\mathcal{B}_R(\R^d)$. 
    \end{enumerate}
    \vspace{-5pt}
    Since in \autoref{sec:representation_open} we restrict our scope to measures $\alpha$ lying in $\mathcal{M}(\mathbb{P}^d_{\mathcal{U}})$, these two kinds of symmetries are already quotiented out in our analysis. The third kind of non-uniqueness that we discuss in this section is conceptually deeper, taking place within $\mathcal{M}(\mathbb{P}^d_{\mathcal{U}})$.
\end{remark}


Let $\{Y_{k,j} \ | \ k \in \mathbb{Z}^{+}, \ 1 \leq j \leq N_{k,d}\}$ be the orthonormal basis of spherical harmonics of the space $L^2(\mathbb{S}^{d-1})$ \citep{atkinson2012spherical}. It is well known that for any $k$, the functions $\{Y_{k,j} \ | \ 1 \leq j \leq N_{k,d}\}$ are the restrictions to $\mathbb{S}^{d-1}$ of homogeneous polynomials of degree $k$, and in fact $N_{k,d}$ is the dimension of the space of homogeneous harmonic polynomials of degree $k$. Consider the following subset of even functions in $C^{\infty}(\mathbb{S}^{d-1} \times (-R,R))$:
\begin{align}
\begin{split} \label{eq:set_A_def}
    A &= \{ Y_{k,j} \otimes X^{k'} \ | \ k, j, k' \in \mathbb{Z}^{+}, \ k \equiv k' \ (\text{mod } 2), \ k' < k-2, \ 1 \leq j \leq N_{d,k} \},
\end{split}    
\end{align}
where $X^{k'}$ denotes the monomial of degree $k'$ on $(-R,R)$.
We have the following result regarding the non-uniqueness of neural network representations:

\begin{theorem} \label{thm:main_thm_non_unique}
If $\alpha \in \mathcal{M}(\mathbb{P}^d_R)$ is such that $\alpha \in \text{cl}_{\text{w}}(\text{span}(A))$, then we have that $0 = \int_{\mathbb{S}^{d-1} \times (-R,R)} (\langle \omega, x \rangle - b)_{+} \, d\alpha(\omega,b)$ for any $x \in \mathcal{B}_{R}(\R^d)$. That is, $\alpha$ yields a neural network representation of the zero-function on $\mathcal{B}_{R}(\R^d)$.
Here, we consider $\text{span}(A)$ as a subset of $\mathcal{M}(\mathbb{P}^d_R)$ by the Riesz-Markov-Kakutani representation theorem via the action $\langle g, \phi \rangle = \int_{\mathbb{P}^d_R} \phi(\omega,b) g(\omega,b) \, d(\omega,b)$ for any $g \in \text{span}(A), \phi \in C_0(\mathbb{P}^d_R)$, and $\text{cl}_{\text{w}}$ denotes the closure in the topology of weak convergence of $\mathcal{M}(\mathbb{S}^{d-1} \times \R)$.
\end{theorem}
In particular, any measure whose density is in the span of $A$ will yield a function which is equal to zero when restricted to $\mathcal{B}_{R}(\R^d)$.
As an example of this result, we show a simple measure in $\mathcal{M}(\mathbb{P}^d_1)$ which represents the zero function on $\mathcal{B}_{1}(\R^2)$.
\begin{example}
[Non-zero measure representing the zero function on $\mathcal{B}_{1}(\R^2)$]
We define the even Radon measure $\alpha \in \mathcal{M}(\mathbb{S}^{1} \times (-1,1))$ with density $d\alpha(\omega,b) = (8 \omega_0^4 - 8 \omega_0^2 + 1) \, d(\omega,b)$ where $\omega = (\omega_0,\omega_1)$. Then, for any $x \in \mathcal{B}_{1}(\R^2)$, $0 = \int_{\mathbb{S}^{1} \times (-1,1)} (\langle \omega, x \rangle - b)_{+} \, d\alpha(\omega,x)$.
\end{example}


On the one hand, \autoref{prop:alpha_existence} states that  there exists a unique measure $\alpha \in \mathcal{M}(\mathbb{P}^d_{\mathcal{U}})$ such that $-c_d \langle f, (-\Delta)^{(d+1)/2} \mathcal{R}^{*}\psi \rangle = \int_{\mathbb{P}^d_{\mathcal{U}}} \psi(\omega,b) \, d\alpha(\omega,b)$ for any $\psi \in \mathcal{S}(\mathbb{P}^d_{\mathcal{U}})$ if $\|f\|_{\mathcal{R},\mathcal{U}}$ is finite. On the other hand, \autoref{thm:main_thm_non_unique} claims that functions admit distinct representations by measures in $\mathcal{M}(\mathbb{P}^d_{\mathcal{U}})$. The following theorem clarifies these two seemingly contradictory statements. Consider the following subset of even functions in $C^{\infty}(\mathbb{S}^{d-1} \times (-R,R))$, which contains $A$:
\begin{align}
\begin{split} \label{eq:set_B_def}
    B = \{ Y_{k,j} \otimes X^{k'} \ | \ k, j, k' \in \mathbb{Z}^{+}, \ k \equiv k' \ (\text{mod } 2), \ k' < k, \ 1 \leq j \leq N_{d,k} \}.
\end{split}    
\end{align}

\begin{proposition} \label{prop:characterization_alpha}
Let $0 < R < R'$. Let $f : \R^d \to \R$ such that $\|f\|_{\mathcal{R}, \mathcal{B}_{R'}(\R^d)} < +\infty$ and let $\alpha \in \mathcal{M}(\mathbb{P}^d_R)$ be the unique measure specified by \autoref{prop:alpha_existence}. 
Then, $\alpha$ is the unique measure in $\mathcal{M}(\mathbb{P}^d_R)$ such that
\begin{align} \label{eq:alpha_characterization_1}
    \forall \phi \in \mathcal{S}(\mathcal{B}_R(\R^d)), \quad \langle \alpha, \mathcal{R} \phi \rangle &= \int_{\mathcal{B}_R(\R^d))} f(x) \Delta \phi(x) \, dx, \\
    \begin{split}
    \forall k, j, k' \in \mathbb{Z}^{+} \text{ s.t. } k' \equiv k \ (\text{mod } 2), \ &k' < k, \ 1 \leq j \leq N_{k,d}, \\ \langle \alpha, Y_{k,j} \otimes X^{k'} \rangle &= - c_d \langle f, (-\Delta)^{(d+1)/2} \mathcal{R}^{*} (Y_{k,j} \otimes \mathds{1}_{|X|<R} X^{k'}) \rangle.
    \label{eq:alpha_characterization_2}
    \end{split}
\end{align}
The condition \eqref{eq:alpha_characterization_1} holds for any measure $\alpha' \in \mathcal{M}(\mathbb{P}^d_R)$ for which $f$ admits a representation of the form \eqref{eq:f_psi_2} on $\mathcal{B}_R(\R^d)$.
Thus, $\alpha$ can be characterized as the unique measure in $\mathcal{M}(\mathbb{P}^d_R)$ such that $f$ admits a representation of the form \eqref{eq:f_psi_2} on $\mathcal{B}_R(\R^d)$ and the condition \eqref{eq:alpha_characterization_2} holds.
\end{proposition}

In \eqref{eq:alpha_characterization_2}, the quantity $\langle f, (-\Delta)^{(d+1)/2} \mathcal{R}^{*} (Y_{k,j} \otimes \mathds{1}_{|X|<R} X^{k'}) \rangle$ is well defined despite $\mathds{1}_{|X|<R} X^{k'}$ not being continuous on $\R$; we define it as $\langle f, (-\Delta)^{(d+1)/2} \mathcal{R}^{*} ((Y_{k,j} \otimes \mathds{1}_{|X|<R} X^{k'}) + \tilde{g}) \rangle$, where $\tilde{g}$ is any function in $\mathcal{S}(\mathbb{P}^d_{R'})$ such that $(Y_{k,j} \otimes \mathds{1}_{|X|<R} X^{k'}) + \tilde{g} \in \mathcal{S}(\mathbb{P}^d_{R'})$ (which do exist, see \autoref{sec:app_not_unique}).

In short, \autoref{prop:characterization_alpha} characterizes the measure $\alpha$ from \autoref{prop:alpha_existence} in terms of its evaluations on the spaces $\mathcal{R}(\mathcal{S}(\mathcal{B}_R(\R^d)))$ and $\text{span}(B)$, and by \autoref{cor:C_0_span_B_R_S} the direct sum of these two spaces dense in $C_0(\mathbb{P}^d_R)$, which by the Riesz-Markov-Kakutani representation theorem is the predual of $\mathcal{M}(\mathbb{P}^d_R)$. 
Interestingly, the condition \eqref{eq:alpha_characterization_1} holds for any measure $\alpha \in \mathcal{M}(\mathbb{P}^d_R)$ which represents the function $f$ on $\mathcal{B}_R(\R^d)$, but it is easy to see that the condition \eqref{eq:alpha_characterization_2} does not: by \autoref{thm:main_thm_non_unique} we have that if $\psi \in \text{span}(A) \subseteq \text{span}(B)$, the measure $\alpha'$ defined as $d\alpha'(\omega,b) = d\alpha(\omega,b) + \psi(\omega, b) \, db$ represents the function $f$ on $\mathcal{B}_R(\R^d)$, and $\langle \alpha', \psi \rangle = \langle \alpha, \psi \rangle + \|\psi\|^2_2$. 

It remains an open question to see whether \autoref{thm:main_thm_non_unique} captures all the measures which represent the zero function on $\mathcal{B}_R(\R^d)$, which we hypothesize. If that was the case, we would obtain a complete characterization of the Radon measures which represent a given function on $\mathcal{B}_R(\R^d)$. 

\textbf{Mode connectivity}. Mode connectivity is the phenomenon that optima of neural network losses (at least the ones found by gradient descent) turn out to be connected by paths where the loss value is almost constant, and was observed empirically by \cite{garipov2018loss,draxler2018essentially}. \cite{kuditipudi2019explainging} provided an algorithmic explanation based on dropout, and an explanation based on the noise stability property. \autoref{thm:main_thm_non_unique} suggests an explanation for mode connectivity from a functional perspective: one can construct finitely-supported measures which approximate a measure $\alpha \in \text{cl}_w(\text{span}(A))$, yielding finite-width neural networks with non-zero weights which approximate the zero function on $\mathcal{B}_R(\R^d)$. Assuming that the data distribution is supported in $\mathcal{B}_R(\R^d)$, adding a multiple of one such network to an optimal network will produce little change in the loss value because the function being represented is essentially unchanged. More work is required to confirm or discard this intuition. 


\section{Conclusion}
We provided in this paper tighter sparse approximation bounds for two-layer ReLU neural networks. Our results build on the introduction of Radon-based $\mathcal{R}, \mathcal{U}$-norms for functions defined on a  bounded open set $\mathcal{U}$. Our bounds refine Fourier-based approximation bounds of \cite{breiman1993hinging, klusowski2018approximation}. We also showed that the representation of infinite width neural networks on bounded open sets are not unique, which can be seen as a functional view of mode connectivity observed in training deep neural networks. We leave two open questions: whether any function admitting an infinite-width representation on $\mathcal{U}$ has a finite $\mathcal{R},\mathcal{U}$-norm, and whether \autoref{thm:main_thm_non_unique} captures all the measures which represent the zero function on $\mathcal{B}_R(\R^d)$. Finally, in order to extend our theory to deeper ReLU networks we believe that non-linear Radon transforms \citep{non-linearRadon} are interesting tools to explore.

\bibliography{refs,simplex,biblio}

\begin{thebibliography}{33}
\providecommand{\natexlab}[1]{#1}
\providecommand{\url}[1]{\texttt{#1}}
\expandafter\ifx\csname urlstyle\endcsname\relax
  \providecommand{\doi}[1]{doi: #1}\else
  \providecommand{\doi}{doi: \begingroup \urlstyle{rm}\Url}\fi

\bibitem[Atkinson and Han(2012)]{atkinson2012spherical}
K.~Atkinson and W.~Han.
\newblock \emph{Spherical Harmonics and Approximations on the Unit Sphere: An
  Introduction}, volume 2044.
\newblock Springer, 01 2012.

\bibitem[Bach(2017)]{bach2017breaking}
F.~Bach.
\newblock Breaking the curse of dimensionality with convex neural networks.
\newblock \emph{Journal of Machine Learning Research}, 18\penalty0
  (19):\penalty0 1--53, 2017.

\bibitem[Barron(1993)]{barron1993universal}
A.~Barron.
\newblock Universal approximation bounds for superpositions of a sigmoidal
  function.
\newblock \emph{IEEE Transactions on Information Theory}, 39:\penalty0 930 --
  945, 1993.

\bibitem[Bartlett(1997)]{bartlett1997forvalid}
P.~Bartlett.
\newblock For valid generalization the size of the weights is more important
  than the size of the network.
\newblock In \emph{Advances in Neural Information Processing Systems},
  volume~9. MIT Press, 1997.

\bibitem[Bartolucci et~al.(2021)Bartolucci, Vito, Rosasco, and
  Vigogna]{bartolucci2021understanding}
F.~Bartolucci, E.~D. Vito, L.~Rosasco, and S.~Vigogna.
\newblock Understanding neural networks with reproducing kernel banach spaces,
  2021.

\bibitem[Breiman(1993)]{breiman1993hinging}
L.~Breiman.
\newblock Hinging hyperplanes for regression, classification, and function
  approximation.
\newblock \emph{IEEE Transactions on Information Theory}, 39\penalty0
  (3):\penalty0 999--1013, 1993.

\bibitem[Chizat and Bach(2018)]{chizat2018global}
L.~Chizat and F.~Bach.
\newblock On the global convergence of gradient descent for over-parameterized
  models using optimal transport.
\newblock In \emph{Advances in neural information processing systems}, pages
  3036--3046, 2018.

\bibitem[Cybenko(1989)]{cybenko1989approximation}
G.~Cybenko.
\newblock {Approximation by superpositions of a sigmoidal function}.
\newblock \emph{Mathematics of Control, Signals, and Systems (MCSS)},
  2\penalty0 (4):\penalty0 303--314, 1989.

\bibitem[Draxler et~al.(2018)Draxler, Veschgini, Salmhofer, and
  Hamprecht]{draxler2018essentially}
F.~Draxler, K.~Veschgini, M.~Salmhofer, and F.~Hamprecht.
\newblock Essentially no barriers in neural network energy landscape.
\newblock In \emph{Proceedings of the 35th International Conference on Machine
  Learning}, volume~80, pages 1309--1318, 2018.

\bibitem[E and Wojtowytsch(2020)]{e2020representation}
W.~E and S.~Wojtowytsch.
\newblock Representation formulas and pointwise properties for barron
  functions, 2020.

\bibitem[Ehrenpreis(2003)]{non-linearRadon}
L.~Ehrenpreis.
\newblock \emph{The Universality of the Radon Transform}.
\newblock Oxford, 2003.

\bibitem[Garipov et~al.(2018)Garipov, Izmailov, Podoprikhin, Vetrov, and
  Wilson]{garipov2018loss}
T.~Garipov, P.~Izmailov, D.~Podoprikhin, D.~P. Vetrov, and A.~G. Wilson.
\newblock Loss surfaces, mode connectivity, and fast ensembling of dnns.
\newblock In \emph{Advances in Neural Information Processing Systems},
  volume~31. Curran Associates, Inc., 2018.

\bibitem[Helgason(1994)]{helgason1994geometric}
S.~Helgason.
\newblock \emph{Geometric Analysis on Symmetric Spaces}.
\newblock Mathematical surveys and monographs. American Mathematical Society,
  1994.

\bibitem[Helgason(2011)]{helgason2011integral}
S.~Helgason.
\newblock \emph{Integral Geometry and Radon Transforms}.
\newblock Springer, 2011.

\bibitem[Hornik et~al.(1989)Hornik, Stinchcombe, and
  White]{hornik1989multilayer}
K.~Hornik, M.~Stinchcombe, and H.~White.
\newblock Multilayer feedforward networks are universal approximators.
\newblock \emph{Neural Networks}, 2\penalty0 (5):\penalty0 359--366, 1989.

\bibitem[Kakade et~al.(2009)Kakade, Sridharan, and Tewari]{kakade2009onthe}
S.~M. Kakade, K.~Sridharan, and A.~Tewari.
\newblock On the complexity of linear prediction: Risk bounds, margin bounds,
  and regularization.
\newblock In \emph{Advances in Neural Information Processing Systems},
  volume~21, pages 793--800. Curran Associates, Inc., 2009.

\bibitem[Klusowski and Barron(2018)]{klusowski2018approximation}
J.~M. Klusowski and A.~R. Barron.
\newblock Approximation by combinations of relu and squared relu ridge
  functions with $\ell^1$ and $\ell^0$ controls.
\newblock \emph{IEEE Transactions on Information Theory}, 64\penalty0
  (12):\penalty0 7649--7656, 2018.

\bibitem[Kuditipudi et~al.(2019)Kuditipudi, Wang, Lee, Zhang, Li, Hu, Ge, and
  Arora]{kuditipudi2019explainging}
R.~Kuditipudi, X.~Wang, H.~Lee, Y.~Zhang, Z.~Li, W.~Hu, R.~Ge, and S.~Arora.
\newblock Explaining landscape connectivity of low-cost solutions for
  multilayer nets.
\newblock In \emph{Advances in Neural Information Processing Systems},
  volume~32. Curran Associates, Inc., 2019.

\bibitem[Lyu and Li(2020)]{lyu2020gradient}
K.~Lyu and J.~Li.
\newblock Gradient descent maximizes the margin of homogeneous neural networks.
\newblock In \emph{8th International Conference on Learning Representations,
  {ICLR} 2020}. OpenReview.net, 2020.

\bibitem[Mei et~al.(2018)Mei, Montanari, and Nguyen]{mei2018mean}
S.~Mei, A.~Montanari, and P.-M. Nguyen.
\newblock A mean field view of the landscape of two-layer neural networks.
\newblock \emph{Proceedings of the National Academy of Sciences}, 115\penalty0
  (33):\penalty0 E7665--E7671, 2018.

\bibitem[Mohri et~al.(2012)Mohri, Rostamizadeh, and
  Talwalkar]{mohri2012foundations}
M.~Mohri, A.~Rostamizadeh, and A.~Talwalkar.
\newblock \emph{Foundations of Machine Learning}.
\newblock The MIT Press, 2012.

\bibitem[Neyshabur et~al.(2015)Neyshabur, Tomioka, and
  Srebro]{neyshabur2015insearch}
B.~Neyshabur, R.~Tomioka, and N.~Srebro.
\newblock In search of the real inductive bias: On the role of implicit
  regularization in deep learning.
\newblock In \emph{ICLR (Workshop)}, 2015.

\bibitem[Ongie et~al.(2019)Ongie, Willett, Soudry, and
  Srebro]{ongie2019function}
G.~Ongie, R.~Willett, D.~Soudry, and N.~Srebro.
\newblock A function space view of bounded norm infinite width relu nets: The
  multivariate case.
\newblock In \emph{International Conference on Learning Representations (ICLR
  2020)}, 2019.

\bibitem[Parhi and Nowak(2021{\natexlab{a}})]{parhi2021banach}
R.~Parhi and R.~D. Nowak.
\newblock Banach space representer theorems for neural networks and ridge
  splines.
\newblock \emph{Journal of Machine Learning Research}, 22\penalty0
  (43):\penalty0 1--40, 2021{\natexlab{a}}.

\bibitem[Parhi and Nowak(2021{\natexlab{b}})]{parhi2021kinds}
R.~Parhi and R.~D. Nowak.
\newblock What kinds of functions do deep neural networks learn? insights from
  variational spline theory, 2021{\natexlab{b}}.

\bibitem[Rotskoff and Vanden-Eijnden(2018)]{rotskoff2018neural}
G.~M. Rotskoff and E.~Vanden-Eijnden.
\newblock Neural networks as interacting particle systems: Asymptotic convexity
  of the loss landscape and universal scaling of the approximation error.
\newblock \emph{arXiv preprint arXiv:1805.00915}, 2018.

\bibitem[Rudin(1991)]{rudin1991functional}
W.~Rudin.
\newblock \emph{Functional Analysis}.
\newblock International series in pure and applied mathematics. McGraw-Hill,
  1991.

\bibitem[Salazar(2018)]{salazar2018fubinitonelli}
J.~Salazar.
\newblock Fubini-tonelli type theorem for non product measures in a product
  space, 2018.

\bibitem[Savarese et~al.(2019)Savarese, Evron, Soudry, and
  Srebro]{savarese2019infinite}
P.~Savarese, I.~Evron, D.~Soudry, and N.~Srebro.
\newblock How do infinite width bounded norm networks look in function space?
\newblock In \emph{Conference on Learning Theory}, 2019.

\bibitem[Shaviv(2020)]{shaviv2020tempered}
A.~Shaviv.
\newblock Tempered distributions and schwartz functions on definable manifolds.
\newblock \emph{Journal of Functional Analysis}, 278\penalty0 (11), 2020.

\bibitem[Tancik et~al.(2020)Tancik, Srinivasan, Mildenhall, Fridovich-Keil,
  Raghavan, Singhal, Ramamoorthi, Barron, and Ng]{tancik2020fourfeat}
M.~Tancik, P.~P. Srinivasan, B.~Mildenhall, S.~Fridovich-Keil, N.~Raghavan,
  U.~Singhal, R.~Ramamoorthi, J.~T. Barron, and R.~Ng.
\newblock Fourier features let networks learn high frequency functions in low
  dimensional domains.
\newblock \emph{NeurIPS}, 2020.

\bibitem[Wei et~al.(2020)Wei, Lee, Liu, and Ma]{wei2020regularization}
C.~Wei, J.~D. Lee, Q.~Liu, and T.~Ma.
\newblock Regularization matters: Generalization and optimization of neural
  nets v.s. their induced kernel, 2020.

\bibitem[Zhang et~al.(2016)Zhang, Bengio, Hardt, Recht, and
  Vinyals]{zhang2016understanding}
C.~Zhang, S.~Bengio, M.~Hardt, B.~Recht, and O.~Vinyals.
\newblock Understanding deep learning requires rethinking generalization.
\newblock \emph{arXiv preprint arXiv:1611.03530}, 2016.

\end{thebibliography}

\appendix

\section{Proofs of \autoref{sec:representation_open}} \label{sec:proof_rep_open}

\textbf{\textit{Proof of \autoref{prop:alpha_existence}.}}
We make an argument analogous to Lemma 10 of \cite{ongie2019function}. For $\psi \in \mathcal{S}(\mathbb{P}^d_{\mathcal{U}})$, let $L_f(\psi) = - c_d \langle f, (-\Delta)^{(d+1)/2} \mathcal{R}^* \psi \rangle$. If $\sup \{ L_f(\psi) \ | \ \psi \in \mathcal{S}(\mathbb{P}^d_{\mathcal{U}}), \ \|\psi\|_{\infty} \leq 1 \} < +\infty$, since $\mathcal{S}(\mathbb{P}^d_{\mathcal{U}})$ is dense in $C_0(\mathbb{P}^d_{\mathcal{U}})$, by the bounded linear transformation theorem (\autoref{lem:blt}) and the Riesz-Markov-Kakutani theorem, there exists a unique $\alpha \in C_0'(\mathbb{P}^d_{\mathcal{U}}) = \mathcal{M}(\mathbb{P}^d_{\mathcal{U}})$ such that for all $\psi \in \mathcal{S}(\mathbb{P}^d_{\mathcal{U}})$, $\int_{\mathbb{P}^d_{\mathcal{U}}} \psi(\omega,b) \, d\alpha(\omega,b) = L_f(\psi)$, and $\sup \{ L_f(\psi) \ | \ \psi \in \mathcal{S}(\mathbb{P}^d_{\mathcal{U}}), \ \|\psi\|_{\infty} \leq 1 \} = \|\alpha\|_{\text{TV}}$.
\qed

\begin{lemma}[Bounded linear transformation (BLT) theorem] \label{lem:blt}
Every bounded linear transformation $T$ from a normed vector space $X$ to a complete, normed vector space 
$Y$ can be uniquely extended to a bounded linear transformation $\tilde{T}$ from the completion of $X$ to $Y$. In addition, the operator norm of $T$ is $c$ iff the norm of $\tilde{T}$ is $c$.
\end{lemma}

\textbf{\textit{Proof of \autoref{thm:main_theorem_1}.}}
Note that without loss of generality we can assume that $0 \in \mathcal{U}$. This is because if $0 \notin \mathcal{U}$ but some other point $y$ does belong to $\mathcal{U}$, we can consider the function $f(\cdot - y)$ and the set $\mathcal{U} - y$. Then, once we have a representation of the form \eqref{eq:ongie_decomposition} for $f(\cdot - y)$ and $x \in \mathcal{U} - y$, we may obtain a representation for $f$ and $x \in \mathcal{U}$ by massaging the bias terms and the independent term appropriately. 

If $0 \in \mathcal{U}$, since $\mathcal{U}$ is open there exists $\epsilon > 0$ such that the ball of radius $\epsilon$ centered at $0$ is included in $\mathcal{U}$. Thus, for any $\omega \in \mathbb{S}^{d-1}$ and $b \in (-\epsilon, \epsilon)$, the point $x = b \omega$ belongs to $\mathcal{U}$ and fulfills $\langle \omega, x \rangle = b$. This means that $\mathbb{S}^{d-1} \times (-\epsilon,\epsilon) \subseteq \tilde{\mathcal{U}} := \{ (\omega,\langle \omega, x \rangle) \in \mathbb{S}^{d-1} \times \R \ | \ x \in \mathcal{U} \}$.

\begin{lemma}[\cite{helgason2011integral}, Lemma 2.1] \label{lem:intertwining}
The following intertwining relations of the Radon transform and its dual with the Laplacian operator hold
\begin{align} \label{eq:laplacian_lambda}
(-\Delta )^{s/2} \mathcal{R}^* = (-i)^{s} \mathcal{R}^* \Lambda^{s} \text{ and } \mathcal{R} (-\Delta )^{s/2} = (-i)^{s} \Lambda^{s} \mathcal{R}, \text{ for any } s \in \mathbb{Z}^{+},
\end{align}
where $\Lambda^s$ is known as the ramp filter and is defined as
\begin{align}
    \Lambda^s \Phi (\omega,b) = 
    \begin{cases}
    \partial_b^s \Phi(\omega,b) &s \text{ even} \\
    \mathcal{H}_b \partial_b^s \Phi(\omega,b) &s \text{ odd}.
    \end{cases}
\end{align}
Here, $\mathcal{H}_b$ denotes the Hilbert transform with respect to the variable $b$.
\end{lemma}

\begin{lemma} \label{lem:phi_Psi}
Let $\epsilon > 0$ such that the ball of radius $\epsilon$ centered at $0$ is included in $\mathcal{U}$. Let $\eta \in C_{c}^{\infty}(\R)$ such that $\int_{-\infty}^{\infty} \eta(x) \, dx = 1$, such that $\text{supp}(\eta) \subseteq (-\epsilon,\epsilon)$.
For any $\chi \in \mathcal{S}(\mathbb{P}^d_{\mathcal{U}})$, define $\beta_{\chi}(\omega) := \int_{-\infty}^{\infty} \chi(\omega,x) \, dx$. Define also $\psi_1 \in \mathcal{S}(\mathbb{P}^d_{\mathcal{U}})$ as
\begin{align}
    \psi_1(\omega,b) = \chi(\omega,b) - \beta_{\chi}(\omega) \, \eta(b)
\end{align}
Define $\Psi_1 : \tilde{\mathcal{U}} \to \R$ as
\begin{align}
    \Psi_1(\omega,b) = \int_{-\infty}^{b} \psi_1(\omega,t) \, dt.
\end{align}
$\Psi_1$ belongs to $\mathcal{S}(\tilde{\mathcal{U}})$. Define $\gamma_{\chi}(\omega) := \int_{-\infty}^{\infty} \Psi_1(\omega,t) \, dt$, and
\begin{align}
    \psi_2(\omega,b) &= \Psi_1(\omega,b) - \gamma_{\chi}(\omega) \, \eta(b), \\
    \Psi_2(\omega,b) &= \int_{-\infty}^{b} \psi_2(\omega,t) \, dt.
\end{align}
$\psi_2$ and $\Psi_2$ belong to $\mathcal{S}(\tilde{\mathcal{U}})$.
\end{lemma}
\begin{proof}
First, note that $\psi_1 \in \mathcal{S}(\tilde{\mathcal{U}})$ because $\chi \in \mathcal{S}(\mathbb{P}^d_{\mathcal{U}}) \subset \mathcal{S}(\tilde{\mathcal{U}})$ and $\eta \in \mathcal{S}(\mathbb{S}^d \times (-\epsilon,\epsilon)) \subseteq \mathcal{S}(\tilde{\mathcal{U}})$.

Since the support of $\psi_1$ is contained in $\text{cl}(\tilde{\mathcal{U}})$, to check that $\Psi_1 \in \mathcal{S}(\tilde{\mathcal{U}})$ it is sufficient to see that for all $\omega \in \mathbb{S}^d$, $\int_{-\infty}^{\infty} \psi_1(\omega,b) \, db = 0$, which holds because
\begin{align}
    \int_{-\infty}^{\infty} (\chi(\omega,b) - \beta_{\chi}(\omega) \, \eta(b)) \, db = \int_{-\infty}^{\infty} \chi(\omega,b) \, db - \beta_{\chi}(\omega) = 0.
\end{align}

Since the support of $\psi_2$ is contained in $\text{cl}(\tilde{\mathcal{U}})$, to check that $\Psi_2 \in \mathcal{S}(\tilde{\mathcal{U}})$ it is sufficient to see that for all $\omega \in \mathbb{S}^d$, $\int_{-\infty}^{\infty} \psi_2(\omega,b) \, db = 0$, which holds because
\begin{align}
    \int_{-\infty}^{\infty} (\Psi_1(\omega,b) - \gamma_{\chi}(\omega) \, \eta(b)) \, db = \int_{-\infty}^{\infty} \Psi_1(\omega,b) \, db - \gamma_{\chi}(\omega) = 0.
\end{align}
\end{proof}

Using the Radon inversion formula \eqref{eq:helgason_inversion} together with the first intertwining relation of \autoref{lem:intertwining}, we obtain an alternative Radon inversion formula:
\begin{align} \label{eq:alternative_inversion}
    \phi = c_d (-\Delta)^{(d-1)/2} \mathcal{R}^* \mathcal{R} \phi = (-i)^{d-1} c_d \mathcal{R}^* \Lambda^{d-1} \mathcal{R} \phi, \quad \forall \phi \in C^{\infty}(\R^d) \text{ integrable}
\end{align}
Similarly, the measure $\alpha \in \mathcal{M}(\mathbb{P}^d_{\mathcal{U}})$ given by \autoref{prop:alpha_existence} satisfies
\begin{align}
\begin{split} \label{eq:alpha_reexpressed}
    \forall \psi \in \mathcal{S}(\mathbb{P}^d_{\mathcal{U}}), \quad \langle \alpha, \psi \rangle = -c_d \langle f, (-\Delta)^{(d+1)/2} \mathcal{R}^* \psi \rangle = (-i)^{d-1} c_d \langle f,  \mathcal{R}^* \Lambda^{d+1} \psi \rangle 
\end{split}    
\end{align}
Now, for any $\phi \in \mathcal{S}(\mathcal{U})$, we have that $\mathcal{R} \phi \in \mathcal{S}(\mathbb{P}^d_{\mathcal{U}})$ by the definition of $\mathbb{P}^d_{\mathcal{U}}$. Let us apply \autoref{lem:phi_Psi} with the choice $\chi = \mathcal{R} \phi$. Then, 
\begin{align} 
\begin{split} \label{eq:f_psi}
    \langle f, \phi \rangle &= (-i)^{d-1} c_d \langle f,  \mathcal{R}^* \Lambda^{d-1} \mathcal{R} \phi \rangle = (-i)^{d-1} c_d \langle f,  \mathcal{R}^* \Lambda^{d-1} \chi \rangle 
    \\ &= (-i)^{d-1} c_d \langle f,  \mathcal{R}^* \Lambda^{d-1} \psi_1 \rangle + (-i)^{d-1} c_d \langle f,  \mathcal{R}^* \beta_{\chi} \, \Lambda^{d-1} \eta \rangle
    \\ &= (-i)^{d-1} c_d \langle f,  \mathcal{R}^* \Lambda^{d-1} \partial_b \Psi_1 \rangle + (-i)^{d-1} c_d \langle f,  \mathcal{R}^* \beta_{\chi} \, \Lambda^{d-1} \eta \rangle
    \\ &= (-i)^{d-1} c_d \langle f,  \mathcal{R}^* \Lambda^{d-1} \partial_b \psi_2 \rangle + (-i)^{d-1} c_d \langle f,  \mathcal{R}^* \gamma_{\chi} \, \Lambda^{d-1} \partial_b \eta \rangle + (-i)^{d-1} c_d \langle f,  \mathcal{R}^* \beta_{\chi} \, \Lambda^{d-1} \eta \rangle
    \\ &= (-i)^{d-1} c_d \langle f,  \mathcal{R}^* \Lambda^{d+1} \Psi_2 \rangle + (-i)^{d-1} c_d \langle f,  \mathcal{R}^* \Lambda^{d-1} (\beta_{\chi} \, \eta + \gamma_{\chi} \, \partial_b \eta) \rangle
\end{split}    
\end{align}
The first equality holds by \eqref{eq:alternative_inversion}. We used that $\chi = \psi_1 + \beta_{\chi} \eta$, that $\partial_b \Psi_1 = \psi_1$, that $\Psi_1 = \psi_2 + \gamma_{\chi} \eta$, that $\partial_b \Psi_2 = \psi_2$ and that $\Lambda^{d-1} \partial_b^2 = \Lambda^{d+1}$.
We develop the first and the second terms of the right-hand side separately. 
\begin{align} \label{eq:first_term_Psi}
    (-i)^{d-1} c_d \langle f,  \mathcal{R}^* \Lambda^{d+1} \Psi_2 \rangle = \int_{\mathcal{V}} \Psi_2(\omega,b) \, d\alpha(\omega,b) = \int_{\mathbb{S}^d} \int_{-\infty}^{\infty} \Psi_2(\omega,b) \, d\alpha_{\R|\omega}(b) \, d\alpha_{\mathbb{S}^{d-1}}(\omega),
\end{align}
In the first equality we use \eqref{eq:alpha_reexpressed}. 
In the second equality we make use of a generalization of the Fubini-Tonelli theorem (c.f. Theorem 2.1, \cite{salazar2018fubinitonelli}), which states that $\alpha_{\R|\omega} \in \mathcal{M}(\R)$ is defined for $\omega$ $\alpha_{\mathbb{S}^{d-1}}$-a.e. in $\mathbb{S}^{d-1}$. While $\alpha$ is a finite signed Radon measure and \cite{salazar2018fubinitonelli} deals with probability measures, the result can be applied by decomposing Radon measures into non-negative finite measures via the Hahn decomposition theorem. We develop the integrand in the right-hand side via Lebesgue-Stieltjes integration by parts:
\begin{align}
\begin{split} \label{eq:Psi_2}
    &\int_{-\infty}^{\infty} \Psi_2(\omega,b) \, d\alpha_{\R|\omega}(b) = \left[ \Psi_2(\omega,b) \int_{-\infty}^{b} d\alpha_{\R|\omega}(t) \right]_{b=-\infty}^{b=+\infty} - \int_{-\infty}^{\infty} \int_{-\infty}^{b} d\alpha_{\R|\omega}(t) \, \partial_b \Psi_2(\omega,b) \, db \\ &= -\int_{-\infty}^{\infty} \int_{-\infty}^{b} d\alpha_{\R|\omega}(t) \, \psi_2(\omega,b) \, db \\ &=  -\int_{-\infty}^{\infty} \int_{-\infty}^{b} d\alpha_{\R|\omega}(t) \, \Psi_1(\omega,b) \, db + \int_{-\infty}^{\infty} \int_{-\infty}^{b} d\alpha_{\R|\omega}(t) \, \gamma_{\chi}(\omega) \eta(b) \, db
\end{split}
\end{align}
Here we applied \autoref{lem:lebesgue-stieltjes} with $U(x) = \Psi_2(\omega,x)$ and $V(x) = \int_{-\infty}^{x} d\alpha_{\R|\omega}(t)$, and we used that $\Psi_1 = \psi_2 + \gamma_{\chi} \eta$. We apply \autoref{lem:lebesgue-stieltjes} again on the first-term in the right-hand side, this time with $U(x) = \Psi_1(\omega,x)$ and $V(x) = \int_{-\infty}^{x} \int_{-\infty}^{t} \,  d\alpha_{\R|\omega}(y) \, dt$. We get
\begin{align}
\begin{split} \label{eq:Psi_1}
    &-\int_{-\infty}^{\infty} \int_{-\infty}^{b} d\alpha_{\R|\omega}(t) \, \Psi_1(\omega,b) \, db \\ &= -\left[ \Psi_1(\omega,b) \int_{-\infty}^{b} \int_{-\infty}^{t} d\alpha_{\R|\omega}(y) \, dt  \right]_{b=-\infty}^{b=\infty} +  \int_{-\infty}^{\infty} \int_{-\infty}^{b} \int_{-\infty}^{t} d\alpha_{\R|\omega}(y) \, dt \, \partial_b \Psi_1(\omega,b) \, db \\ &=  \int_{-\infty}^{\infty} \int_{-\infty}^{b} \int_{-\infty}^{t} \, d\alpha_{\R|\omega}(y) \, dt \, \psi_1(\omega,b) \, db =   \int_{-\infty}^{\infty} \int_{-\infty}^{\infty} (b - t)_{+} \, d\alpha_{\R|\omega}(t) \, \psi_1(\omega,b) \, db \\ &= 
    \int_{-\infty}^{\infty} \int_{-\infty}^{\infty} (b - t)_{+} \, d\alpha_{\R|\omega}(t) \, \left((\mathcal{R} \phi)(\omega,b) - \beta_{\mathcal{R} \phi}(\omega) \eta(b) \right) \, db
\end{split}
\end{align}
The last equality holds because $\chi = \psi_1 + \beta_{\chi} \eta$ and in the second-to-last one, we used that 
\begin{align} 
\begin{split} \label{eq:alpha_lebesgue}
&\int_{-\infty}^{b} \int_{-\infty}^{t} d\alpha_{\R|\omega}(y) \, dt = \int_{-\infty}^{\infty} \mathds{1}_{t \leq b} \int_{-\infty}^{t} d\alpha_{\R|\omega}(y) \, dt \\ &= \left[-(b-t)_{+} \int_{-\infty}^{t} d\alpha_{\R|\omega}(y) \right]_{t=-\infty}^{t=\infty} + \int_{-\infty}^{\infty} (b - t)_{+} \, d\alpha_{\R|\omega}(t) = \int_{-\infty}^{\infty} (b - t)_{+} \, d\alpha_{\R|\omega}(t),
\end{split}
\end{align}
which holds again by \autoref{lem:lebesgue-stieltjes}, with $U(x) = \int_{-\infty}^{t} d\alpha_{\R|\omega}(y)$ and $V(x) = -(b-x)_{+}$, which has derivative $\frac{dV}{dx}(x) = \mathds{1}_{x \leq b}$. The last equality holds because $\liminf_{t \to -\infty} (b-t)_{+} \int_{-\infty}^{t} d\alpha_{\R|\omega}(y) = 0$, as $\alpha$ is supported on $\overline{\mathbb{P}^{d}_{\mathcal{U}}}$, which is bounded because $\mathcal{U}$ is bounded. 

\begin{lemma}[Lebesgue-Stieltjes integration by parts] \label{lem:lebesgue-stieltjes}
Given two functions $U$ and $V$ of finite variation, if at each point at least one of $U$ or $V$ is continuous, then an integration by parts formula for the Lebesgue–Stieltjes integral holds:
\begin{align} \int_{a}^{b}U\,dV+\int _{a}^{b}V\,dU=U(b+)V(b+)-U(a-)V(a-),\qquad -\infty <a<b<\infty,
\end{align}
where $U(b+), V(b+) = \lim_{r \to b^{+}} U(r), V(r)$ and $U(b-), V(b-) = \lim_{r \to b^{-}} U(r), V(r)$ are the right and left limits respectively, which exist at all points since $U$ and $V$ are of bounded variation.
\end{lemma}

If we plug \eqref{eq:Psi_1} into \eqref{eq:Psi_2} and then this back into \eqref{eq:first_term_Psi}, we obtain
\begin{align}
\begin{split} \label{eq:first_term_Psi_2}
    (-i)^{d-1} c_d \langle f,  \mathcal{R}^* \Lambda^{d+1} \Psi_2 \rangle &= \int_{\mathbb{S}^d} \int_{-\infty}^{\infty} \int_{-\infty}^{\infty} (b - t)_{+} \, d\alpha_{\R|\omega}(t) \, (\mathcal{R} \phi)(\omega,b) \, db \, d\alpha_{\mathbb{S}^{d-1}}(\omega) \\ &- \int_{\mathbb{S}^d} \int_{-\infty}^{\infty} \int_{-\infty}^{\infty} (b - t)_{+} \, d\alpha_{\R|\omega}(t) \, \beta_{\mathcal{R} \phi}(\omega) \eta(b) \, db \, d\alpha_{\mathbb{S}^{d-1}}(\omega) \\ &+ \int_{\mathbb{S}^d} \int_{-\infty}^{\infty} \int_{-\infty}^{b} d\alpha_{\R|\omega}(t) \, \gamma_{\mathcal{R} \phi}(\omega) \eta(b) \, db \, d\alpha_{\mathbb{S}^{d-1}}(\omega)
\end{split}
\end{align}
We develop the first term of \eqref{eq:first_term_Psi_2}:
\begin{align}
\begin{split} \label{eq:alpha_term}
    &\int_{\mathbb{S}^d} \int_{-\infty}^{\infty} \int_{-\infty}^{\infty} (b - t)_{+} \, d\alpha_{\R|\omega}(t) \, (\mathcal{R} \phi)(\omega,b) \, db \, d\alpha_{\mathbb{S}^{d-1}}(\omega) \\ &= \int_{\mathbb{S}^d} \int_{-\infty}^{\infty} \int_{-\infty}^{\infty} (b - t)_{+} \, d\alpha_{\R|\omega}(t) \, \int_{\{x | \langle \omega, x \rangle = b\}} \phi(x) \, dx \, db \, d\alpha_{\mathbb{S}^{d-1}}(\omega) \\ &= \int_{\mathbb{S}^d} \int_{-\infty}^{\infty} \int_{\{x | \langle \omega, x \rangle = b\}} \int_{-\infty}^{\infty} (\langle \omega, x \rangle - t)_{+} \,  d\alpha_{\R|\omega}(t) \, \phi(x) \, dx \, db \, d\alpha_{\mathbb{S}^{d-1}}(\omega)
    \\ &= \int_{\mathbb{S}^d} \int_{-\infty}^{\infty} \int_{-\infty}^{\infty} \int_{\{x | \langle \omega, x \rangle = b\}} (\langle \omega, x \rangle - t)_{+} \, \phi(x) \, dx \,  d\alpha_{\R|\omega}(t) \, db \, d\alpha_{\mathbb{S}^{d-1}}(\omega)
    \\ &= \int_{\mathbb{S}^d} \int_{-\infty}^{\infty} \int_{-\infty}^{\infty} \int_{\{x | \langle \omega, x \rangle = b\}} (\langle \omega, x \rangle - t)_{+} \, \phi(x) \, dx \, db \, d\alpha_{\R|\omega}(t) \, d\alpha_{\mathbb{S}^{d-1}}(\omega)
    \\ &= \int_{\mathbb{P}^d_{\mathcal{U}}} \int_{\R^d} (\langle \omega, x \rangle - t)_{+} \, \phi(x) \, dx \, d\alpha(\omega, t) \\ &= \int_{\mathcal{U}} \int_{\mathbb{P}^d_{\mathcal{U}}} (\langle \omega, x \rangle - t)_{+} \, d\alpha(\omega, t) \, \phi(x) \, dx.
\end{split}
\end{align}
In the third, fourth and sixth equalities we used Fubini's theorem. Before going any further, we express $\beta_{\mathcal{R} \phi}(\omega)$ and $\gamma_{\mathcal{R} \phi}(\omega)$ 
as integrals over $\mathcal{U}$ of some functions against $\phi$:
\begin{align}
\begin{split} \label{eq:alpha_phi}
    \beta_{\mathcal{R} \phi}(\omega) = \int_{-\infty}^{\infty} (\mathcal{R} \phi)(\omega,b) \, db = \int_{-\infty}^{\infty} \int_{\{x | \langle \omega, x \rangle = b\}} \phi(x) \, dx \, db = \int_{\mathcal{U}} \phi(x) \, dx.
\end{split}    
\end{align}
\begin{align}
\begin{split} \label{eq:beta_phi}
    \gamma_{\mathcal{R} \phi}(\omega) &= \int_{-\infty}^{\infty} \Psi_1(\omega,b) \, db = \int_{-\infty}^{\infty} \int_{-\infty}^{b} \psi_1(\omega,t) \, dt \, db \\ &= \left[ b \int_{-\infty}^{b} \psi_1(\omega,t) \, dt \right]_{b=-\infty}^{b=\infty} - \int_{-\infty}^{\infty} b \psi_1(\omega,b) \, db
    \\ &= - \int_{-\infty}^{\infty} b \left( (\mathcal{R} \phi)(\omega,b) - \beta_{\mathcal{R} \phi}(\omega) \eta(b) \right) \, db
    \\ &= - \int_{-\infty}^{\infty} b \int_{\{x | \langle \omega, x \rangle = b\}} \phi(x) \, dx \, db + \beta_{\mathcal{R} \phi}(\omega) \int_{-\infty}^{\infty} b  \eta(b) \, db
    \\ &= - \int_{-\infty}^{\infty} \int_{\{x | \langle \omega, x \rangle = b\}} \langle \omega, x \rangle \phi(x) \, dx \, db + \beta_{\mathcal{R} \phi}(\omega) \int_{-\infty}^{\infty} b  \eta(b) \, db
    \\ &= - \int_{\mathcal{U}} \langle \omega, x \rangle \phi(x) \, dx + \int_{\mathcal{U}} \phi(x) \, dx \int_{-\infty}^{\infty} b  \eta(b) \, db
\end{split}    
\end{align}
In the third equality we used \autoref{lem:lebesgue-stieltjes} with $U(b) = \int_{-\infty}^{b} \psi_1(\omega,t) \, dt$ and $V(b) = b$. 

We develop the second term of \eqref{eq:first_term_Psi_2}:
\begin{align}
\begin{split} \label{eq:second_term_11}
    &-\int_{\mathbb{S}^d} \int_{-\infty}^{\infty} \int_{-\infty}^{\infty} (b - t)_{+} \, d\alpha_{\R|\omega}(t) \, \beta_{\mathcal{R} \phi}(\omega) \, \eta(b) 
    \, db \, d\alpha_{\mathbb{S}^{d-1}}(\omega) \\ &= -\int_{\mathbb{S}^d} \int_{-\infty}^{\infty} \int_{-\infty}^{\infty} (b - t)_{+} \beta_{\mathcal{R} \phi}(\omega) \, \eta(b) 
    \, db \, d\alpha_{\R|\omega}(t) \, d\alpha_{\mathbb{S}^{d-1}}(\omega) \\ &= 
    -\int_{\mathbb{P}^d_{\mathcal{U}}} \int_{-\infty}^{\infty} (b - t)_{+} \beta_{\mathcal{R} \phi}(\omega) \, \eta(b) 
    \, db \, d\alpha(\omega,t) 
    = -\int_{\mathcal{U}} \phi(x) \, dx \int_{\mathbb{P}^d_{\mathcal{U}}} \int_{-\infty}^{\infty} (b - t)_{+} \, \eta(b) \, db \, d\alpha(\omega,t) 
\end{split}
\end{align}

We compute the third term of \eqref{eq:first_term_Psi_2}:
\begin{align}
\begin{split} \label{eq:third_term_11}
    &\int_{\mathbb{S}^d} \int_{-\infty}^{\infty} \int_{-\infty}^{b} d\alpha_{\R|\omega}(t) \, \gamma_{\mathcal{R} \phi}(\omega) \eta(b) \, db \, d\alpha_{\mathbb{S}^{d-1}}(\omega)
    \\ &= \int_{\mathbb{S}^d} \gamma_{\mathcal{R} \phi}(\omega) \left( \int_{-\infty}^{\infty} \, d\alpha_{\R|\omega}(b) - \int_{-\infty}^{\infty} \int_{-\infty}^{b} \eta(t) \, dt \, d\alpha_{\R|\omega}(b) \right) \, d\alpha_{\mathbb{S}^{d-1}}(\omega)
    \\ &= \int_{\tilde{\mathcal{U}}} \gamma_{\mathcal{R} \phi}(\omega) \left( 1 - \int_{-\infty}^{b} \eta(t) \, dt \right) \, d\alpha(\omega,b) 
    \\ &= \int_{\tilde{\mathcal{U}}} \left( - \int_{\mathcal{U}} \langle \omega, x \rangle \phi(x) \, dx + \int_{\mathcal{U}} \phi(x) \, dx \int_{-\infty}^{\infty} t  \eta(t) \, dt \right) \left( 1 - \int_{-\infty}^{b} \eta(t) \, dt \right) \, d\alpha(\omega,b) \\ &=
    -\int_{\mathcal{U}} \left\langle \int_{\tilde{\mathcal{U}}} \omega \left( 1 - \int_{-\infty}^{b} \eta(t) \, dt \right) \, d\alpha(\omega,b), x \right\rangle \phi(x) \, dx
    \\ &+ \int_{\mathcal{U}} \phi(x) \, dx \int_{-\infty}^{\infty} t  \eta(t) \, dt \int_{\tilde{\mathcal{U}}} \left( 1 - \int_{-\infty}^{b} \eta(t) \, dt \right) \, d\alpha(\omega,b)
\end{split}
\end{align}
In the first equality we used that by \autoref{lem:lebesgue-stieltjes} with $U(b) = \int_{-\infty}^{b} \, d\alpha_{\R|\omega}(t)$ and $V(b) = \int_{-\infty}^{b} \eta(t) \, dt$, and since $\int_{-\infty}^{\infty} \eta(t) \, dt = 1$
\begin{align}
\begin{split}
    \int_{-\infty}^{\infty} \int_{-\infty}^{b} d\alpha_{\R|\omega}(t) \, \eta(b) \, db &= \left[ \int_{-\infty}^{b} \, d\alpha_{\R|\omega}(t) \int_{-\infty}^{b} \eta(t) \, dt \right]_{b=-\infty}^{b=\infty} - \int_{-\infty}^{\infty} \int_{-\infty}^{b} \eta(t) \, d\alpha(t) \\ &= \int_{-\infty}^{\infty} \, d\alpha_{\R|\omega}(t) - \int_{-\infty}^{\infty} \int_{-\infty}^{b} \eta(t) \, d\alpha(t)
\end{split}
\end{align}

It remains to compute the second term in the right-hand side of \eqref{eq:f_psi}:
\begin{align}
\begin{split} \label{eq:second_term_10}
    &(-i)^{d-1} c_d \langle f,  \mathcal{R}^* \Lambda^{d-1} (\beta_{\mathcal{R} \phi} \, \eta + \gamma_{\mathcal{R} \phi} \, \partial_b \eta) \rangle =  (-i)^{d-1} c_d \langle f,  \mathcal{R}^* (\beta_{\mathcal{R} \phi} \, \Lambda^{d-1} \eta + \gamma_{\mathcal{R} \phi} \, \Lambda^{d-1} \partial_b \eta) \rangle \\ &= (-i)^{d-1} c_d \left\langle f,  \int_{\mathbb{S}^{d-1}} (\beta_{\mathcal{R} \phi}(\omega) \, (\Lambda^{d-1} \eta)(\langle \omega, \cdot \rangle) + \gamma_{\mathcal{R} \phi}(\omega) \, (\Lambda^{d-1} \partial_b \eta)(\langle \omega, \cdot \rangle)) \, d\omega \right\rangle
    \\ &= (-i)^{d-1} c_d \int_{\mathcal{U}} \phi(x) \, dx \left\langle f,  \int_{\mathbb{S}^{d-1}} (\Lambda^{d-1} \eta)(\langle \omega, \cdot \rangle) \, d\omega \right\rangle \\ &+ (-i)^{d-1} c_d \left\langle f,  \int_{\mathbb{S}^{d-1}} \left( - \int_{\mathcal{U}} \langle \omega, x \rangle \phi(x) \, dx + \int_{\mathcal{U}} \phi(x) \, dx \int_{-\infty}^{\infty} b  \eta(b) \, db \right) \, (\Lambda^{d-1} \partial_b \eta)(\langle \omega, \cdot \rangle) \, d\omega \right\rangle
    \\ &= (-i)^{d-1} c_d \int_{\mathcal{U}} \phi(x) \, dx \left\langle f,  \int_{\mathbb{S}^{d-1}} \left((\Lambda^{d-1} \eta)(\langle \omega, \cdot \rangle) + \int_{-\infty}^{\infty} b  \eta(b) \, db \, (\Lambda^{d-1} \partial_b \eta)(\langle \omega, \cdot \rangle) \right) \, d\omega \right\rangle \\ &- (-i)^{d-1} c_d \int_{\mathcal{U}} \left\langle \left\langle f,  \int_{\mathbb{S}^{d-1}}  \omega \, (\Lambda^{d-1} \partial_b \eta)(\langle \omega, \cdot \rangle) \, d\omega \right\rangle , x \right\rangle \phi(x) \, dx
\end{split}
\end{align}

Plugging \eqref{eq:alpha_term}, \eqref{eq:second_term_11} and \eqref{eq:third_term_11} into \eqref{eq:first_term_Psi}, and then plugging \eqref{eq:first_term_Psi} and \eqref{eq:second_term_10} into \eqref{eq:f_psi}, we obtain that \eqref{eq:f_psi_2} holds when we set
\begin{align}
\begin{split}
    v &= 
    -\int_{\tilde{\mathcal{U}}} \omega \left( 1 - \int_{-\infty}^{b} \eta(t) \, dt \right) \, d\alpha(\omega,b) - (-i)^{d-1} c_d \int_{\mathcal{U}} f(x) \int_{\mathbb{S}^{d-1}}  \omega \, (\Lambda^{d-1} \partial_b \eta)(\langle \omega, x \rangle) \, d\omega \, dx
    \\ c &= 
    -\int_{\mathbb{P}^d_{\mathcal{U}}} \int_{-\infty}^{\infty} (b - t)_{+} \, \eta(b) \, db \, d\alpha(\omega,t) +  \int_{-\infty}^{\infty} t  \eta(t) \, dt \int_{\tilde{\mathcal{U}}} \left( 1 - \int_{-\infty}^{b} \eta(t) \, dt \right) \, d\alpha(\omega,b) \\ &+ (-i)^{d-1} c_d \int_{\mathcal{U}} f(x)  \int_{\mathbb{S}^{d-1}} \left((\Lambda^{d-1} \eta)(\langle \omega, x \rangle) + \int_{-\infty}^{\infty} b  \eta(b) \, db \, (\Lambda^{d-1} \partial_b \eta)(\langle \omega, x \rangle) \right) \, d\omega \, dx.
\end{split}
\end{align}
\qed

\section{Proofs of \autoref{sec:sparse_approx}} \label{sec:sparse_approx_proofs}

\begin{lemma}[Hahn decomposition theorem] \label{lem:hahn_decomposition}
For any measurable space $(X,\Sigma )$ and any signed measure $\mu$ defined on the 
$\sigma$-algebra $\Sigma$, there exist two $\Sigma$-measurable sets, $P$ and $N$, of $X$ such that:
\begin{itemize}
\item $P\cup N=X$ and $P\cap N=\emptyset$.
\item For every $E\in \Sigma$ such that $E\subseteq P$, one has $\mu (E)\geq 0$, i.e., $P$ is a positive set for $\mu$.
\item For every $E\in \Sigma$ such that 
$E\subseteq N$, one has $\mu (E)\leq 0$, i.e., 
$N$ is a negative set for $\mu$.
\end{itemize}
\end{lemma}

\textbf{\textit{Proof of \autoref{thm:main_theorem_2}.}}
By Theorem~\ref{thm:main_theorem_1}, there exists $\alpha \in \mathcal{M}(\mathbb{P}^d_{\mathcal{U}})$ such that $\|\alpha\|_{\text{TV}} = \|f\|_{\mathcal{R},\mathcal{U}}$ and $f(x) = \int_{\mathbb{P}^d_{\mathcal{U}}} (\langle \omega, x \rangle - t)_{+} \, d\alpha(\omega, t) + \langle v, x \rangle + c$ for $x$ a.e. in $\mathcal{U}$. By the Hahn decomposition theorem (Lemma~\ref{lem:hahn_decomposition}), we can write $\alpha = \alpha_{+} - \alpha_{-}$, where $\alpha_{+}, \alpha_{-} \in \mathcal{M}_{+}(\mathbb{P}^d_{\mathcal{U}})$ are non-negative Radon measures. Also, there exist measurable sets $P,N \subseteq \mathbb{P}^d_{\mathcal{U}}$ such that for any measurable set $E$ included in $P$ (resp. $N$), $\mu_{-}(E) = 0$ (resp., $\mu_{+}(E) = 0$).

We define $|\alpha| = \alpha_{+} + \alpha_{-} \in \mathcal{M}_{+}(\mathbb{P}^d_{\mathcal{U}})$, and the normalized measure $\mu = |\alpha|/\|\alpha\|_{\text{TV}}$, which belongs to $\mathcal{P}(\mathbb{P}^d_{\mathcal{U}})$, the set of Borel probability measures over $\mathbb{P}^d_{\mathcal{U}}$.

Let $\{(\omega_i, b_i)\}_{i=1}^{n}$ be $n$ i.i.d. samples from $\mu$, and let $\hat{\mu}_{n} = \frac{1}{n} \sum_{i=1}^n \delta_{(\omega_i, b_i)}$. We construct the function 
\begin{align}
\begin{split}
    \tilde{f}(x) &= \frac{\|\alpha\|_{\text{TV}}}{n} \sum_{i=1}^n (\mathds{1}_{P}(\omega_i, b_i) - \mathds{1}_{N}(\omega_i, b_i)) (\langle \omega_i, x \rangle - b_i)_{+} + \langle v, x \rangle + c \\ &= \|\alpha\|_{\text{TV}} 
    \int_{\tilde{\mathcal{U}}} (\mathds{1}_{P}(\omega, b) - \mathds{1}_{N}(\omega, b)) (\langle \omega, x \rangle - b)_{+} \ d\hat{\mu}_{n}(\omega, b) + \langle v, x \rangle + c.
\end{split}
\end{align}
Notice that if we take the expectation over $\{(\omega_i, b_i)\}_{i=1}^{n}$, we have
\begin{align}
\begin{split}
    \mathbb{E}_{\{(\omega_i, b_i)\}_{i=1}^{n}} \tilde{f}(x) &= \|\alpha\|_{\text{TV}} 
    \int_{\tilde{\mathcal{U}}} (\mathds{1}_{P}(\omega, b) - \mathds{1}_{N}(\omega, b)) (\langle \omega, x \rangle - b)_{+} \ d\mu(\omega, b) + \langle v, x \rangle + c \\ &= 
    \int_{\tilde{\mathcal{U}}} (\mathds{1}_{P}(\omega, b) - \mathds{1}_{N}(\omega, b)) (\langle \omega, x \rangle - b)_{+} \ d|\alpha|(\omega, b) + \langle v, x \rangle + c \\ &= 
    \int_{\tilde{\mathcal{U}}} (\langle \omega, x \rangle - b)_{+} \ d(\alpha_{+} - \alpha_{-})(\omega, b) + \langle v, x \rangle + c \\ &= 
    \int_{\tilde{\mathcal{U}}} (\langle \omega, x \rangle - b)_{+} \ d\alpha(\omega, b) + \langle v, x \rangle + c = f(x),
\end{split}    
\end{align}
where the last equality holds almost everywhere.

To show the uniform bound, we want to upper-bound 
\begin{align}
\sup_{x \in \mathcal{U}} \left|
\int_{\tilde{\mathcal{U}}} (\mathds{1}_{P}(\omega, b) - \mathds{1}_{N}(\omega, b)) (\langle \omega, x \rangle - b)_{+} \ d(\hat{\mu}_{n}-\mu)(\omega, b) \right|.
\end{align}
We compute the Rademacher complexity of the class 
\begin{align}
\mathcal{F} = \{ f: 
\tilde{\mathcal{U}} \to \R, \ (\omega,b) \mapsto (\mathds{1}_{P}(\omega, b) - \mathds{1}_{N}(\omega, b)) (\langle \omega, x \rangle - b)_{+} \ | \ x \in \mathcal{U}\}.
\end{align}
\begin{align}
\begin{split}
    \mathcal{R}_n(\mathcal{F}) &= \mathbb{E}_{\sigma, \{(\omega_i, b_i)\}_{i=1}^{n}} \left[ \sup_{f \in \mathcal{F}} \frac{1}{n} \sum_{i=1}^{n} \sigma_i f(\omega_i, b_i) \right] \\ &=  \mathbb{E}_{\sigma, \{(\omega_i, b_i)\}_{i=1}^{n}} \left[ \sup_{x \in \mathcal{U}} \frac{1}{n} \sum_{i=1}^{n} \sigma_i (\mathds{1}_{P}(\omega_i, b_i) - \mathds{1}_{N}(\omega_i, b_i)) (\langle \omega_i, x \rangle - b_i)_{+} \right] \\ &= \mathbb{E}_{\sigma, \{(\omega_i, b_i)\}_{i=1}^{n}} \left[ \sup_{x \in \mathcal{U}} \frac{1}{n} \sum_{i=1}^{n} \sigma_i (\langle \omega_i, x \rangle - b_i)_{+} \right] \leq \mathbb{E}_{\sigma, \{(\omega_i, b_i)\}_{i=1}^{n}} \left[ \sup_{x \in \mathcal{U}} \frac{1}{n} \sum_{i=1}^{n} \sigma_i (\langle \omega_i, x \rangle - b_i) \right]
    \\ &\leq \mathbb{E}_{\sigma, \{\omega_i\}_{i=1}^{n}} \left[ \sup_{x \in \mathcal{B}_{R}(\R^d)} \frac{1}{n} \sum_{i=1}^{n} \sigma_i \langle \omega_i, x \rangle \right] = R \mathbb{E}_{\sigma, \{\omega_i\}_{i=1}^{n}} \left[  \bigg\|\frac{1}{n} \sum_{i=1}^{n} \sigma_i \omega_i \bigg\|_2 \right] \leq \frac{R}{\sqrt{n}}.
\end{split}
\end{align}
In the first inequality we used that $x \mapsto (x)_{+}$ is Lipschitz and Talagrand's lemma (Lemma 5.7, \cite{mohri2012foundations}). In the second inequality we used the assumption that $\mathcal{U} \subseteq \mathcal{B}_{R}(\R^d)$, and also that $\mathbb{E}_{\sigma} [\frac{1}{n} \sum_{i=1}^{n} \sigma_i b_i] = 0$. The second-to-last equality is by Cauchy-Schwarz and the last equality is by \cite{kakade2009onthe} (Sec. 3.1).
Thus, we have the following Rademacher complexity bound (\cite{mohri2012foundations}, Thm. 3.3): 
\begin{align}
    \mathbb{E}_{\{(\omega_i, b_i)\}_{i=1}^{n}} \left[ \sup_{x \in \mathcal{U}} \left|
    \int_{\tilde{\mathcal{U}}} (\mathds{1}_{P}(\omega, b) - \mathds{1}_{N}(\omega, b)) (\langle \omega, x \rangle - b)_{+} \ d(\hat{\mu}_{n}-\mu)(\omega, b) \right| \right] \leq \mathcal{R}_n(\mathcal{F}) \leq \frac{R}{\sqrt{n}}.
\end{align}
And this means there must exist some $\{(\omega_i, b_i)\}_{i=1}^{n}$ such that
\begin{align}
    \sup_{x \in \mathcal{U}} \left| \frac{\|\alpha\|_{\text{TV}}}{n} \sum_{i=1}^n (\mathds{1}_{P}(\omega_i, b_i) - \mathds{1}_{N}(\omega_i, b_i)) (\langle \omega_i, x \rangle - b_i)_{+} - 
    \int_{\tilde{\mathcal{U}}} (\langle \omega, x \rangle - b)_{+} \ d\alpha(\omega, b) \right| \leq \frac{R \|\alpha\|_{\text{TV}}}{\sqrt{n}},
\end{align}
which concludes the proof.
\qed

\begin{lemma}[Theorem 1 of \cite{klusowski2018approximation}] \label{lem:thm_1_klusowski}
If $f : [-1,1]^d \to \R$ admits an integral representation
\begin{align}
    f(x) = \kappa \int_{\{ \omega \in \R^d | \|\omega\|_1 = 1 \} \times [0,1]} \eta(\omega,b) (\langle \omega, x \rangle - b)_{+} \, d\mu(\omega,b)
\end{align}
for some probability measure $\mu$ and $\eta(\omega,b)$ taking values in $\{\pm 1\}$, then there exist $\{a_{i}\}_{i=1}^n \subseteq [-1,1]$, $\{\omega_i\}_{i=1}^n \subseteq \{ \omega \in \R^d | \|\omega\|_1 = 1 \}$ and $\{b_i\}_{i=1}^n \subseteq [0,1]$ such that the function 
\begin{align}
    \tilde{f}(x) = \frac{\kappa}{n} \sum_{i=1}^n a_i (\langle \omega_i, x \rangle - b_i)_{+}
\end{align}
fulfills
\begin{align}
    \sup_{x \in [-1,1]^d} |f(x) - \tilde{f}(x)| \leq c \kappa \sqrt{d + \log n} \, n^{-1/2 - 1/d},
\end{align}
for some universal constant $c > 0$.
\end{lemma}

\textbf{\textit{Proof of \autoref{prop:like_klusowski}.}}
By \autoref{thm:main_theorem_1}, $f$ admits an infinite-width neural network representation of the form \eqref{eq:f_psi_2} on $\mathcal{B}_1(\R^d)$. That is, for almost every $x \in \mathcal{B}_1(\R^d)$, we have that $f(x)$ is equal to
\begin{align}
\begin{split} \label{eq:klusowski_transformation}
    \hat{f}(x) &:= 
    \int_{\tilde{\mathcal{U}}} (\langle \omega, x \rangle - b)_{+} \, d\alpha(\omega,b) + \langle v, x \rangle + c \\ &= 
    \int_{\tilde{\mathcal{U}}} (\mathds{1}_{P}(\omega, b) - \mathds{1}_{N}(\omega, b)) (\langle \omega, x \rangle - b)_{+} \ d|\alpha|(\omega, b) + \langle v, x \rangle + c \\ &= 
    \int_{\tilde{\mathcal{U}} \cap \mathbb{S}^{d-1} \times [0,1)} (\mathds{1}_{P}(\omega, b) - \mathds{1}_{N}(\omega, b)) (\langle \omega, x \rangle - b)_{+} \ d|\alpha|(\omega, b) + \langle v', x \rangle + c' \\ &= \int_{\tilde{\mathcal{U}} \cap \mathbb{S}^{d-1} \times [0,1)} (\mathds{1}_{P}(\omega, b) - \mathds{1}_{N}(\omega, b)) \bigg( \left\langle \frac{\omega}{\|\omega\|_1}, x \right\rangle - \frac{b}{\|\omega\|_1} \bigg)_{+} \|\omega\|_1 \ d|\alpha|(\omega, b) + \langle v', x \rangle + c' \\ &= \int_{\{\omega | \|\omega\|_1 = 1\} \times [0,1)} (\langle \tilde{\omega}, x \rangle - \tilde{b})_{+} \, d\mu(\tilde{\omega},\tilde{b}) + \langle v', x \rangle + c'.
\end{split}
\end{align}
In the third equality we use that $(\langle - \omega, x \rangle + b)_{+} = (\langle \omega, x \rangle - b)_{+} - \langle \omega, x \rangle + t$, and that $|\alpha|$ is even, and we define $v' = v + \int_{\tilde{\mathcal{U}} \cap \mathbb{S}^{d-1} \times (-1,0)} (\mathds{1}_{P}(\omega, b) - \mathds{1}_{N}(\omega, b)) \omega \, d|\alpha|(\omega,b), c' = c + \int_{\tilde{\mathcal{U}} \cap \mathbb{S}^{d-1} \times (-1,0)} (\mathds{1}_{P}(\omega, b) - \mathds{1}_{N}(\omega, b)) t \, d|\alpha|(\omega,b)$. Let $T: \int_{\tilde{\mathcal{U}} \cap \mathbb{S}^{d-1} \times [0,1)} \to \{\omega | \|\omega\|_1 = 1\} \times [0,1)$ be the map $(\omega,b) \mapsto \left(\frac{\omega}{\|\omega\|_1},\frac{b}{\|\omega\|_1} \right)$, which is invertible. In the fifth equality, we define the Radon measure $\mu \in \mathcal{M}(\{\omega | \|\omega\|_1 = 1\} \times [0,1))$ as
\begin{align}
    \mu = T_{\#} \left[(\mathds{1}_{P}(\omega, b) - \mathds{1}_{N}(\omega, b)) \|\omega\|_1 |\alpha| \right],
\end{align}
where $T_{\#}$ denotes the pushforward by $T$. Since $T$ is surjective and the pushforward by a surjective map preserves the total variation norm, we have that the total variation norm of $\mu$ is
\begin{align}
    \|\mu\|_{\text{TV}} := \int_{\mathbb{P}^{d}_{\mathcal{U}}} \|\omega\|_1 d|\alpha|(\omega,\alpha) \leq \sqrt{d} \|\alpha\|_{\text{TV}} = \sqrt{d} \|f\|_{\mathcal{R}, \mathcal{U}},
\end{align}
where we used that by the Cauchy-Schwarz inequality, $\|\omega\|_1 \leq \sqrt{d}\|\omega\|_2$. Defining $\tilde{\mu} = |\mu|/\|\mu\|_{\text{TV}}$, the right-hand side of \eqref{eq:klusowski_transformation} becomes
\begin{align}
    \hat{f}(x) = \|\mu\| \int_{\{\omega | \|\omega\|_1 = 1\} \times [0,1)} \eta(\tilde{\omega}, \tilde{b}) (\langle \tilde{\omega}, x \rangle - \tilde{b})_{+} \, d\tilde{\mu}(\tilde{\omega},\tilde{b}) + \langle v, x \rangle + c,
\end{align}
where $\eta$ is a function taking values in $\{\pm 1\}$. Since t fulfills the assumptions of \autoref{lem:thm_1_klusowski}, we conclude that there exists $\tilde{f}$ of the prescribed form for which $\sup_{x \in [-1,1]^d} |f(x) - \tilde{f}(x)| \leq c \|\tau\|_{\text{TV}} \sqrt{d + \log n} \, n^{-1/2 - 1/d}$. Since $f$ and $\hat{f}$ coincide a.e. in $\mathcal{B}_1(\R_d)$, we obtain that the final result by setting $\kappa = \|\tau\|_{\text{TV}}$. 
\qed

\textbf{\textit{Proofs of \autoref{thm:fourier_refinement} and \autoref{thm:fourier_corrected}.}} We will prove the two theorems at once because the proofs are analogous: the only difference is that in \autoref{thm:fourier_refinement} we assume that $f$ admits a Fourier representation, while in \autoref{thm:fourier_corrected} we assume that it admits a representation of the form \eqref{eq:spherical_representation}. We will indicate the steps where the proofs differ.

By \autoref{prop:alpha_existence}, the measure $\alpha$ is the only one that satisfies $\langle \psi, \alpha \rangle = -c_d \langle f, (-\Delta)^{(d+1)/2} \mathcal{R}^{*}\psi \rangle$ for any $\psi \in \mathcal{S}(\mathbb{S}^{d-1} \times (-R,R))$. We develop this expression in different ways for $d$ odd and $d$ even.

\textit{Case $d$ odd:} We prove the case in which $d$ is odd first, which is simpler because $\Lambda^{d+1} \psi = \partial_b^{d+1} \psi$. By \autoref{eq:alpha_reexpressed}, we have
\begin{align} \label{eq:alpha_def_proof}
    - c_d \langle f, (-\Delta)^{(d+1)/2} \mathcal{R}^{*}\psi \rangle = (-i)^{d-1} c_d \langle f, \mathcal{R}^{*} \Lambda^{d+1} \psi \rangle = (-i)^{d-1} c_d \langle f, \mathcal{R}^{*} \partial_b^{d+1} \psi \rangle
\end{align}
Let $\psi \in \mathcal{S}(\mathbb{S}^{d-1} \times (-R,R)) \subseteq \mathcal{S}(\mathbb{S}^{d-1} \times \R)$ even. We have $\partial_b^{d+1} \psi \in \mathcal{S}(\mathbb{S}^{d-1} \times \R)$ which means that the Fourier inversion formula holds for $\partial_b^{d+1} \psi$. That is, if we define $\chi = \widehat{\partial_b^{d+1} \psi}$ to be the Fourier transform of $\Lambda^{d+1} \psi$, we can write $\Lambda^{d+1} \psi = \check{\chi}$. 
Since $\check{\chi}(\omega,b) = \frac{1}{(2\pi)^{1/2}} \int_{\R} e^{ibt} \chi(\omega,t) \, dt$, for the proof of \autoref{thm:fourier_refinement} the right-hand side of \eqref{eq:alpha_def_proof} becomes (up to constant factors)
\begin{align}
\begin{split} \label{eq:alpha_def_proof_2_fourier}
    \langle f, \mathcal{R}^{*} \check{\chi} \rangle &= \int_{\R^d} f(x) \int_{\mathbb{S}^{d-1}} \check{\chi}(\omega,\langle \omega, x \rangle) \, d\omega \, dx = \int_{\R^d} f(x) \int_{\mathbb{S}^{d-1}} \frac{1}{(2\pi)^{1/2}} \int_{\R} e^{i t \langle \omega, x \rangle} \chi(\omega,t) \, dt \, d\omega \, dx \\ &= \int_{\R^d} \frac{1}{(2\pi)^{d/2}}\int_{\R^d} e^{i \langle \xi, x \rangle} \, d\hat{f}(\xi) \int_{\mathbb{S}^{d-1}} \frac{1}{(2\pi)^{1/2}} \int_{\R} e^{i t \langle \omega, x \rangle} \chi(\omega,t) \, dt \, d\omega \, dx,
\end{split}    
\end{align}
where the last equality follows from the representation of the function $f$. For the proof of \autoref{thm:fourier_corrected} the analogous development reads
\begin{align}
\begin{split} \label{eq:alpha_def_proof_2}
    \langle f, \mathcal{R}^{*} \check{\chi} \rangle &= \int_{\R^d} f(x) \int_{\mathbb{S}^{d-1}} \check{\chi}(\omega,\langle \omega, x \rangle) \, d\omega \, dx = \int_{\R^d} f(x) \int_{\mathbb{S}^{d-1}} \frac{1}{(2\pi)^{1/2}} \int_{\R} e^{i t \langle \omega, x \rangle} \chi(\omega,t) \, dt \, d\omega \, dx \\ &= \int_{\R^d} \int_{\mathbb{S}^{d-1} \times \R} e^{i \hat{t} \langle \hat{\omega}, x \rangle} \, d\mu(\hat{\omega},\hat{t}) \int_{\mathbb{S}^{d-1}} \frac{1}{(2\pi)^{1/2}} \int_{\R} e^{i t \langle \omega, x \rangle} \chi(\omega,t) \, dt \, d\omega \, dx.
\end{split}    
\end{align}
At this point, note that $\chi$ is even because
\begin{align}
    \chi(\omega,b) &= \frac{1}{(2\pi)^{1/2}} \int_{\R} (\partial_b^{d+1} \psi)(\omega,t) e^{-ibt} \, dt = \frac{1}{(2\pi)^{1/2}} \int_{\R} (\partial_b^{d+1} \psi)(-\omega,-t) e^{-ibt} \, dt \\ &= \frac{1}{(2\pi)^{1/2}} \int_{\R} (\partial_b^{d+1} \psi)(-\omega,t) e^{ibt} \, dt = \chi(-\omega,-b),
\end{align}
where we used that $\partial_b^{d+1} \psi$ is even, which follows from $\psi$ being even. Consequently,
\begin{align}
\begin{split}
    &\int_{\mathbb{S}^{d-1}}\int_{\R} e^{i t \langle \omega, x \rangle} \chi(\omega,t) \, dt = \int_{\mathbb{S}^{d-1}}\int_{\R^{+}} e^{i t \langle \omega, x \rangle} \chi(\omega,t) \, dt + \int_{\mathbb{S}^{d-1}}\int_{\R^{-}} e^{i t \langle \omega, x \rangle} \chi(\omega,t) \, dt \\ &= \int_{\mathbb{S}^{d-1}}\int_{\R^{+}} e^{i t \langle \omega, x \rangle} \chi(\omega,t) \, dt + \int_{\mathbb{S}^{d-1}}\int_{\R^{+}} e^{-i t \langle \omega, x \rangle} \chi(\omega,-t) \, dt \\ &= \int_{\mathbb{S}^{d-1}}\int_{\R^{+}} e^{i t \langle \omega, x \rangle} \chi(\omega,t) \, dt + \int_{\mathbb{S}^{d-1}}\int_{\R^{+}} e^{i t \langle -\omega, x \rangle} \chi(-\omega,t) \, dt = 2 \int_{\mathbb{S}^{d-1}}\int_{\R^{+}} e^{i t \langle \omega, x \rangle} \chi(\omega,t) \, dt
\end{split}    
\end{align}
Applying the change of variables from spherical coordinates to Euclidean coordinates (\autoref{lem:euclidean_spherical}), we may write
\begin{align}
    \forall x \in \mathbb{R}^d, \quad \int_{\mathbb{S}^{d-1}}\int_{\R} e^{i t \langle \omega, x \rangle} \chi(\omega,t) \, dt = \int_{\R^d} e^{i \langle y, x \rangle} \tilde{\chi}(y) \, dy,
\end{align}
where we define $\tilde{\chi} : \R^d \to \R$ as
\begin{align}
    \tilde{\chi}(y) = \frac{2 \chi(y/\|y\|,\|y\|)}{\|y\|^{d-1}}.
\end{align}
Thus, for the proof of \autoref{thm:fourier_refinement}, the right-hand side of \eqref{eq:alpha_def_proof} can be written as
\begin{align}
\begin{split} \label{eq:long_development_fourier}
    &(-i)^{d-1} c_d \frac{1}{(2\pi)^{1/2}} \int_{\R^d} \frac{1}{(2\pi)^{d/2}}\int_{\R^d} e^{i \langle \xi, x \rangle} \, d\hat{f}(\xi) \, \int_{\R^d} e^{i \langle y, x \rangle} \tilde{\chi}(y) \, dy \, dx \\ &= (-i)^{d-1} c_d \frac{1}{(2\pi)^{1/2}} \frac{1}{(2\pi)^{d/2}} \int_{\R^d} \int_{\R^d} e^{i \langle \xi, x \rangle} \int_{\R^d} e^{i \langle y, x \rangle} \tilde{\chi}(y) \, dy \, dx \, d\hat{f}(\xi) \\ &= (-i)^{d-1} c_d (2\pi)^{d-1/2} \frac{1}{(2\pi)^{d/2}} \int_{\R^d} \frac{1}{(2\pi)^{d/2}} \int_{\R^d} e^{i \langle \xi, x \rangle} \frac{1}{(2\pi)^{d/2}} \int_{\R^d} e^{-i \langle y, x \rangle} \tilde{\chi}(-y) \, dy \, dx \, d\hat{f}(\xi) \\ &= (-i)^{d-1} c_d (2\pi)^{d-1/2} \frac{1}{(2\pi)^{d/2}} \int_{\R^d} \tilde{\chi}(-\xi) \, d\hat{f}(\xi) = (-i)^{d-1} c_d 2(2\pi)^{(d-1)/2} \int_{\mathbb{S}^{d-1} \times \R} \frac{\chi(-\xi/\|\xi\|,\|\xi\|)}{ \|\xi\|^{d-1}} \, d\hat{f}(\xi) \\ &= (-i)^{d-1} c_d 2(2\pi)^{d-1/2} \frac{1}{(2\pi)^{d/2}} \int_{\R^d} \frac{\widehat{\Lambda^{d+1} \psi}(-\xi/\|\xi\|,\|\xi\|)}{ \|\xi\|^{d-1}} \, d\hat{f}(\xi) \\ &= (-i)^{d-1} c_d 2(2\pi)^{d-1/2} \frac{1}{(2\pi)^{d/2}} \int_{\R^d} \frac{i^{d+1} \|\xi\|^{d+1} \hat{\psi}(-\xi/\|\xi\|,\|\xi\|)}{ \|\xi\|^{d-1}} \, d\hat{f}(\xi) \\ &= (-i)^{d-1} c_d 2(2\pi)^{d-1/2} i^{d+1} \frac{1}{(2\pi)^{d/2}} \int_{\R^d} \|\xi\|^{2} \frac{1}{(2\pi)^{1/2}} \int_{\R} e^{i\|\xi\|b} \psi(-\xi/\|\xi\|,b) \, db \, d\hat{f}(\xi) \\ &= - \frac{1}{(2\pi)^{d/2}} \int_{\R^d} \|\xi\|^{2} \int_{\R} e^{-i\|\xi\|b} \psi(\xi/\|\xi\|,b) \, db \, d\hat{f}(\xi)
\end{split}    
\end{align}
In the sixth equality we used \autoref{eq:fourier_transform_ramp}. In the last equality we used that $\psi$ is even, which means that
$\int_{\R} e^{i\|\xi\|b} \psi(-\xi/\|\xi\|,b) \, db = \int_{\R} e^{i\|\xi\|b} \psi(\xi/\|\xi\|,-b) \, db = \int_{\R} e^{-i\|\xi\|b} \psi(\xi/\|\xi\|,b) \, db$.
And analogously for the proof of \autoref{thm:fourier_corrected}, the right-hand side of \eqref{eq:alpha_def_proof_2} becomes
\begin{align}
\begin{split} \label{eq:long_development}
    &(-i)^{d-1} c_d \frac{1}{(2\pi)^{1/2}} \int_{\R^d} \int_{\mathbb{S}^{d-1} \times \R} e^{i \tilde{t} \langle \tilde{\omega}, x \rangle} \, d\mu(\tilde{\omega},\tilde{t}) \, \int_{\R^d} e^{i \langle y, x \rangle} \tilde{\chi}(y) \, dy \, dx \\ &= (-i)^{d-1} c_d \frac{1}{(2\pi)^{1/2}} \int_{\mathbb{S}^{d-1} \times \R} \int_{\R^d} e^{i \tilde{t} \langle \tilde{\omega}, x \rangle} \int_{\R^d} e^{i \langle y, x \rangle} \tilde{\chi}(y) \, dy \, dx \, d\mu(\tilde{\omega},\tilde{t}) \\ &= (-i)^{d-1} c_d (2\pi)^{d-1/2} \int_{\mathbb{S}^{d-1} \times \R} \frac{1}{(2\pi)^{d/2}} \int_{\R^d} e^{i \langle \tilde{t} \tilde{\omega}, x \rangle} \frac{1}{(2\pi)^{d/2}} \int_{\R^d} e^{-i \langle y, x \rangle} \tilde{\chi}(-y) \, dy \, dx \, d\mu(\tilde{\omega},\tilde{t}) \\ &= (-i)^{d-1} c_d (2\pi)^{d-1/2} \int_{\mathbb{S}^{d-1} \times \R} \tilde{\chi}(-\tilde{t} \tilde{\omega}) \, d\mu(\tilde{\omega},\tilde{t}) = (-i)^{d-1} c_d 2(2\pi)^{(d-1)/2} \int_{\mathbb{S}^{d-1} \times \R} \frac{\chi(-\tilde{\omega},\tilde{t})}{ |\tilde{t}|^{d-1}} \, d\mu(\tilde{\omega},\tilde{t}) \\ &= (-i)^{d-1} c_d 2(2\pi)^{d-1/2} \int_{\mathbb{S}^{d-1} \times \R} \frac{\widehat{\Lambda^{d+1} \psi}(-\tilde{\omega},\tilde{t})}{ |\tilde{t}|^{d-1}} \, d\mu(\tilde{\omega},\tilde{t}) \\ &= (-i)^{d-1} c_d 2(2\pi)^{d-1/2} \int_{\mathbb{S}^{d-1} \times \R} \frac{i^{d+1} |\tilde{t}|^{d+1} \hat{\psi}(-\tilde{\omega},\tilde{t})}{ |\tilde{t}|^{d-1}} \, d\mu(\tilde{\omega},\tilde{t}) \\ &= (-i)^{d-1} c_d 2(2\pi)^{d-1/2} i^{d+1} \int_{\mathbb{S}^{d-1} \times \R} \tilde{t}^{2} \frac{1}{(2\pi)^{1/2}} \int_{\R} e^{i\tilde{t}b} \psi(-\tilde{\omega},b) \, db \, d\mu(\tilde{\omega},\tilde{t}) \\ &= (-i)^{d-1} c_d 2(2\pi)^{d-1/2} i^{d+1} \int_{\mathbb{S}^{d-1} \times \R} \tilde{t}^{2} \frac{1}{(2\pi)^{1/2}} \int_{\R} e^{-i\tilde{t}b} \psi(\tilde{\omega},b) \, db \, d\mu(\tilde{\omega},\tilde{t})
    \\ &= - \int_{\R} \int_{\mathbb{S}^{d-1}} \psi(\tilde{\omega},b) \int_{\R} \tilde{t}^{2} e^{-i\tilde{t}b} \, d\mu(\tilde{\omega},\tilde{t}) \, db.
\end{split}    
\end{align}

\textit{Case $d$ even:} When $d$ is even, we write 
\begin{align} \label{eq:alpha_def_proof_even}
    - c_d \langle f, (-\Delta)^{(d+1)/2} \mathcal{R}^{*}\psi \rangle &= - c_d \langle f, (-\Delta)^{1/2} (-\Delta)^{d/2} \mathcal{R}^{*}\psi \rangle = - c_d \langle (-\Delta)^{1/2} f, (-\Delta)^{d/2} \mathcal{R}^{*}\psi \rangle \\ &= - (-i)^{d} c_d \langle (-\Delta)^{1/2} f, \mathcal{R}^{*} \partial_b^{d} \psi \rangle.
\end{align}
The second equality holds by the definition of the fractional Laplacian $(-\Delta)^s$ for tempered distributions\footnote{We denote by $\mathcal{S}'(\R^{d})$ the dual space of $\mathcal{S}(\R^d)$, which is known as the space of tempered distributions on $\R^{d}$. Functions that grow no faster than polynomials can be embedded in $\mathcal{S}'(\R^{d})$ by defining $\langle g, \phi \rangle := \int_{\R^{d}} \varphi(x) g(x) \ dx$ for any $\varphi \in \mathcal{S}(\R^{d})$.}: if $T \in \mathcal{S}'(\R^d)$, for any $\phi \in \mathcal{S}(\R^d)$ we have $\langle (-\Delta)^s T, \phi \rangle = \langle T, (-\Delta)^s \phi \rangle$, and because $(-\Delta)^{d/2} \mathcal{R}^{*}\psi \in \mathcal{S}(\R^d)$ as both $\mathcal{R}^{*}$ and $(-\Delta)^{d/2}$ map $\mathcal{S}(\R^d)$ into itself (Theorem 2.5, \cite{helgason1994geometric}). The third equality holds by \autoref{lem:intertwining}. Reusing the argument for the even case, this time taking $\chi = \widehat{\partial_b^{d} \psi}$ and $\check{\chi} = \partial_b^{d} \psi$, for the proof of \autoref{thm:fourier_refinement} the right-hand side of \eqref{eq:alpha_def_proof_even} is equal to 
\begin{align}
\begin{split} \label{eq:long_development_fourier_even}
    &- (-i)^{d} c_d \int_{\R^d} (-\Delta)^{1/2} \left(\frac{1}{(2\pi)^{d/2}}\int_{\R^d} e^{i \langle \xi, x \rangle} \, d\hat{f}(\xi) \right) \int_{\mathbb{S}^{d-1}} \frac{1}{(2\pi)^{1/2}} \int_{\R} e^{i t \langle \omega, x \rangle} \chi(\omega,t) \, dt \, d\omega \, dx \\ &= - (-i)^{d} c_d \int_{\R^d} \frac{1}{(2\pi)^{d/2}}\int_{\R^d} e^{i \langle \xi, x \rangle} \|\xi\| \, d\hat{f}(\xi) \int_{\mathbb{S}^{d-1}} \frac{1}{(2\pi)^{1/2}} \int_{\R} e^{i t \langle \omega, x \rangle} \chi(\omega,t) \, dt \, d\omega \, dx \\ &= - (-i)^{d} c_d 2(2\pi)^{d-1/2} \frac{1}{(2\pi)^{d/2}} \int_{\R^d} \frac{\widehat{\Lambda^{d} \psi}(-\xi/\|\xi\|,\|\xi\|)}{ \|\xi\|^{d-1}} \|\xi\| \, d\hat{f}(\xi) \\ &= - (-i)^{d} c_d 2(2\pi)^{d-1/2} \frac{1}{(2\pi)^{d/2}} \int_{\R^d} \frac{i^{d} \|\xi\|^{d} \hat{\psi}(-\xi/\|\xi\|,\|\xi\|)}{\|\xi\|^{d-1}} \|\xi\| \, d\hat{f}(\xi) \\ &= - \frac{1}{(2\pi)^{d/2}} \int_{\R^d} \|\xi\|^{2} \int_{\R} e^{-i\|\xi\|b} \psi(\xi/\|\xi\|,b) \, db \, d\hat{f}(\xi).
\end{split}
\end{align}
Note that the only difference with respect to the odd case is the factor $i^{d}$ instead of $i^{d+1}$.
And for the proof of \autoref{thm:fourier_corrected}, we have analogously:
\begin{align}
\begin{split} \label{eq:long_development_even}
    &- (-i)^{d} c_d \int_{\R^d} (-\Delta)^{1/2} \left(\int_{\mathbb{S}^{d-1} \times \R} e^{i \tilde{t} \langle \tilde{\omega}, x \rangle} \, d\mu(\tilde{\omega},\tilde{t}) \right) \int_{\mathbb{S}^{d-1}} \frac{1}{(2\pi)^{1/2}} \int_{\R} e^{i t \langle \omega, x \rangle} \chi(\omega,t) \, dt \, d\omega \, dx \\ &= - (-i)^{d} c_d \int_{\R^d} \int_{\mathbb{S}^{d-1} \times \R} e^{i \tilde{t} \langle \tilde{\omega}, x \rangle} \, |\tilde{t}| d\mu(\tilde{\omega},\tilde{t}) \int_{\mathbb{S}^{d-1}} \frac{1}{(2\pi)^{1/2}} \int_{\R} e^{i t \langle \omega, x \rangle} \chi(\omega,t) \, dt \, d\omega \, dx \\ &= - (-i)^{d} c_d 2(2\pi)^{d-1/2} \int_{\mathbb{S}^{d-1} \times \R} \frac{\widehat{\Lambda^{d} \psi}(-\tilde{\omega},\tilde{t})}{ |\tilde{t}|^{d-1}} |\tilde{t}| \, d\mu(\tilde{\omega},\tilde{t}) \\ &= - (-i)^{d} c_d 2(2\pi)^{d-1/2} \int_{\mathbb{S}^{d-1} \times \R} \frac{i^{d} |\tilde{t}|^{d} \hat{\psi}(-\tilde{\omega},\tilde{t})}{|\tilde{t}|^{d-1}} |\tilde{t}| \, d\mu(\tilde{\omega},\tilde{t}) \\ &= - \int_{\R} \int_{\mathbb{S}^{d-1}} \psi(\tilde{\omega},b) \int_{\R} \tilde{t}^{2} e^{-i\tilde{t}b} \, d\mu(\tilde{\omega},\tilde{t}) \, db.
\end{split}
\end{align}

\textit{Conclusion of the proof of \autoref{thm:fourier_refinement}:} Since $c_d = \frac{1}{2(2\pi)^{d-1}}$, from \eqref{eq:long_development_fourier} and \eqref{eq:long_development_fourier_even} we obtain
\begin{align}
\begin{split}
    \|f\|_{\mathcal{R}, \mathcal{B}_R(\R^d)} &= \sup_{\psi \in \mathcal{S}(\mathbb{P}^{d}_{R}), \|\psi\|_{\infty} \leq 1} \langle \alpha, \psi \rangle \\ &= \sup_{\psi \in \mathcal{S}(\mathbb{P}^{d}_{R}), \|\psi\|_{\infty} \leq 1}  -\frac{1}{(2\pi)^{d/2}} \int_{\R^d} \|\xi\|^{2} \int_{-R}^{R} e^{-i\|\xi\|b} \psi(\xi/\|\xi\|,b) \, db \, d\hat{f}(\xi) \\ &\leq \sup_{\psi \in \mathcal{S}(\mathbb{P}^{d}_{R}), \|\psi\|_{\infty} \leq 1} \frac{1}{(2\pi)^{d/2}} \int_{\R^d} \|\xi\|^{2} \int_{-R}^{R} |\psi(\xi/\|\xi\|,b)| \, db \, d|\hat{f}|(\xi) \\ &\leq 2R \sup_{\psi \in \mathcal{S}(\mathbb{P}^{d}_{R}), \|\psi\|_{\infty} \leq 1} \frac{1}{(2\pi)^{d/2}} \int_{\R^d} \|\xi\|^{2} \, d|\hat{f}|(\xi) = 2 R C_f. 
\end{split}
\end{align}
This concludes the proof.

\textit{Conclusion of the proof of \autoref{thm:fourier_corrected}:} 
Since $c_d = \frac{1}{2(2\pi)^{d-1}}$, from \eqref{eq:long_development} and \eqref{eq:long_development_even} we obtain that
\begin{align}
    \langle \psi, \alpha \rangle = -
\int_{\R} \int_{\mathbb{S}^{d-1}} \psi(\omega,b) 
\int_{\R} t^{2} e^{-i t b} \, d\mu(\omega,t) 
\, db.
\end{align} 
This concludes the proof.
\qed

\begin{lemma}[Change from Euclidean to spherical coordinates] \label{lem:euclidean_spherical} 
Let $\int_{\mathbb{S}^{d-1} \times \R} \phi(\omega,b) \, d(\omega,b)$ denote the integral of the integrable even function $\phi : \mathbb{S}^{d-1} \times \R \to \R$ with respect to the Lebesgue measure on $\mathbb{S}^{d-1} \times \R$. Then,
\begin{align}
    \int_{\mathbb{S}^{d-1} \times \R} \phi(\omega,b) \, d(\omega,b) = \int_{\R^d} \frac{2\phi(x/\|x\|,\|x\|)}{\|x\|^{d-1}} \, dx
\end{align}
Conversely, if $g : \R^d \to \R$ is an integrable function, we have
\begin{align}
    \int_{\R^d} g(x) \, dx = \frac{1}{2} \int_{\mathbb{S}^{d-1} \times \R} g(b\omega) |b|^{d-1} \, d(\omega,b). 
\end{align}
\end{lemma}
\begin{proof}
The Jacobian determinant of the change from Euclidean to spherical coordinates takes the form $(\omega,b) \mapsto C |b|^{d-1}$ for some constant $C$. Consider the function $x \mapsto \mathds{1}_{\|x\| \leq R}$ for an arbitrary $R > 0$. We have
\begin{align}
    \text{vol}(\mathcal{B}_R(\R^d)) &= \int_{\R^d} \mathds{1}_{\|x\| \leq R} \, dx = C \int_{\mathbb{S}^{d-1} \times \R} \mathds{1}_{b \leq R} b^{d-1} \, d(\omega,b) \\ &= C \int_{\R} \mathds{1}_{b \leq R} |b|^{d-1} \int_{\mathbb{S}^{d-1}} \, d\omega \, db = C \int_{-R}^R |b|^{d-1} \text{vol}(\mathbb{S}^{d-1}) \, db \\ &= C \int_{-R}^R \text{vol}(\mathbb{S}_b^{d-1}) \, db = 2 C \int_{0}^R \text{vol}(\mathbb{S}_b^{d-1}) \, db = 2 C \text{vol}(\mathcal{B}_R(\R^d))
\end{align}
This implies that $C=1/2$.
\end{proof}

\begin{lemma}[Fourier transform of the ramp filter $\Lambda^s$] \label{eq:fourier_transform_ramp}
Let $\psi \in \mathcal{S}(\R)$. Then, $\widehat{\Lambda^{s} \psi}(\xi) = \frac{i^{s}}{(2\pi)^{1/2}} |\xi|^{s} \hat{\psi}(\xi)$.
\end{lemma}
\begin{proof}
We reproduce the proof for completeness, but the result is ubiquitous. When $s$ is even, we have $\widehat{\Lambda^{s} \psi}(\xi) = \widehat{\partial^s \psi}(\xi) = (-i\xi)^s \hat{\psi}(\xi) = i^s |\xi|^s \hat{\psi}(\xi)$. When $s$ is odd, we have $\widehat{\Lambda^{s} \psi}(\xi) = \widehat{\mathcal{H} \partial^s \psi}(\xi) = -i\text{sign}(\xi) \widehat{\partial^s \psi}(\xi) = -i\text{sign}(\xi) (-i\xi)^s \hat{\psi}(\xi) = i^s |\xi|^s \hat{\psi}(\xi)$.
\end{proof}

\textbf{\textit{Proof of \autoref{ex:one_dim}.}}
The Fourier representation $f(x) = \frac{1}{(2\pi)^{1/2}}\int_{\R} \sqrt{\frac{\pi}{2}} (\delta_{1}(\xi) + \delta_{-1}(\xi) - \delta_{1+\epsilon}(\xi) - \delta_{-1-\epsilon}(\xi)) e^{i \xi x} \, d\xi$ holds because we may write $\cos(\xi_0 x) = \frac{1}{2} \int_{\R} (\delta_{\xi_0}(\xi) + \delta_{-\xi_0}(\xi)) e^{i\xi x} \, d\xi$.
From the definition of $C_f$ in \autoref{subsec:existing_sparse}, we have
\begin{align}
\begin{split}
    C_f &= \frac{1}{(2\pi)^{1/2}} \int_{\R} \sqrt{\frac{\pi}{2}} |\xi|^2 (\delta_{1}(\xi) + \delta_{-1}(\xi) + \delta_{1+\epsilon}(\xi) + \delta_{-1-\epsilon}(\xi)) \, d\xi \\ &= \frac{1}{2} (2+ 2 (1+\epsilon)^2) = 2 + 2\epsilon + \epsilon^2.
\end{split}
\end{align}
On the other hand, $f(x) = \frac{1}{4} \int_{\mathbb{S}^0} \int_{\R}  (\delta_{1}(b) + \delta_{-1}(b) + \delta_{1+\epsilon}(b) + \delta_{-1-\epsilon}(b)) e^{ib \omega x} \, db \, d\omega$ the representation \eqref{eq:spherical_representation} holds for $f$ with $d\mu(\omega,b) = \frac{1}{4} (\delta_{1}(b) + \delta_{-1}(b) + \delta_{1+\epsilon}(b) + \delta_{-1-\epsilon}(b)) \, db \, d\omega$. By equation \eqref{eq:norm_alpha_fourier} we have  
\begin{align}
\begin{split}
    \|f\|_{\mathcal{R},\mathcal{B}_R(\R^d)} &= \frac{1}{4} \int_{-R}^{R} \int_{\mathbb{S}^{0}} \left|
    \int_{\R} t^2 e^{-itb} (\delta_{1}(t) + \delta_{-1}(t) + \delta_{1+\epsilon}(t) + \delta_{-1-\epsilon}(t)) \, dt 
    \right| \, d\omega \, db \\ &= \int_{-R}^R \left|
    \int_{\R} \frac{t^2}{2} \left( \delta_{1}(t) + \delta_{-1}(t) - \delta_{1+\epsilon}(t) - \delta_{-1-\epsilon}(t) \right) e^{-i t b} \, dt 
    \right|\, db \\ &= \int_{-R}^R \left|f''(b) \right|\, db = \int_{-R}^R \left|-\cos(x)  + (1+\epsilon)^2 \cos((1+\epsilon)x) \right|\, db \\ &\leq \int_{-R}^R \left|\cos((1+\epsilon)x)-\cos(x)\right| + (2\epsilon + \epsilon^2)|\cos((1+\epsilon)x)|\, db \\ &= \int_{-R}^R \epsilon |x| \left|\sin(\tilde{x})\right| + (2\epsilon + \epsilon^2)|\cos((1+\epsilon)x)|\, db \\ &\leq R^2 \epsilon + 2R(2\epsilon + \epsilon^2) = R \left( R \epsilon + 4 \epsilon + 2\epsilon^2 \right).
\end{split}
\end{align}
\qed

\textbf{\textit{Proof of \autoref{prop:finite_width}.}} 
We show the bound on $\|f\|_{\mathcal{R},\mathcal{U}}$ first. Comparing \eqref{eq:r_norm_def} with \eqref{eq:R_U_norm} we obtain that $\|f\|_{\mathcal{R},\mathcal{U}} \leq \|f\|_{\mathcal{R}}$ for all $f: \R^d \to \R$ and all $\mathcal{U} \subseteq \R^d$ bounded open. Theorem 1 of \cite{ongie2019function} states that $\|f\|_{\mathcal{R}} = \bar{R}_1(f)$, where
\begin{align}
    \bar{R}_1(f) = \min_{\substack{\alpha \in \mathcal{M}(\mathbb{S}^{d-1} \times \R), \\ v \in \R^d, c \in \R}} \left\{ \|\alpha\|_{\text{TV}} \ | \ \forall x \in \R^d, \ f(x) = \int_{\mathbb{S}^{d-1} \times \R} (\langle \omega, x \rangle - b)_{+} \, d\alpha(\omega,b) + \langle v, x \rangle + c \right\}.
\end{align}
By assumption, for all $x \in \R^d$ we have $f(x) = \int_{\mathbb{S}^{d-1} \times \R} (\langle \omega, x \rangle - b)_{+} d\tilde{\alpha}(\omega,b)$ where $\tilde{\alpha} = \frac{1}{n} \sum_{i=1}^{n} a_i \delta_{(\omega_i,b_i)}$, and $\|\alpha\|_{\text{TV}} = \frac{1}{n} \sum_{i=1}^{n} |a_i|$. Thus, $\bar{R}_1(f) \leq \frac{1}{n} \sum_{i=1}^{n} |a_i|$. Putting everything together we get that $\|f\|_{\mathcal{R},\mathcal{U}} \leq \frac{1}{n} \sum_{i=1}^{n} |a_i|$.

To show that $C_f, \tilde{C}_f = +\infty$, we will show that $C_f < +\infty$ implies that $f$ is $C_1$, which is an argument that we take from \cite{e2020representation}, p. 3. If $\int_{\R^d} \, d|\hat{f}|(\xi) < +\infty$ and $\int_{\R^d} \|\xi\|^2 \, d|\hat{f}|(\xi) < +\infty$, we have that $\int_{\R^d} \|\xi\| \, d|\hat{f}|(\xi) \leq \frac{1}{2} \int_{\R^d} (1+\|\xi\|^2) \, d|\hat{f}|(\xi) < +\infty$. And for any $i \in \{1, \dots, d\}$, $\xi_i \hat{f}$ is the Fourier transform of $\partial_i f$, which means that $\int_{\R^d} \, d|\widehat{\partial_i f}|(\xi) < +\infty$. Hence, $\partial_i f$ is continuous, or equivalently, $f \in C^1(\R^d)$. Since a finite-width neural network which is not an affine function does not belong to $C^1(\R^d)$, we have that $C_f = +\infty$. 
\qed

\section{Proofs of \autoref{sec:not_unique}}
\label{sec:app_not_unique}

\subsection{Preliminaries}

\textit{Notation:} Throughout this section, $\mathbb{H}_k^d$ denotes the space of homogeneous polynomials of degree $k$ on $\R^d$, and $\mathbb{Y}_{k}^d$ denotes the space of harmonic homogeneous polynomials of degree $k$ on $\R^d$. $\{Y_{k,j} \ | \ 1 \leq j \leq N_{k,d}\}$ is an orthonormal basis of $\mathbb{Y}_{k}^d$, where for $d \geq 2$,
\begin{align}
N_{k,d} = \frac{(2k + d - 2)(k+d-3)!}{k!(d-2)!}
\end{align}
The sequence of Legendre polynomials in dimension $d$ is the sequence ${(P_{k,d})}_{k \geq 0}$ of polynomials on $[-1,1]$ such that $\int_{-1}^{1} P_{k,d}(x) P_{k',d}(x) \, dx = \delta_{k,k'}$.

\begin{lemma}[Theorem 2.18, \cite{atkinson2012spherical}] \label{lem:homogeneous_decomposition}
We have
\begin{align} \label{eq:full_decomposition}
    \mathbb{H}_{n}^d = \mathbb{Y}_{n}^d \oplus \| \cdot \|^2 \mathbb{Y}_{n-2}^d \oplus \cdots \oplus \|\cdot\|^{2[n/2]} \mathbb{Y}_{n-2[n/2]}^d,
\end{align}
where $\oplus$ denotes the direct sum of vector spaces. In particular,
\begin{align}
    \mathbb{H}_{n}^d \rvert_{\mathbb{S}^{d-1}} = \mathbb{Y}_{n}^d \rvert_{\mathbb{S}^{d-1}} \oplus \mathbb{Y}_{n-2}^d \rvert_{\mathbb{S}^{d-1}} \oplus \cdots \oplus \mathbb{Y}_{n-2[n/2]}^d \rvert_{\mathbb{S}^{d-1}}.
\end{align}
\end{lemma}

\begin{lemma}[Funk-Hecke formula, \cite{atkinson2012spherical}] \label{lem:funk-hecke}
If $\eta$ is a measurable function on $(-1,1)$ such that $\int_{-1}^1 |\eta(t)| (1-t^2)^{(d-3)/2} dt < +\infty$ (in particular, if $f \in C([-1,1])$ for $d \geq 2$), for any spherical harmonic $Y \in \mathbb{Y}_{k}^d \rvert_{\mathbb{S}^{d-1}}$, we have for any $\omega \in \mathbb{S}^{d-1}$,
\begin{align}
    \int_{\mathbb{S}^{d-1}} \eta(\langle \omega, x \rangle) Y(x) \, dx = |\mathbb{S}^{d-2}| Y(\omega) \int_{-1}^{1} \eta(t) P_{k,d}(t) (1-t^2)^{(d-3)/2} \, dt.
\end{align}
\end{lemma}

\begin{lemma}[Schwartz theorem for the Radon transform, Theorem 2.4 of \cite{helgason2011integral}] \label{lem:schwartz}
The Radon transform $f \mapsto \mathcal{R} f$ is a linear one-to-one mapping of $\mathcal{S}(\R^d)$ onto $\mathcal{S}^H(\mathbb{P}^d)$, where
\begin{align}
    \mathcal{S}^H(\mathbb{P}^d) = \left\{ \psi \in \mathcal{S}(\mathbb{P}^d) \ \bigg| \ \forall k \in \mathbb{Z}^{+}, \ \int_{\R} \psi(\omega,b) b^k \, db \text{ is a homogeneous polynomial of degree } k \right\}.
\end{align}
\end{lemma}

\begin{lemma}[Support theorem for the Radon transform, Theorem 2.6 of \cite{helgason2011integral}] \label{lem:support_radon}
Let $f \in C(\R^d)$ satisfy the following conditions:
\begin{itemize}
    \item For each integer $k > 0$, $|x|^k f(x)$ is bounded.
    \item There exists a constant $A > 0$ such that $(\mathcal{R} f)(\omega,b) = 0$ for all $(\omega,b) \in \mathbb{S}^{d-1} \times \R$ such that $b > A$.
\end{itemize}
Then, $f(x) = 0$ for all $x$ such that $\|x\|_2 > A$.
\end{lemma}

\begin{lemma}[\cite{helgason2011integral}, Ch. I, Lemma 5.1] \label{lem:prop22helgason}
Assume that $\phi \in C_c(\R^d)$ and $\psi$ is locally integrable on $\mathbb{P}^d$. Then,
\begin{align}
    \int_{\mathbb{P}^d} (\mathcal{R} \phi)(\omega,b) \psi(\omega,b) \, d(\omega,b) = \int_{\mathbb{R}^d} \phi(x) (\mathcal{R}^{*} \psi)(x) \, dx.
\end{align}
\end{lemma}

\begin{lemma}[Adaptation of \cite{helgason1994geometric}, Proposition 2.7] \label{lem:prop27helgason}
Let $\phi \in C^{\infty}(\overline{\mathcal{B}_R(\R^d)})$. Then, we have $\phi(\omega,b) = \sum_{k=0}^{\infty} \phi_k(x),$
where the convergence is in the topology of $C^{\infty}(\overline{\mathcal{B}_R(\R^d)})$ (the topology of uniform convergence for derivatives of all orders), and the functions $\phi_k$ are of the form
\begin{align}
    \phi_k(r\omega) = \sum_{j=1}^{N_{k,d}} \phi_{k,j}(r) Y_{k,j}(\omega), \quad \text{where} \quad \psi_{k,j}(r) = \int_{\mathbb{S}^{d-1}} \psi(r\omega) Y_{k,j}(\omega) \, d\omega.
\end{align}
\end{lemma}

\begin{lemma}[Adaptation of \cite{helgason1994geometric}, Proposition 2.8] \label{lem:prop28helgason}
Let $\psi \in C^{\infty}(\overline{\mathbb{P}_R^d})$. Then, we have $\psi(\omega,b) = \sum_{k=0}^{\infty} \psi_k(\omega,b)$, 
where the convergence is in the topology of $C^{\infty}(\overline{\mathbb{P}_R^d})$, and the functions $\psi_k$ are of the form
\begin{align}
    \psi_k(\omega,b) = \sum_{j=1}^{N_{k,d}} \psi_{k,j}(b) Y_{k,j}(\omega), \quad \text{where} \quad \psi_{k,j}(b) = \int_{\mathbb{S}^{d-1}} \psi(\omega,b) Y_{k,j}(\omega) \, d\omega,
\end{align}
and $\psi_{k,j}(-b) = (-1)^{k} \psi_{k,j}(b)$.
\end{lemma}

\begin{lemma}[\cite{rudin1991functional}, Theorem 1.21] \label{lem:finite_dimensional_closed}
All finite dimensional subspaces of Hausdorff topological vector spaces are closed.
\end{lemma}

\begin{lemma} \label{lem:annihilator_sum}
Let $U,V$ be subspaces of a vector space $\mathcal{E}$ endowed with a bilinear form. Let $\tilde{U}$ be the annihilator of $U$ within $V$, that is, $\tilde{U} = \{ v \in V \ | \ \forall u \in U, \ \langle v, u \rangle = 0 \}$. Then, $\tilde{U}$ is the annihilator of $U$ within $V + \tilde{U}$. 
\end{lemma}
\begin{proof}
Since $V \subseteq V + \tilde{U}$, $\tilde{U}$ is included in the annihilator of $U$ within $V + \tilde{U}$. To show the reverse inclusion, note that an arbitrary element $x$ of the annihilator of $U$ within $V + \tilde{U}$ may be written as $x = v + \tilde{u}$, where $v \in V, \tilde{u} \in \tilde{U}$. For any $u \in U$, we have $0 = \langle x, u \rangle = \langle v + \tilde{u}, u \rangle = \langle v, u \rangle$, where we used that $\tilde{u} \in \tilde{U}$. This implies that $v$ belongs to the annihilator in the annihilator of $U$ within $V$, i.e. $v \in \tilde{U}$. Hence, we obtain that $x \in \tilde{U}$. 
\end{proof}

\begin{lemma} \label{lem:direct_sum_annihilator}
Let $U$ be a finite-dimensional subspace of a (possibly infinite-dimensional) vector space $\mathcal{E}$ endowed 
with an inner product $\langle \cdot, \cdot \rangle$. Let $U^{\perp}$ be the annihilator of $U$ within $\mathcal{E}$. Then, one can write
\begin{align}
    \mathcal{E} = U \oplus U^{\perp}.
\end{align}
\end{lemma}
\begin{proof}
Let $(u_1, \cdots, u_m)$ be an orthonormal basis of $U$ according to $\langle \cdot, \cdot \rangle$, which may be constructed according to the Gram-Schmidt procedure. Then, one can define the orthogonal projection $P_{U}$ onto $U$ as $P_{U} x = \sum_{i=1}^{m} \langle x, u_i \rangle u_i$. Note that for any $x \in \mathcal{E}$, $P_{U} x \in U$. Moreover, $x = P_{U} x + x - P_{U} x$, and $x - P_{U} x \in U^{\perp}$ because for any $u' \in U$, one can write $u' = \sum_{i=1}^{m} \langle u', u_i \rangle u_i$ and consequently, $\langle u', x - P_{U} x \rangle = \sum_{i=1}^{m} \langle u', u_i \rangle \langle u_i, x \rangle - \sum_{i=1}^{m} \langle x, u_i \rangle \langle u', u_i \rangle = 0$.
\end{proof}

\subsection{Proof of \autoref{thm:main_thm_non_unique}}

\begin{proposition} \label{prop:equals_zero_span_A}
Let $A = \{ Y_{k,j} \otimes X^{k'} \ | \ k, j, k' \in \mathbb{Z}^{+}, \ k \equiv k' \ (\text{mod } 2), \ k' < k-2, \ 1 \leq j \leq N_{d,k} \}$. If $Y_{k,j} \otimes X^{k'} \in A$, then, 
\begin{align}
    \forall x \in \mathcal{B}_R(\R^d), \quad \int_{\mathbb{S}^{d-1} \times (-R,R)} (\langle \omega, x \rangle - b)_{+} (Y_{k,j} \otimes X^{k'})(\omega,b) \, d(\omega,b) = 0.
\end{align}
Moreover,
\begin{align}
\text{for } x \text{ a.e. in } \mathcal{B}_R(\R^d), \quad \int_{\mathbb{S}^{d-1} \times (-R,R)} (\langle \omega, x \rangle - b)_{+} (Y_{k,j} \otimes X^{k-2})(\omega,b) \, d(\omega,b) \neq 0.
\end{align}
\end{proposition}
\begin{proof}
Note that 
\begin{align}
    \int_{\mathbb{S}^{d-1} \times (-R,R)} (\langle \omega, x \rangle - b)_{+} (Y_{k,j} \otimes X^{k'})(\omega,b) \, d(\omega,b) = \int_{\mathbb{S}^{d-1}} \int_{-R}^{R} (\langle \omega, x \rangle - b)_{+} b^{k'} \, db \, Y_{k,j}(\omega) \, d\omega,
\end{align}
and since for $x \in \mathcal{B}_R(\R^d)$ we have $\langle \omega, x \rangle \in (-R,R)$ for any $\omega \in \mathbb{S}^{d-1}$, by \autoref{lem:second_derivative_relu} we have
\begin{align}
\begin{split}
    \int_{-R}^{R} (\langle \omega, x \rangle - b)_{+} b^{k'} \, db &= \frac{1}{(k'+1)(k'+2)} (\langle \omega, x \rangle^{k'+2} - (-R)^{k'+2}) - \frac{1}{k'+1} (-R)^{k'+1} (\langle \omega, x \rangle + R) \\ &= Q_{k'+2}(\langle \omega, x \rangle)
\end{split}
\end{align}
using the fact that $\frac{d^2}{db^2}(\frac{1}{(k'+1)(k'+2)} b^{k'+2}) = \frac{d^2}{db^2}(\frac{1}{k'+1} b^{k'+1}) = b^{k'}$ and defining the degree $k'+2$ polynomial $Q_{k'+2}$ appropriately. Consequently,
\begin{align} \label{eq:rhs_zero}
    \int_{\mathbb{S}^{d-1}} \int_{-R}^{R} (\langle \omega, x \rangle - b)_{+} b^{k'} \, db \, Y_{k,j}(\omega) \, d\omega = \int_{\mathbb{S}^{d-1}} Q_{k'+2}(\langle \omega, x \rangle) \, Y_{k,j}(\omega) \, d\omega.
\end{align}
For any $x \in \R^d$, $Q_{k'+2}(\langle \omega, x \rangle)$ is a polynomial on $\omega \in \mathbb{S}^{d-1}$ of degree at most $k'+2$ (the degree is $k'+2$ if $x \neq 0$ and zero otherwise). That is, 
\begin{align}
Q_{k'+2}(\langle \cdot, x \rangle) \in \bigoplus_{i=0}^{k'+2} \mathbb{H}_{i}^d \rvert_{\mathbb{S}^{d-1}} = \bigoplus_{i=0}^{k'+2} \mathbb{Y}_{i}^d \rvert_{\mathbb{S}^{d-1}},
\end{align}
where the last equality is by \autoref{lem:homogeneous_decomposition} applied to each $\mathbb{H}^d_i \rvert_{\mathbb{S}^{d-1}}$. 
Since any pair of spaces $\mathbb{Y}_{i}^d \rvert_{\mathbb{S}^{d-1}}$ and $\mathbb{Y}_{i'}^d \rvert_{\mathbb{S}^{d-1}}$ are orthogonal when $i \neq i'$, we have that $Y_{k,j} \in \mathbb{Y}_{k}^d \rvert_{\mathbb{S}^{d-1}}$ must be orthogonal to $Q_{k'+2}(\langle \omega, x \rangle)$ as $k' + 2 < k$. This means that the right-hand side of \eqref{eq:rhs_zero} is zero.

To deal with the case of $Y_{k,j} \otimes X^{k-2}$, we reproduce the same argument, but in this case $Q_{k}(\langle \cdot, x \rangle)$ has a term of degree $k$ that survives:
\begin{align}
\begin{split}
    &\int_{\mathbb{S}^{d-1}} Q_{k}(\langle \omega, x \rangle) \, Y_{k,j}(\omega) \, d\omega = \int_{\mathbb{S}^{d-1}} \frac{1}{k(k-1)} \langle \omega, x \rangle^{k} \, Y_{k,j}(\omega) \, d\omega \\ &= \frac{\|x\|^k}{k(k-1)} \int_{\mathbb{S}^{d-1}} \langle \omega, x/\|x\| \rangle^{k} \, Y_{k,j}(\omega) \, d\omega \\ &= \frac{\|x\|^k |\mathbb{S}^{d-2}|}{k(k-1)} Y_{k,j}(x/\|x\|) \int_{-1}^{1} t^k P_{k,d}(t) (1-t^2)^{(d-3)/2} \, dt.
\end{split}
\end{align}
In the second equality we used that $x \neq 0$ (in the case $x = 0$ the result is zero). In the third equality we used the Funk-Hecke formula (\autoref{lem:funk-hecke}). Note that the sequence $(P_{k'',d})_{k'' \geq 0}$ of Legendre polynomials in dimension $d$ are orthogonal with respect to the inner product $\langle f, g \rangle = \int_{-1}^{1} f(t) g(t) (1-t^2)^{(d-3)/2} \, dt$. Consequently, $\int_{-1}^{1} t^k P_{k,d}(t) (1-t^2)^{(d-3)/2} \, dt$ is proportional to the coefficient of $X^k$ in the basis $(P_{k'',d})_{k'' \geq 0}$, which is non-zero. Finally, $Y_{k,j}$ being a non-zero polynomial on $\R^d$, $Y_{k,j}(x/\|x\|)$ is non-zero almost everywhere.
\end{proof}

\begin{lemma} \label{lem:second_derivative_relu}
Let $f \in C^2((-R,R))$. For any $x \in (-R,R)$, we have
\begin{align}
\begin{split}
    f(x) &= f(-R) + (x+R)f'(-R) + \int_{-R}^{R} (x-t)_{+} f''(t) \, dt, \\
    f(x) &= f(R) + (x-R)f'(-R) + \int_{-R}^{R} (t-x)_{+} f''(t) \, dt.
\end{split}
\end{align}
\end{lemma}
\begin{proof}
By the fundamental theorem of calculus,
\begin{align}
\begin{split} \label{eq:tilde_f_fundamental_thm}
    &f(x) = f(-R) + \int_{-R}^{x} f'(t) \, dt,
    \\ &f(x) = f(R) - \int_{x}^{R} f'(t) \, dt.
\end{split}    
\end{align}
And integrating by parts, we get that for any $x \in (-R,R)$,
\begin{align}
\begin{split} \label{eq:tilde_f_parts}
    \int_{-R}^{x} f'(t) \, dt &= \int_{-R}^{R} \mathds{1}_{t \leq x} f'(t) \, dt = \left[-(x - t)_{+} f'(t) \right]_{t=-R}^{t=R} + \int_{-R}^{R} (x - t)_{+} f''(t) \, dt \\ &= (x + R) f'(-R) + \int_{-R}^{R} (x - t)_{+} f''(t) \, dt, 
\end{split}
\end{align}
In the last equality we used that $(x - R)_{+} = 0, (x + R)_{+} = x + R$ for all $x \in (-R,R)$. Similarly,
\begin{align}
\begin{split} \label{eq:tilde_f_parts_2}
    -\int_{x}^{R} f'(t) \, dt &= -\int_{-R}^{R} \mathds{1}_{t \geq x} f'(t) \, dt = \left[-(t - x)_{+} f'(t) \right]_{t=-R}^{t=R} + \int_{-R}^{R} (t - x)_{+} f''(t) \, dt \\ &= (x - R) f'(R) + \int_{-R}^{R} (t - x)_{+} f''(t) \, dt.
\end{split}    
\end{align}
\end{proof}

\textbf{\textit{Proof of \autoref{thm:main_thm_non_unique}.}}
By \autoref{prop:equals_zero_span_A}, for any $h \in \text{span}(A)$, we have
\begin{align} \label{eq:h_span_A_zero}
    \forall x \in \mathcal{B}_R(\R^d), \quad \int_{\mathbb{S}^{d-1} \times (-R,R)} (\langle \omega, x \rangle - b)_{+} h(\omega,b) \, d(\omega,b) = 0.
\end{align}
If $\alpha \in \mathcal{M}(\mathbb{S}^{d-1} \times (-R,R))$ belongs to the closure $\text{cl}_{\text{w}}(\text{span}(A))$ in the topology of weak convergence, there exists a sequence ${(h_k)}_{k \geq 0} \subseteq \text{span}(A)$ such that for any $\psi \in C_b(\mathbb{S}^{d-1} \times (-R,R)) 
$, $\int_{\mathbb{S}^{d-1} \times (-R,R)} \psi(\omega,b) h_k(\omega,b) \, d(\omega,b) \to \int_{\mathbb{S}^{d-1} \times (-R,R)} \psi(\omega,b) \, d\alpha(\omega,b)$. Since $(\omega, b) \mapsto (\langle \omega, x \rangle - b)_{+}$ is in $C_b(\mathbb{S}^{d-1} \times (-R,R))$ for any $x \in \R^d$, we have
\begin{align}
    \int_{\mathbb{S}^{d-1} \times (-R,R)} (\langle \omega, x \rangle - b)_{+} \, d\alpha(\omega,b) = \lim_{k \to \infty} \int_{\mathbb{S}^{d-1} \times (-R,R)} (\langle \omega, x \rangle - b)_{+} h_k(\omega,b) \, d(\omega,b) = 0,
\end{align}
where the second equality is by \eqref{eq:h_span_A_zero}. This concludes the proof.
\qed

\subsection{Proof of \autoref{prop:characterization_alpha}}

To prove \autoref{prop:characterization_alpha}, we use two main intermediate results: 
\vspace{-5pt}
\begin{itemize}[leftmargin=5.5mm]
    \item \autoref{lem:alpha_Y_kj_X_kprime}, which proves that \eqref{eq:alpha_characterization_2} holds. Its proof uses \autoref{lem:g_existence}, which in turn builds on \autoref{lem:h_existence}.
    \item \autoref{cor:C_0_span_B_R_S}, which proves that $\text{span}(B) \oplus \mathcal{R}(\mathcal{S}(\mathcal{B}_R(\R^d)))$ is dense in $C_0(\mathbb{P}_R^{d})$ and thus implies that \eqref{eq:alpha_characterization_1} and \eqref{eq:alpha_characterization_2} specify a unique measure. \autoref{cor:C_0_span_B_R_S} is a corollary of \autoref{prop:I_N} (which makes use of \autoref{lem:commutation_R_k}), and also relies on \autoref{lem:R_k_inclusion}. Both \autoref{prop:I_N} and \autoref{cor:C_0_span_B_R_S} use \autoref{prop:radon_set_equality_2}.
\end{itemize}
\vspace{5pt}

\begin{proposition} \label{lem:alpha_Y_kj_X_kprime}
Let $\alpha \in \mathcal{M}(\mathbb{P}^d_R)$ as in \autoref{prop:alpha_existence}. Suppose that $\|f\|_{\mathcal{R}, \mathcal{B}_{R'}(\R^d)}$ is finite for some $R' > R$. For any $k, j, k' \in \mathbb{Z}^{+}$ such that $k' \equiv k \ (\text{mod } 2)$, we have
\begin{align}
    \langle \alpha, Y_{k,j} \otimes X^{k'} \rangle = - c_d \langle f, (-\Delta)^{(d+1)/2} \mathcal{R}^{*} (Y_{k,j} \otimes \mathds{1}_{|X|<R} X^{k'}) \rangle.
\end{align}
The quantity $\langle f, (-\Delta)^{(d+1)/2} \mathcal{R}^{*} (Y_{k,j} \otimes \mathds{1}_{|X|<R} X^{k'}) \rangle$ is well defined despite $\mathds{1}_{|X|<R} X^{k'}$ not being continuous on $\R$; we define it as $\langle f, (-\Delta)^{(d+1)/2} \mathcal{R}^{*} (Y_{k,j} \otimes  (\tilde{g} +\mathds{1}_{|X|<R} X^{k'})) \rangle$, where $\tilde{g}$ is any function in $\mathcal{S}((-R,R))$ such that $\tilde{g} + \mathds{1}_{|X|<R} X^{k'} \in \mathcal{S}((-R,R))$.
\end{proposition}
\begin{proof}
Take $Y_{k,j} \otimes X^{k'}$ as in the statement of the lemma. Let $N = \max\{k', k''\}$ and define ${(b_n)}_{n=0}^{N}$ as $b_n = -\frac{d^n}{dx^n}(x^{k'}) \rvert_{x = R}$. Note that higher-order derivatives are zero. For $\epsilon, \delta > 0$, let $g : \R \to \R$ be the function described in \autoref{lem:g_existence} corresponding to the sequence ${(b_n)}_{n=0}^{N}$ and define $\tilde{g} : \R \to \R$ as $\tilde{g}(x) = g(-x - R) + g(x-R)$. 
We define the function $\tilde{\psi} : \mathbb{P}^d \to \R$ as $\tilde{\psi} = Y_{k,j} \otimes (\tilde{g} + \mathds{1}_{|X|<R}X^{k'})$. Note that $\tilde{\psi} \in \mathcal{S}(\mathbb{P}^d_{R+\epsilon+\delta})$, as both $\mathds{1}_{|X|<R}X^{k'}$ and $\tilde{g}$ take value zero outside of $(-R-\epsilon-\delta,R+\epsilon+\delta)$. Let $\tilde{\alpha} \in \mathcal{M}(\mathbb{P}^d_{R+\epsilon+\delta})$ be the measure from \autoref{prop:alpha_existence} for $\mathcal{U} = \mathcal{B}_{R+\epsilon+\delta}(\R^d)$, and note that by definition, $\langle \tilde{\alpha}, \psi \rangle = \langle \alpha, \psi \rangle$ for any $\psi \in C_0(\mathbb{P}^d_{R})$. Thus, we have that
\begin{align}
\begin{split}
    &\langle \alpha, Y_{k,j} \otimes \mathds{1}_{|X| \leq R}X^{k'} \rangle = \langle \tilde{\alpha}, Y_{k,j} \otimes \mathds{1}_{|X| \leq R} X^{k'} \rangle \\ = &\langle \tilde{\alpha}, \tilde{\psi} \rangle - \langle \tilde{\alpha}, Y_{k,j} \otimes \tilde{g} \rangle = - c_d \langle f, (-\Delta)^{(d+1)/2} \mathcal{R}^{*} \tilde{\psi} \rangle - c_d \langle f, (-\Delta)^{(d+1)/2} \mathcal{R}^{*} (Y_{k,j} \otimes \tilde{g}) \rangle 
\end{split}
\end{align}
The first equality holds because $\langle \tilde{\alpha}, \psi \rangle = \langle \alpha, \psi \rangle$ for any $\psi \in C_0(\mathbb{P}^d_{R})$. The second and fourth equalities are by the definition of $\psi$. The third equality holds by the definition of $\tilde{\alpha}$ and because $\tilde{\psi} \in \mathcal{S}(\mathbb{P}^d_{R+\epsilon+\delta})$. The fifth equality holds by the definition of $\tilde{\alpha}$ and because $Y_{k,j} \otimes \tilde{g} \in \mathcal{S}(\mathbb{P}^d_{R+\epsilon+\delta})$. 
\end{proof}

\begin{lemma} \label{lem:g_existence}
Let ${(b_n)}_{n=0}^{N} \subseteq \R$. For any $\epsilon, \delta > 0$, there exists a function $g : \R \to \R$ which is $C^{\infty}$ such that $g(x) = 0$ for $\|x\| \geq \epsilon + \delta$, and such that $\frac{d^n}{dx^n} g(x) = b_n$ for $n=\{0,\dots,N\}$, $\frac{d^n}{dx^n} g(x) = 0$ for $n > N$. 
\end{lemma}
\begin{proof}
Let $h : (-\infty,0]$ be the function described in \autoref{lem:h_existence} such that $(\frac{d^n h}{dx^n})(0) = b_n$ for $n=\{0,\dots,N\}$. Let $\tilde{h} : (-\infty,0]$ be the function described in \autoref{lem:h_existence} such that $(\frac{d^n \tilde{h}}{dx^n})(0) = (-1)^n b_n$ for $n=\{0,\dots,N\}$. We define $g$ piecewise in terms of $h$ and $\tilde{h}$:
\begin{align}
    g(x) = 
    \begin{cases}
    h(x) \quad &\text{if } x \geq 0, \\
    \tilde{h}(-x) \quad &\text{if } x > 0.
    \end{cases}
\end{align}
Note that $g$ defined in this form is $C^{\infty}$ as $(\frac{d^n}{dx^n} (\tilde{h} \circ (t \mapsto -t)))(x) = (-1)^n \tilde{h}(-x)$, which means that $(\frac{d^n}{dx^n} (\tilde{h} \circ (t \mapsto -t)))(0) = (-1)^n \tilde{h}(0) = (-1)^{2n} b_n = b_n = (\frac{d^n h}{dx^n})(0)$. 
\end{proof}

\begin{lemma} \label{lem:h_existence}
For any $\epsilon, \delta > 0$, define $\xi_{\epsilon,\delta} : (-\infty,0] \to \R$ as a $C^{\infty}$ function such that $\xi_{\epsilon,\delta}(x) = 0$ for $x \leq - (\epsilon + \delta)$, $\xi_{\epsilon,\delta}(x) = 1$ for $x \geq -\epsilon$, and $\xi_{\epsilon,\delta}$ is non-decreasing. Define $\partial^{-1} \xi_{\epsilon,\delta}$ to be the antiderivative of $\xi_{\epsilon,\delta}$ taking value 0 at $-\infty$, i.e. $(\partial^{-1} \xi_{\epsilon,\delta}) (x) = \int_{-\epsilon - \delta}^{x} \xi_{\epsilon,\delta} (t) \, dt$. Similarly, for all $n \geq 0$, let $\partial^{-(n+1)} \xi_{\epsilon,\delta} (x) = \int_{-\epsilon - \delta}^{x} (\partial^{-n} \xi_{\epsilon,\delta}) (t) \, dt$. For $\epsilon, \delta > 0$ small enough, given a sequence $(b_n)_{n=0}^{N} \subseteq \R$, there exist coefficients $(a_n)_{n=0}^{N} \subseteq \R$ such that the function $h$ on $(-\infty,0]$ defined as
\begin{align}
    h(x) = \sum_{n=0}^N a_n (\partial^{-n} \xi_{\epsilon,\delta})(x)
\end{align}
satisfies $(\frac{d^n h}{dx^n})(0) = b_n$ for $n = \{0, \cdots, N\}$ and $(\frac{d^n h}{dx^n})(0) = 0$ for $n > N$. In the limit $\delta \to 0$, for $x \in (-\epsilon,0]$ we have $(\partial^{-n} \xi_{\epsilon,\delta}) (x) = \frac{1}{n!}(x+\epsilon)^n + O(\delta)$. Moreover, in the limits $\epsilon, \delta \to 0$, we have $a_n = b_n + O(\epsilon + \delta)$.
\end{lemma}
\begin{proof}
First, we show that for $n=\{0,\dots,N\}$ and $x \in (-\epsilon-\delta,0]$, $(\partial^{-n} \xi_{\epsilon,\delta}) (x) = \frac{1}{n!}(x+\epsilon)^n + O(\delta)$. This is shown by proving that $\frac{1}{n!}(x+\epsilon)^n \leq (\partial^{-n} \xi_{\epsilon,\delta}) (x) \leq \frac{1}{n!}(x+\epsilon+\delta)^n$ by finite induction: the holds for $n=0$ trivially and for the general case one integrates both bounds.

Consequently,
\begin{align}
\begin{split}
    \frac{d^n h}{dx^n}(x) &= \sum_{n'=n}^N a_{n'} (\partial^{n-n'} \xi_{\epsilon,\delta})(x) = \sum_{n'=0}^{N-n} a_{n'+n} (\partial^{-n'} \xi_{\epsilon,\delta})(x) = \sum_{n'=0}^{N-n} \frac{ a_{n'+n}}{(n')!}(x+\epsilon)^{n'} + O(\delta), \\
    \implies \frac{d^n h}{dx^n}(0) &= \sum_{n'=0}^{N-n} \frac{ a_{n'+n}}{(n')!} \epsilon^{n'} + O(\delta).
\end{split}
\end{align}
Thus, imposing that $(\frac{d^n h}{dx^n})(0) = b_n$ for $n = \{0, \cdots, N\}$ we have that $a_n = b_n - \sum_{n'=1}^{N-n} \frac{ a_{n'+n}}{(n')!} \epsilon^{n'} +  O(\delta) = b_n + O(\epsilon + \delta)$. 
\end{proof}

\begin{lemma} \label{prop:radon_set_equality_2}
For any $k \in \mathbb{Z}^{+}$, let $\mathbb{H}_{k}^d \rvert_{\mathbb{S}^{d-1}}$ be the space of homogeneous polynomials over $\R^d$ of degree $k$ restricted to the hypersphere $\mathbb{S}^{d-1}$. We have that the image
$\left\{ \mathcal{R}\phi \, | \, \phi \in \mathcal{S}(\mathcal{B}_{R}(\R^d)) \right\}$ is equal to the space
$\{ \psi \in \mathcal{S}(\mathbb{P}^d_R) \ | \ \forall k \in \mathbb{Z}^{+}, \ \int_{\R} \psi(\omega,b) b^k \, db \, \in \, \mathbb{H}_{k}^d \rvert_{\mathbb{S}^{d-1}} \}$.
\end{lemma}
\begin{proof}
We want to show the equality
\begin{align} \label{eq:radon_set_equality}
    \left\{ \mathcal{R}\phi \, | \, \phi \in \mathcal{S}(\mathcal{B}_{R}(\R^d)) \right\} = \left\{ \psi \in \mathcal{S}(\mathbb{P}^d) \ \bigg| \ \substack{\forall (\omega,b) \in \mathbb{S}^{d-1} \times (\R \setminus (-R,R)), \ \psi(\omega,b) = 0, \\ \forall k \in \mathbb{Z}^{+}, \ \int_{\R} \psi(\omega,b) b^k \, db \, \in \, \mathbb{H}_{k}^d \rvert_{\mathbb{S}^{d-1}}}
    \right\}.
\end{align}
The left-to-right inclusion holds because:
\vspace{-5pt}
\begin{itemize}
    \item For all $\phi \in \mathcal{S}(\mathcal{B}_{R}(\R^d))$, when $\{ x | \langle \omega,x \rangle = b\}$ does not intersect $\mathcal{B}_{R}(\R^d)$, we have $(\mathcal{R}\phi)(\omega,b) = \int_{\{ x | \langle \omega,x \rangle = b\}} \phi(x) \, dx = 0$. This happens when $(\omega,b) \in \mathbb{S}^{d-1} \times (\R \setminus (-R,R))$.
    \item By \autoref{lem:schwartz}, for all $k \in \mathbb{Z}^{+}$, $\int_{\R} (\mathcal{R}\phi)(\omega,b) b^k \, db$ is a homogeneous polynomial of degree $k$.
\end{itemize}
\vspace{-5pt}
To show the right-to-left inclusion note that by \autoref{lem:schwartz}, for all $\psi \in \mathcal{S}(\mathbb{P}^d)$ such that for all $k \in \mathbb{Z}^{+}$, $\int_{\R} (\omega,b) b^k \, db$ is a homogeneous polynomial of degree $k$, there exists $\phi \in \mathcal{S}(\mathbb{R}^d)$ such that $\psi = \mathcal{R} \phi$. Since for all $(\omega,b) \in \mathbb{S}^{d-1} \times \R$ with $b > R$ we have $(\mathcal{R} \phi)(\omega,b) = \psi(\omega,b) = 0$, we apply \autoref{lem:support_radon} and obtain that $\phi(x) = 0$ for all $x$ such that $\|x\|_2 > R$. Note that this implies that the derivatives of $\phi$ of any order are equal to $0$ for all $x$ such that $\|x\|_2 > R$, and by continuity they must be $0$ for all $x$ such that $\|x\|_2 \geq R$. By the definition of the space of Schwartz functions on an open set in the first paragraph of \autoref{sec:representation_open}, we conclude that $\phi \in \mathcal{S}(\mathcal{B}_{R}(\R^d))$.
\end{proof}

\begin{proposition} \label{prop:I_N}
Consider the spaces $\mathcal{I} = \{ \mathcal{R} \phi \ | \phi \in \mathcal{S}(\mathcal{B}_R(\R^d)) \}$ and $\mathcal{N} = \{ \psi \in C^{\infty}(\overline{\mathbb{P}_R^d}) \ | \ \forall x \in \mathcal{B}_R(\R^d), \ (\mathcal{R}^{*} \psi)(x) = 0 \}$ as subspaces of $C^{\infty}(\overline{\mathbb{P}_R^d})$. The topology we consider on $C^{\infty}(\overline{\mathbb{P}_R^d})$ is the one for which a sequence $(f_n)$ converges to $0$ if and only for any differential operator $D$ the sequence $(D f_n)$ converges to zero uniformly. Define the set 
\begin{align} \label{eq:B_set_def}
B =  \{ Y_{k,j} \otimes X^{k'} \ | \ k, j, k' \in \mathbb{Z}^{+}, \ k \equiv k' \ (\text{mod } 2), \ k' < k, \ 1 \leq j \leq N_{d,k} \}.
\end{align}
We have that $\mathcal{N} = \text{cl}(\text{span}(B))$ in the topology of $C^{\infty}(\overline{\mathbb{P}_R^d})$. Also, $\mathcal{N} = \mathcal{I}^{\perp}$, where $\mathcal{I}^{\perp} = \{ \psi \in C^{\infty}(\overline{\mathbb{P}_R^d}) \ | \ \forall \chi \in \mathcal{I}, \ \int_{\mathbb{P}^d_R} \psi(\omega,b) \chi(\omega,b) \, d(\omega,b) = 0 \}$ is the annihilator of $\mathcal{I}$ with respect to the bilinear form on $C^{\infty}(\overline{\mathbb{P}_R^d}) \times \mathcal{S}(\mathcal{B}_R(\R^d))$. 
\end{proposition}

\begin{proof}
The proof is based on pages 15-17 of \cite{helgason1994geometric}, which shows an analogous result for $\R^d$ and $\mathbb{P}^d$. First, note that 
\begin{align}
\begin{split}
    \mathcal{I}^{\perp} &= \{ \psi \in C^{\infty}(\overline{\mathbb{P}_R^d}) \ | \ \forall \phi \in \mathcal{S}(\mathcal{B}_R(\R^d)), \ \int_{\mathbb{P}^d_R} \psi(\omega,b) (\mathcal{R} \phi)(\omega,b) \, d(\omega,b) = 0 \} \\ &= \{ \psi \in C^{\infty}(\overline{\mathbb{P}_R^d}) \ | \ \forall \phi \in \mathcal{S}(\mathcal{B}_R(\R^d)), \ \int_{\mathbb{P}^d} \psi(\omega,b) (\mathcal{R} \phi)(\omega,b) \, d(\omega,b) = 0 \} \\ &= \{ \psi \in C^{\infty}(\overline{\mathbb{P}_R^d}) \ | \ \forall \phi \in \mathcal{S}(\mathcal{B}_R(\R^d)), \ \int_{\mathbb{R}^d} (\mathcal{R}^{*} \psi)(x) \phi(x) \, dx = 0 \} \\ &= \{ \psi \in C^{\infty}(\overline{\mathbb{P}_R^d}) \ | \ \forall x \in \mathcal{B}_R(\R^d), \ (\mathcal{R}^{*} \psi)(x) = 0 \} = \mathcal{N}.
\end{split}
\end{align}
The first equality holds by the definition of the annihilator of $\mathcal{I}$. The second equality holds because if $\phi \in \mathcal{S}(\mathcal{B}_R(\R^d))$, then $\mathcal{R} \phi$ is zero outside of $\mathbb{P}_R^d$. The third equality is by \autoref{lem:prop22helgason}, which holds because $\mathcal{S}(\mathcal{B}_R(\R^d)) \subseteq C_c(\R^d)$ and any function $\psi \in C^{\infty}(\overline{\mathbb{P}_R^d})$ may be trivially extended to a (locally) integrable function on $\mathbb{P}^d$ by setting it equal to zero on the complement of $\overline{\mathbb{P}_R^d}$ (we also use this extension to define the dual Radon transform $\mathcal{R}^{*} \psi$). The fourth equality holds because $\mathcal{S}(\mathcal{B}_R(\R^d))$ is dense in $C_0(\mathcal{B}_R(\R^d))$ and the fifth equality is by the definition of $\mathcal{N}$.

It remains to show that $\mathcal{I}^{\perp} = \text{cl}(B)$. First, we show that if $\psi \in \mathcal{S}(\mathbb{P}^{d}_R)$, then
\begin{align} \label{eq:double_implication_I}
    \psi \in \mathcal{I} \iff \int_{\mathbb{P}_R^d} \psi(\omega,b) \chi(\omega,b) \, d(\omega,b) = 0, \quad \forall \chi \in B.
\end{align}
Suppose that $\psi = \mathcal{R} \phi$ with $\phi \in \mathcal{S}(\mathcal{B}_R(\R^d))$. Then, by \autoref{prop:radon_set_equality_2} we have $\psi \in \mathcal{S}(\mathbb{P}^d_R)$ and
\begin{align} \label{eq:equivalence_mathcal_I}
\forall k' \in \mathbb{Z}^{+}, \quad \int_{-R}^{R} \psi(\omega,b) b^{k'} \, db = \int_{\R} \psi(\omega,b) b^{k'} \, db \, \in \, \mathbb{H}_{k'}^d \rvert_{\mathbb{S}^{d-1}}
\end{align}
Since $\mathbb{H}_{k'}^d \rvert_{\mathbb{S}^{d-1}} = \mathbb{Y}_{k'}^d \rvert_{\mathbb{S}^{d-1}} \oplus \mathbb{Y}_{k'-2}^d \rvert_{\mathbb{S}^{d-1}} \oplus \cdots \oplus \mathbb{Y}_{k'-2[k'/2]}^d \rvert_{\mathbb{S}^{d-1}}$ by \autoref{lem:homogeneous_decomposition}, we have that for any $k > k'$, $Y_{k,j}$ is orthogonal to $\mathbb{H}_{k'}^d$. Hence,
\begin{align}
    \forall k' \in \mathbb{Z}^{+}, \quad \int_{\mathbb{P}_R^d} \psi(\omega,b) (Y_{k,j} \otimes X^{k'})(\omega,b) \, d(\omega,b) = \int_{\mathbb{S}^{d-1}} Y_{k,j}(\omega) \int_{-R}^{R} \psi(\omega,b) b^{k'} \, db \, d\omega = 0,
\end{align}
which concludes the left-to-right implication in \eqref{eq:double_implication_I}. For the reverse implication, we use \autoref{lem:prop28helgason}. From the right-hand side of \eqref{eq:double_implication_I}, we obtain that
\begin{align}
    \forall 0 \leq k' < k, k \equiv k' \ (\text{mod } 2), 1 \leq j \leq N_{k,d}, \quad 0 = Y_{k,j}(\omega) \int_{-R}^{R} \psi_{k,j}(b) b^{k'} \, db,
\end{align}
which implies that $\int_{-R}^{R} \psi_{k,j}(b) b^{k'} \, d(\omega,b) = 0$. Since $\psi_{k,j}(-b) = (-1)^{k} \psi_{k,j}(b)$, when $k$ and $k'$ have different parity we have $\int_{-R}^{R} \psi_{k,j}(b) b^{k'} \, db = 0$ as well. Thus, we obtain that 
\begin{align} \label{eq:int_b_k_decomp}
    \int_{-R}^{R} \psi(\omega,b) b^{k'} \, db = \sum_{0 \leq k \leq k', \, k \equiv k' (\text{mod } 2)} \sum_{j=1}^{N_{k,d}} Y_{k,j}(\omega) \int_{-R}^{R} \psi_{k,j}(\omega,b) b^{k'} \, db.
\end{align}
Since $\mathbb{H}_{k'}^d \rvert_{\mathbb{S}^{d-1}} = \sum_{0 \leq k \leq k', \, k \equiv k' (\text{mod } 2)} \mathbb{Y}_{k}^d \rvert_{\mathbb{S}^{d-1}}$, we conclude that the right-hand side of \eqref{eq:int_b_k_decomp} belongs to $\mathbb{H}_{k'}^d \rvert_{\mathbb{S}^{d-1}}$. Thus, by \autoref{prop:radon_set_equality_2} we obtain that $\psi \in \mathcal{I}$.

Remark that when $\psi \in \mathcal{N}$, we have
\begin{align} \label{eq:dual_is_zero}
    \forall x = r\omega \in \mathcal{B}_R(\R^d), \quad (\mathcal{R}^{*} \psi)_k(x) = \sum_{j=1}^{N_{k,d}} Y_{k,j}(\omega) \int_{\mathbb{S}^{d-1}} (\mathcal{R}^{*} \psi)(r\tilde{\omega}) Y_{k,j}(\tilde{\omega}) \, d\tilde{\omega} = 0.
\end{align}
Consider now the spaces $\mathcal{S}(\mathcal{B}_R(\R^d))_k = \{ \phi_k \ | \ \phi \in \mathcal{S}(\mathcal{B}_R(\R^d)) \}$ and $\mathcal{S}(\mathbb{P}_R^d)_k = \{ \psi_k \ | \ \phi \in \mathcal{S}(\mathbb{P}_R^d)\}$. By \eqref{eq:double_implication_I} we have that when $\psi \in \mathcal{S}(\mathbb{P}_R^d)_k$,
\begin{align}
\begin{split}
    &\psi \in \mathcal{R}(\mathcal{S}(\mathcal{B}_R(\R^d)))_k \iff \\ &\forall k' < k, k \equiv k' \ (\text{mod } 2), 1 \leq j \leq N_{k,d}, \quad \int_{\mathbb{P}_R^d} \psi(\omega,b) (Y_{k,j} \otimes X^{k'})(\omega,b) \, d(\omega,b) = 0.
\end{split}
\end{align}
In other words, the finite-dimensional space $E_{k} = \text{span}(\{ Y_{k,j} \otimes X^{k'} \ | \ k' < k, k \equiv k' \ (\text{mod } 2), 1 \leq j \leq N_{k,d} \})$ has annihilator $E_{k}^{\perp} = \mathcal{R}(\mathcal{S}(\mathcal{B}_R(\R^d)))_k$ within $\mathcal{S}(\mathbb{P}_R^d)_k$. Since for $\psi \in \mathcal{N}$ we have $(\mathcal{R}^{*} \psi)_k = 0$ on $\mathcal{B}_R(\R^d)$ by \eqref{eq:dual_is_zero}, we get that \begin{align} \label{eq:double_annihilator}
   \forall \phi \in \mathcal{S}(\mathcal{B}_R(\R^d)), \quad \langle (\mathcal{R} \phi)_k, \psi_k \rangle = \langle \mathcal{R} \phi, \psi_k \rangle = \langle \phi, \mathcal{R}^{*} \psi_k \rangle = \langle \phi, (\mathcal{R}^{*} \psi)_k \rangle = 0. 
\end{align}
Here, the first equality holds because $\psi_k \in \mathcal{S}(\mathbb{P}_R^d)_k$, the second equality holds by \autoref{lem:prop22helgason} and the third equality holds because $\mathcal{R}^{*} \psi_k = (\mathcal{R}^{*} \psi)_k$ by \autoref{lem:commutation_R_k}. Equation \eqref{eq:double_annihilator} implies that $\psi_k$ belongs to the annihilator $(\mathcal{R}(\mathcal{S}(\mathcal{B}_R(\R^d)))_{k})^{\perp} = (E_{k}^{\perp})^{\perp}$. As argued in \cite{helgason1994geometric}, the finite dimension of $E_k$ implies that $E_k = (E_{k}^{\perp})^{\perp}$. Thus, $\psi_k$ belongs to $E_k$. Since the convergence of $\sum_{k=0}^{\infty} \psi_k$ to $\psi$ is in the topology of $C^{\infty}(\overline{\mathbb{P}_R^d})$, we conclude that $\psi \in \text{cl}(\bigoplus_{k=0}^{\infty} E_k) = \text{cl}(\text{span}(B))$. 
\end{proof}

\begin{corollary} \label{cor:C_0_span_B_R_S}
Let $B$ be the set defined in \eqref{eq:B_set_def}.
Let $C_0(\mathbb{P}^d_R)$ be the set of continuous functions $f$ on $\mathbb{P}^d_R$ such that $\forall \epsilon > 0$ there exists a compact $K \subseteq \mathbb{P}^d_R$ such that $|f(\omega,b)| < \epsilon$ for $(\omega,b) \in \mathbb{P}^d_R \setminus K$ (note that $C_0(\mathbb{P}^d_R)$ may be regarded as a subset of $C(\overline{\mathbb{P}^d_R})$ by considering the trivial extension to the boundary). Then, we can write
\begin{align} \label{eq:C_0_subset}
    C_0(\mathbb{P}^d_R) \subseteq \text{cl}_{\infty}(\text{span}(B) \oplus \mathcal{R}(\mathcal{S}(\mathcal{B}_R(\R^d)))),
\end{align}
where the closure $\text{cl}_{\infty}$ is in the topology of $C_0(\mathbb{P}^d_R)$ (the topology of uniform convergence).
\end{corollary}
\begin{proof}

Note also that by the proof of \autoref{prop:I_N}, we have that the annihilator of $E_k$ within $\mathcal{S}(\mathbb{P}^{d}_R)_k$ is $\mathcal{R}(\mathcal{S}(\mathcal{B}_R(\R^d)))_k$. We want to show that the annihilator of $E_k$ within $\text{cl}_{\infty}(\mathcal{S}(\mathbb{P}^{d}_R)_k)$ is $\text{cl}_{\infty}(\mathcal{R}(\mathcal{S}(\mathcal{B}_R(\R^d)))_k)$. 
Note that $\text{cl}_{\infty}(\mathcal{S}(\mathbb{P}^{d}_R)_k) = \{ \psi = \sum_{j=1}^{N_{k,d}} \psi_{k,j} \otimes Y_{k,j} \ | \ \psi_{k,j} \in C_0((-R,R)) \text{ even } \}$. To show that the annihilator of $E_k$ within $\text{cl}_{\infty}(\mathcal{S}(\mathbb{P}^{d}_R)_k)$ is $\text{cl}_{\infty}(\mathcal{R}(\mathcal{S}(\mathcal{B}_R(\R^d)))_k)$, it suffices to show that 
\begin{align}
\begin{split} \label{eq:closure_R_S}
&\text{cl}_{\infty}(\mathcal{R}(\mathcal{S}(\mathcal{B}_{R}(\R^d)))_k) \\ &= \{ \psi = \sum_{j=1}^{N_{k,d}} \psi_{k,j} \otimes Y_{k,j} \ | \ \psi_{k,j} \in C_0((-R,R)) \text{ even }, 
\forall k' < k, \ \int_{\R} \psi(\omega,b) b^{k'} \, db = 0 \}.
\end{split}
\end{align}
Before going forward, remark that one can construct functions $(\chi_{k'})_{k'<k, k' \text{ even}}$ in $\mathcal{S}((-R,R))$ such that for any $k', k'' < k,$ $k', k''$ even, we have $\int_{\R} \chi_{k'}(b) b^{k''} \, db = \delta_{k',k''}$. 
Suppose that ${(\psi_i)}_{i=0}^{\infty} \subseteq \mathcal{S}(\mathbb{P}^{d}_R)_k$ converges to $\psi$ belonging to the right-hand side of \eqref{eq:closure_R_S}. Then, we construct the sequence ${(\tilde{\psi}_i)}_{i=0}^{\infty}$ as
\begin{align}
    \tilde{\psi}_i(\omega,b) = \psi_i(\omega,b) - \sum_{j=1}^{N_{k,d}} \sum_{k' < k, k' \text{even}} \langle \psi_i, Y_{k,j} \otimes X^{k'} \rangle Y_{k,j}(\omega) \chi^{k'}(b).
\end{align}
Remark that for any $k' < k$, $k'$ even,
\begin{align}
\begin{split}
    \int_{\R} \tilde{\psi}_i(\omega,b) b^{k'} \, db &= \int_{\R} \psi_i(\omega,b) b^{k'} \, db - \sum_{j=1}^{N_{k,d}} \sum_{k'' < k, k'' \text{even}} \langle \psi_i, Y_{k,j} \otimes X^{k''} \rangle Y_{k,j}(\omega) \int_{\R} \chi_{k''}(b) b^{k'} \, db \\ &= \int_{\R} \left( \sum_{j=1}^{N_{k,d}} \langle \psi_i(\cdot,b), Y_{k,j} \rangle Y_{k,j}(\omega) \right) b^{k'} \, db - \sum_{j=1}^{N_{k,d}} \langle \psi_i, Y_{k,j} \otimes X^{k'} \rangle Y_{k,j}(\omega) = 0
\end{split} 
\end{align}
We have that ${(\tilde{\psi}_i)}_{i=0}^{\infty}$ belongs to  $\mathcal{R}(\mathcal{S}(\mathcal{B}_R(\R^d)))_k$ because by \autoref{prop:radon_set_equality_2} and the decomposition in \autoref{lem:homogeneous_decomposition}, $\mathcal{R}(\mathcal{S}(\mathcal{B}_R(\R^d)))_k = \{ \psi \in \mathcal{S}(\mathbb{P}^{d}_R)_k \ | \ \forall k' < k, \ \int_{\R} \psi(\omega,b) b^{k'} \, db = 0 \}$. 
Moreover, 
\begin{align}
\begin{split}
    \lim_{i \to \infty} \tilde{\psi}_i(\omega,b) &= \lim_{i \to \infty} \psi_i(\omega,b) - \sum_{j=1}^{N_{k,d}} \sum_{k' < k, k' \text{even}} \left( \lim_{i \to \infty} \langle \psi_i, Y_{k,j} \otimes X^{k'} \rangle \right) Y_{k,j}(\omega) \chi^{k'}(b) \\ &= \psi(\omega,b) - \sum_{j=1}^{N_{k,d}} \sum_{k' < k, k' \text{even}} \langle \psi, Y_{k,j} \otimes X^{k'} \rangle Y_{k,j}(\omega) \chi^{k'}(b) = \psi(\omega,b),
\end{split}
\end{align}
where the last equality follows from $\psi$ fulfilling $\int_{\R} \psi(\omega,b) b^{k'} \, db = 0$. Moreover, the convergence is uniform on $\mathbb{P}^{d}_R$, which shows that $\psi \in \text{cl}_{\infty}(\mathcal{R}(\mathcal{S}(\mathcal{B}_{R}(\R^d)))_k)$.

By \autoref{lem:annihilator_sum}, we have that the annihilator of $E_k$ within $\text{cl}_{\infty}(\mathcal{S}(\mathbb{P}^{d}_R)_k) + E_k$ is $\text{cl}_{\infty}(\mathcal{R}(\mathcal{S}(\mathcal{B}_R(\R^d)))_k)$. Now, note that $\langle \cdot, \cdot \rangle$ is an inner product on $\text{cl}_{\infty}(\mathcal{S}(\mathbb{P}^{d}_R)_k) + E_k$, and that we are in position to apply \autoref{lem:direct_sum_annihilator} with $\mathcal{E} = \text{cl}_{\infty}(\mathcal{S}(\mathbb{P}^{d}_R)_k) \oplus E_k$, $U = E_k$, and $U^{\perp} = \text{cl}_{\infty}(\mathcal{R}(\mathcal{S}(\mathcal{B}_R(\R^d)))_k)$. We obtain that
\begin{align} \label{eq:direct_sum_annihilator}
    \text{cl}_{\infty}(\mathcal{S}(\mathbb{P}^{d}_R)_k) + E_k = E_k \oplus \text{cl}_{\infty}(\mathcal{R}(\mathcal{S}(\mathcal{B}_R(\R^d)))_k).
\end{align}
Consequently,
\begin{align}
\begin{split}
    C_0(\mathbb{P}^{d}_R) &\subseteq \text{cl}_{\infty} \left( \bigoplus_{k=0}^{\infty} \text{cl}_{\infty}(\mathcal{S}(\mathbb{P}^{d}_R)_k) + E_k \right) 
    = \text{cl}_{\infty} \left( \bigoplus_{k=0}^{\infty} E_k \oplus \text{cl}_{\infty}(\mathcal{R}(\mathcal{S}(\mathcal{B}_R(\R^d)))_k) \right) 
    \\ &\subseteq \text{cl}_{\infty} \left( \text{cl}_{\infty}(\mathcal{R}(\mathcal{S}(\mathcal{B}_R(\R^d)))) \oplus \bigoplus_{k=0}^{\infty} E_k \right) = \text{cl}_{\infty} \left(\mathcal{R}(\mathcal{S}(\mathcal{B}_R(\R^d))) \oplus \bigoplus_{k=0}^{\infty} E_k \right).
\end{split}
\end{align}
Here, the first inclusion holds because $C_0(\mathbb{P}^{d}_R) = \text{cl}_{\infty}(\mathcal{S}(\mathbb{P}^{d}_R)) = \text{cl}_{\infty}(\text{cl}(\bigoplus_{k=0}^{\infty} \mathcal{S}(\mathbb{P}^{d}_R)_k)) = \text{cl}_{\infty}(\bigoplus_{k=0}^{\infty} \mathcal{S}(\mathbb{P}^{d}_R)_k)$, where we used that the uniform norm topology is weaker than the topology of $C^{\infty}(\overline{\mathbb{P}_R^d})$. The first equality holds by \eqref{eq:direct_sum_annihilator}. The second inclusion holds because for any $k \in \mathbb{Z}^{+}$, $\mathcal{R}(\mathcal{S}(\mathcal{B}_R(\R^d)))_k \subseteq \mathcal{R}(\mathcal{S}(\mathcal{B}_R(\R^d)))$ by \autoref{lem:R_k_inclusion}. The last equality holds because by a diagonal argument, $\text{cl}_{\infty}(\text{cl}_{\infty} A + B) = \text{cl}_{\infty}(A + B)$ for any subspaces $A,B$.
\end{proof}

\begin{lemma} \label{lem:commutation_R_k}
Let $\psi \in \mathcal{S}(\mathcal{B}_R(\R^d))$. Let $(\mathcal{R}^{*} \psi)_k$ be the $k$-th component of the decomposition from \autoref{lem:prop27helgason} and $\mathcal{R}^{*} \psi_k$ the Radon transform of the $k$-th component of the decomposition given by \autoref{lem:prop28helgason}. Then $(\mathcal{R}^{*} \psi)_k = \mathcal{R}^{*} \psi_k$ for any $k \in \mathbb{Z}^{+}$.
\end{lemma}
\begin{proof}
Note also that
\begin{align}
\begin{split} \label{eq:r_k_first}
    (\mathcal{R}^{*} \psi)_k(x) &= \sum_{j=1}^{N_{k,d}} Y_{k,j}(\omega) \int_{\mathbb{S}^{d-1}} (\mathcal{R}^{*} \psi)(r\tilde{\omega}) Y_{k,j}(\tilde{\omega}) \, d\tilde{\omega} \\ &= \sum_{j=1}^{N_{k,d}} Y_{k,j}(\omega) \int_{\mathbb{S}^{d-1}} \int_{\mathbb{S}^{d-1}} \psi(\hat{\omega}, r \langle \hat{\omega}, \tilde{\omega} \rangle) \, d\hat{\omega} \, Y_{k,j}(\tilde{\omega}) \, d\tilde{\omega} \\ &= \sum_{j=1}^{N_{k,d}} Y_{k,j}(\omega) \int_{\mathbb{S}^{d-1}} \int_{\mathbb{S}^{d-1}} \sum_{k'=0}^{\infty} \sum_{j'=1}^{N_{k',d}} \psi_{k',j'}(r \langle \hat{\omega}, \tilde{\omega} \rangle) Y_{k',j'}(\hat{\omega}) \, d\hat{\omega} \, Y_{k,j}(\tilde{\omega}) \, d\tilde{\omega} \\ &= \sum_{k'=0}^{\infty} \sum_{j'=1}^{N_{k',d}} \sum_{j=1}^{N_{k,d}} Y_{k,j}(\omega) \int_{\mathbb{S}^{d-1}} \int_{\mathbb{S}^{d-1}}  Y_{k',j'}(\hat{\omega}) \psi_{k',j'}(r \langle \hat{\omega}, \tilde{\omega} \rangle) \, d\hat{\omega} \, Y_{k,j}(\tilde{\omega}) \, d\tilde{\omega} \\ &= \sum_{k'=0}^{\infty} \sum_{j'=1}^{N_{k',d}} Y_{k',j'}(\omega) \int_{\mathbb{S}^{d-1}} \int_{\mathbb{S}^{d-1}} Y_{k',j'}(\hat{\omega}) \psi_{k',j'}(r \langle \hat{\omega}, \tilde{\omega} \rangle) \, d\hat{\omega} \, Y_{k',j'}(\tilde{\omega}) \, d\tilde{\omega}
\end{split}
\end{align}
In the last equality, we used that $\int_{\mathbb{S}^{d-1}} \int_{\mathbb{S}^{d-1}}  Y_{k',j'}(\hat{\omega}) \psi_{k',j'}(r \langle \hat{\omega}, \tilde{\omega} \rangle) \, d\hat{\omega} \, Y_{k,j}(\tilde{\omega}) \, d\tilde{\omega} = 0$ if $(k',j') \neq (k,j)$, which holds by the Funk-Hecke formula (\autoref{lem:funk-hecke}). Also,
\begin{align}
\begin{split} \label{eq:r_k_second}
    (\mathcal{R}^{*} \psi_k)(x) &= \int_{\mathbb{S}^{d-1}} \psi_k(\hat{\omega}, r \langle \hat{\omega}, \omega \rangle) \, d\hat{\omega} = \int_{\mathbb{S}^{d-1}} \sum_{j=1}^{N_{k,d}} \int_{\mathbb{S}^{d-1}} \psi(\tilde{\omega}, r \langle \hat{\omega}, \omega \rangle) Y_{k,j}(\tilde{\omega}) \, d\tilde{\omega} \, Y_{k,j}(\hat{\omega}) \, d\hat{\omega} \\ &= \int_{\mathbb{S}^{d-1}} \sum_{j=1}^{N_{k,d}} \psi_{k,j}(r \langle \hat{\omega}, \omega \rangle) Y_{k,j}(\hat{\omega}) \, d\hat{\omega}
    \\ &= \sum_{k'=0}^{\infty} \sum_{j'=1}^{N_{k',d}} Y_{k',j'}(\omega) \int_{\mathbb{S}^{d-1}} \int_{\mathbb{S}^{d-1}} \sum_{j=1}^{N_{k,d}} \psi_{k,j}(r \langle \hat{\omega}, \tilde{\omega} \rangle) Y_{k,j}(\hat{\omega}) \, d\hat{\omega} \, Y_{k',j'}(\tilde{\omega}) \, d\tilde{\omega} \\ &= 
    \sum_{k'=0}^{\infty} \sum_{j'=1}^{N_{k',d}} Y_{k',j'}(\omega) \int_{\mathbb{S}^{d-1}} \int_{\mathbb{S}^{d-1}} \psi_{k',j'}(r \langle \hat{\omega}, \tilde{\omega} \rangle) Y_{k',j'}(\hat{\omega}) \, d\hat{\omega} \, Y_{k',j'}(\tilde{\omega}) \, d\tilde{\omega}
\end{split}    
\end{align}
Note that the right-hand sides of \eqref{eq:r_k_first} and \eqref{eq:r_k_second} are equal, which concludes the proof.
\end{proof}

\begin{lemma} \label{lem:R_k_inclusion}
We have
\begin{align}
    \mathcal{R}(\mathcal{S}(\mathcal{B}_R(\R^d)))_k \subseteq \mathcal{R}(\mathcal{S}(\mathcal{B}_R(\R^d))).
\end{align}
\end{lemma}
\begin{proof}
By \autoref{prop:radon_set_equality_2} and the decomposition in \autoref{lem:homogeneous_decomposition}, 
\begin{align} \label{eq:R_S_k_development}
\mathcal{R}(\mathcal{S}(\mathcal{B}_R(\R^d)))_k = \{ \psi \in \mathcal{S}(\mathbb{P}^{d}_R)_k \ | \ \forall k' < k, \ \int_{\R} \psi(\omega,b) b^{k'} \, db = 0 \}.
\end{align}
By the decomposition in \autoref{lem:prop28helgason}, for any $\psi \in \mathcal{S}(\mathbb{P}^{d}_R)$, we have that
\begin{align}
    \psi_k(\omega,b) = \sum_{j=1}^{N_{k,d}} \psi_{k,j}(b) Y_{k,j}(\omega), \quad \text{where} \quad \psi_{k,j}(b) = \int_{\mathbb{S}^{d-1}} \psi(\omega,b) Y_{k,j}(\omega) \, d\omega.
\end{align}
Note that $\psi_{k,j} \in \mathcal{S}((-R,R))$ because differentiating under the integral sign, one obtains $(\frac{d^j}{db^j}\psi_k)(\omega,R) = (\frac{d^j}{db^j}\psi_k)(\omega,-R) = 0$. Thus, $\mathcal{S}(\mathbb{P}^{d}_R)_k = \{ \sum_{j=1}^{N_{k,d}} \psi_{k,j} \otimes Y_{k,j} \ | \ \psi_{k,j} \in \mathcal{S}((-R,R)) \}$, which is included in $\mathcal{S}(\mathbb{P}^{d}_R)$. Plugging this into \eqref{eq:R_S_k_development}, we get
\begin{align}
\begin{split}
    &\mathcal{R}(\mathcal{S}(\mathcal{B}_R(\R^d)))_k = \left\{ \sum_{j=1}^{N_{k,d}} \psi_{k,j} \otimes Y_{k,j} \ \bigg| \ \psi_{k,j} \in \mathcal{S}((-R,R)), \ \forall k' < k, \ \int_{\R} \psi(\omega,b) b^{k'} \, db = 0 \right\} \\ &\subseteq \left\{ \psi \in \mathcal{S}(\mathbb{P}^{d}_R) \ | \ \forall k \in \mathbb{Z}^{+}, \ \int_{\R} \psi(\omega,b) b^k \, db \, \in \, \mathbb{H}_{k}^d \rvert_{\mathbb{S}^{d-1}} \right\} = \mathcal{R}(\mathcal{S}(\mathcal{B}_R(\R^d)))
\end{split}
\end{align}
where the inclusion is by the decomposition in \autoref{lem:homogeneous_decomposition} and the last equality is by \autoref{prop:radon_set_equality_2}.
\end{proof}

\begin{proposition}[Adaptation of Lemma 9 of \cite{ongie2019function}] \label{prop:9_ongie}
Let $\mathcal{U} \subseteq \R^d$ be an open set. Define $f : \mathcal{U} \to \R$ as $f(x) = \int_{\mathbb{P}^d_{\mathcal{U}}} (\langle \omega, x \rangle - b)_{+} \, d\alpha(\omega,x) + \langle v, x \rangle + c$ for some $\alpha \in \mathcal{M}(\mathbb{P}^d_{\mathcal{U}}), \ v \in \R^d, \ c \in \R$. For any $\phi \in \mathcal{S}(\mathcal{U})$,
\begin{align}
    \langle f, \Delta \phi \rangle = \int_{\mathbb{P}^d_{\mathcal{U}}} (\mathcal{R} \phi)(\omega,b) \, d\alpha(\omega,b). 
\end{align}
\end{proposition}

\begin{proof}
For any $\phi \in \mathcal{S}(\mathcal{U})$,
\begin{align}
\begin{split}
    &\int_{\mathcal{U}} f(x) \Delta \phi(x) \, dx = \int_{\R^d} f(x) \Delta \phi(x) \, dx \\ &= \int_{\R^d} \int_{\mathbb{P}^{d}_R} (\langle \omega, x \rangle - b)_{+} \, d\alpha(\omega,b) \, \Delta \phi(x) \, dx \\ &= \int_{\mathbb{P}^{d}_R} \int_{\R^d} (\langle \omega, x \rangle - b)_{+} \Delta \phi(x) \, dx \, d\alpha(\omega,b) \\ &= - \int_{\mathbb{P}^{d}_R} \int_{\R^d} \mathds{1}_{\langle \omega, x \rangle > b} \langle \omega, \nabla \phi(x) \rangle \, dx \, d\alpha(\omega,b) 
    \\ &= - \int_{\mathbb{S}^{d-1} \times \R} \int_{\{ x' | \langle \omega, x' \rangle = b \}} \int_{\mathbb{R}^{+}} \mathds{1}_{t > b} \langle \omega, \nabla \phi(x' + t\omega) \rangle \, dt \, dx' \, d\alpha(\omega,b) \\ &= - \int_{\mathbb{S}^{d-1} \times \R} \int_{\{ x' | \langle \omega, x' \rangle = b \}} \int_{\mathbb{R}^{+}} \mathds{1}_{t > b} \frac{d}{dt} \phi(x'+t\omega) \, dt \, dx' \, d\alpha(\omega,b) \\ &= - \int_{\mathbb{S}^{d-1} \times \R} \int_{\{ x' | \langle \omega, x' \rangle = b \}} - \phi(x') \, dx' \, d\alpha(\omega,b) = \int_{\mathbb{S}^{d-1} \times \R} (\mathcal{R} \phi)(\omega,b) \, d\alpha(\omega,b)
\end{split}
\end{align}
The first equality holds because the Laplacian of an affine function is 0. The second equality holds by Fubini's theorem. The third equality follows from the divergence theorem, since $\phi \in \mathcal{S}(\mathcal{U})$, the boundary term is zero. 
\end{proof}

\textbf{\textit{Proof of \autoref{prop:characterization_alpha}.}}
By the definition of $\alpha$ in \autoref{prop:alpha_existence}, we have
\begin{align}
\begin{split}
\forall \phi \in \mathcal{S}(\mathcal{B}_R(\R^d)), \quad \langle \mathcal{R} \phi, \alpha \rangle = c_d \langle f, (-\Delta)^{(d+1)/2} \mathcal{R}^{*}\mathcal{R} \phi \rangle = -\langle f, \Delta \phi \rangle.
\end{split}
\end{align}
In the last equality we use that \eqref{eq:helgason_inversion}.
\autoref{lem:alpha_Y_kj_X_kprime} readily implies that
\begin{align}
    \forall k, j, k' \text{ s.t. } k' \equiv k \ (\text{mod } 2), \ k' < k, \ \ \langle \alpha, Y_{k,j} \otimes X^{k'} \rangle = - c_d \langle f, (-\Delta)^{(d+1)/2} \mathcal{R}^{*} (Y_{k,j} \otimes \mathds{1}_{|X|<R} X^{k'}) \rangle.
\end{align}
To show that $\alpha$ is the unique measure fulfilling \eqref{eq:alpha_characterization_1} and \eqref{eq:alpha_characterization_2}, we argue via the Riesz-Markov-Kakutani representation theorem, which states that the space of signed Radon measures $\mathcal{M}(\mathbb{P}^d_R) $ is the dual space $C_0'(\mathbb{P}^d_R)$ of the space $C_0(\mathbb{P}^d_R)$ of continuous functions vanishing at infinity. That is, a signed Radon measure is uniquely determined by its action on $C_0(\mathbb{P}^d_R)$, or on any dense space in $C_0(\mathbb{P}^d_R)$ in the topology of uniform convergence (as a corollary of the dominated convergence theorem). \autoref{cor:C_0_span_B_R_S} states that $C_0(\mathbb{P}^d_R) \subseteq \text{cl}_{\infty}(\text{span}(B) \oplus \mathcal{R}(\mathcal{S}(\mathcal{B}_R(\R^d))))$, or in other words, that $\text{span}(B) \oplus \mathcal{R}(\mathcal{S}(\mathcal{B}_R(\R^d)))$ is dense in $C_0(\mathbb{P}^d_R)$ in the uniform convergence topology. Since \eqref{eq:alpha_characterization_1}-\eqref{eq:alpha_characterization_2} specify how $\alpha$ acts on $\text{span}(B) \oplus \mathcal{R}(\mathcal{S}(\mathcal{B}_R(\R^d)))$. By the BLT theorem (\autoref{lem:blt}), any linear mapping which is continuous on $\text{span}(B) \oplus \mathcal{R}(\mathcal{S}(\mathcal{B}_R(\R^d)))$ in the uniform convergence topology can be extended uniquely to a continuous linear mapping on $C_0(\mathbb{P}^d_R)$. This means that there exists a unique measure in $C_0'(\mathbb{P}^d_R) = \mathcal{M}(\mathbb{P}^d_R)$ that coincides with $\alpha$ on $\text{span}(B) \oplus \mathcal{R}(\mathcal{S}(\mathcal{B}_R(\R^d)))$, which is $\alpha$ itself.

\autoref{prop:9_ongie} readily states that \eqref{eq:alpha_characterization_1} holds for any measure $\alpha' \in \mathcal{M}(\mathbb{P}^d_R)$ for which $f$ admits a representation of the form \eqref{eq:f_psi_2} on $\mathcal{B}_R(\R^d)$. The characterization of $\alpha$ as the unique measure in $\mathcal{M}(\mathbb{P}^d_R)$ such that \eqref{eq:alpha_characterization_2} holds and $f$ admits a representation of the form \eqref{eq:f_psi_2} on $\mathcal{B}_R(\R^d)$ can be directly deduced from the uniqueness of the characterization \eqref{eq:alpha_characterization_1}-\eqref{eq:alpha_characterization_2} and the fact that \eqref{eq:alpha_characterization_1} holds for any measure for which $f$ admits a representation.
\qed

\end{document}